\documentclass[12pt]{report} 

\usepackage{amssymb}
\usepackage{graphicx}
\usepackage{multicol}
\usepackage{amsmath}
\usepackage{ragged2e}
\usepackage{subfig}
\usepackage{booktabs}
\usepackage[most]{tcolorbox}
\usepackage{mathtools}
\usepackage[sort&compress]{natbib}
\usepackage{hyperref}
\usepackage{amsthm}
\usepackage{multirow}
\usepackage[most]{tcolorbox}
\usepackage[printonlyused]{acronym} 
\usepackage[export]{adjustbox}
\usepackage{caption}
\usepackage{subcaption}
\usepackage{graphbox}   
\usepackage{lmodern}
\usepackage{color,soul}
\usepackage{wrapfig}
\usepackage{epsfig}

\usepackage{tikz}
\usetikzlibrary{shapes.geometric, arrows, positioning, fit, calc}
\newtheorem{theorem}{Theorem}
\newtheorem{lemma}{Lemma}
\newtheorem{definition}{Definition}
\usepackage[ruled,linesnumbered]{algorithm2e}
\def\mytitle{Fast \& Efficient Normalizing Flows and Applications of Image Generative Models}
\def\myname{Sandeep Nagar}
\def\mydegree{Doctor of Philosophy}
\def\mysupervisor{Dr. Girish Varma}
\def\myrollno{2018701015}
\def\myemailid{sandeep.nagar@research.iiit.ac.in}

\begin{document}
 \baselineskip=18pt plus1pt
 \setcounter{secnumdepth}{3}
 \setcounter{tocdepth}{3}
 \pagenumbering{roman}
 
 \thispagestyle{empty}
\begin{center}
    { \huge {\bfseries {\mytitle}} \par}
\vspace{1.5\baselineskip}
    {\textit{Thesis to be submitted in partial fulfillment of the}\\
    \textit{requirements for the degree}}\par
\vspace{\baselineskip}
    {\textit{of} \par}
\vspace{\baselineskip}
    {\large \bf \mydegree \par} 
    {\textit{in} \par}
    {\textit{Computer Science and Engineering} \par}
\vspace{\baselineskip}
    {\textit{by} \par}
\vspace{\baselineskip}
    {{\large {\bf \myname \\ \myrollno \\ \myemailid}} \par}
\vspace{1.5\baselineskip}
    {Under the guidance of \par}
\vspace{\baselineskip}
    {{\large \bf \mysupervisor} \par}
\vspace{.5\baselineskip}
    {\begin{figure}[!h] 
	\centering
	\includegraphics[width=25mm]{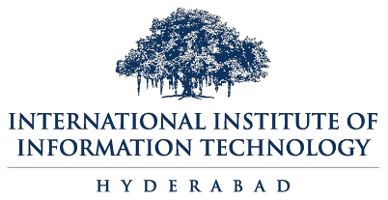} 
     \end{figure}
    }
    {\bf International Institute of Information Technology \par (Deemed to be University)
    \par Hyderabad - 500032, INDIA
    \par November 2025}
 \end{center}
\thispagestyle{empty}

\vspace*{\fill}
\centering {{\large Copyright \copyright~~Sandeep Nagar, 2025\\}{\large All Rights Reserved\\}}
\vspace*{\fill}
 \thispagestyle{plain}

 \includegraphics[width=60mm]{Images/iiit-new.png}
 

\vspace{0.5\baselineskip}
\vspace{1\baselineskip}

\begin{center}
{\Large {\bf \uppercase{Certificate}}}
\end{center}

\vspace{\baselineskip}

\justifying

\noindent It is certified that the work contained in this PhD thesis titled "{\bf \mytitle}", submitted by {\bf \myname} (Roll Number: {\textit{\myrollno}}) has been carried out under our supervision and is not submitted elsewhere for a degree.

\vspace*{6cm}
\begin{tabular}{cc}
\underline{\makebox[1in]{}} & \hspace*{5cm} \underline{\makebox[2.5in]{ \includegraphics[height=3cm]{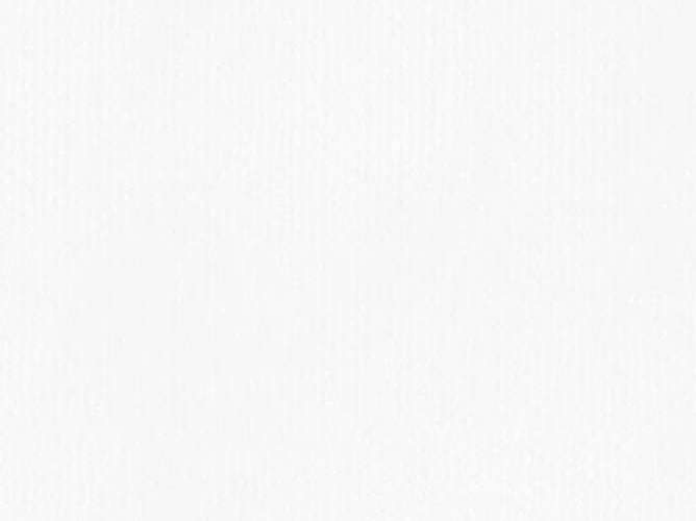}}} \\
Date & \hspace*{5cm} Adviser: Dr. Girish Varma \\
& \hspace*{5cm} Assistant Prof. IIIT Hyderabad
\end{tabular}

 
\thispagestyle{empty}

\vspace*{\fill}
\centering {\large To my parents and family.}
\vspace*{\fill}
 \thispagestyle{plain}

\begin{center}
 \Large {\bf \uppercase{Acknowledgment}}
\end{center}

\vspace{\baselineskip}
\justifying
\noindent I am deeply grateful to my supervisor, \textbf{Dr. Girish Varma}, for his guidance, invaluable insights, and unwavering support throughout this doctoral journey. His mentorship has been instrumental in shaping the direction of this research and fostering my academic growth. His contributions have greatly enriched the quality of this thesis and my overall research experience.

I am grateful to \textbf{ML Lab and C-STAR Lab, IIIT Hyderabad} for providing the essential resources and academic environment crucial for the successful completion of this research. I am deeply thankful to \textbf{Dr. Dendi Sathya Veera Reddy}, Chief Engineer, for his feedback and assistance during my research internship at Samsung Research Bangalore (SRI-B).

I express my sincere gratitude to \textbf{Prof. Narendra Ahuja} (UIUC) and \textbf{Dr. Rohitash Chandra} (UNSW) for their invaluable support and understanding during my research internship. Their guidance and encouragement were instrumental in laying a strong foundation for my work.

Special thanks to my friends \textbf{Mukesh, Rishabh, Habeeba, Vatan}, and colleagues at IIIT Hyderabad for their unwavering support, camaraderie, and encouragement throughout this journey.

I am profoundly thankful to \textbf{my family} for their unwavering love, encouragement, and understanding. Their belief in me has been a constant source of motivation and strength.

\noindent
\vspace{\baselineskip} \\
\textbf{\myname} \\
 \thispagestyle{plain}
\begin{center}
    \Large \textbf{\uppercase{Abstract}}
\end{center}

\justifying
\noindent

This thesis presents novel contributions in two primary areas: advancing the efficiency of generative models, particularly normalizing flows, and applying generative models to solve real-world computer vision challenges. In the first part, we introduce significant improvements to normalizing flow architectures through six key innovations: (1) Development of invertible $3\times3$ Convolution layers with mathematically proven necessary and sufficient conditions for invertibility, (2) introduction of a more efficient Quad-coupling layer, (3) Design of a fast and efficient parallel inversion algorithm for k×k convolutional layers, (4) Fast \& efficient backpropagation algorithm for inverse of convolution, (5) Using inverse of convolution, in Inverse-Flow, for the forward pass and training it using proposed backpropagation algorithm, and (6) Affine-StableSR, a compact and efficient super-resolution model that leverages pre-trained weights and Normalizing Flow layers to reduce parameter count while maintaining performance. These advances maintain model expressiveness while substantially improving computational efficiency compared to existing approaches.

The second part demonstrates the practical applications of generative modeling advances across diverse computer vision tasks. This thesis develops: (1) An automated quality assessment system for agricultural produce using Conditional GANs to address class imbalance, data scarcity and annotation challenges, achieving good accuracy in seed purity testing; (2) An unsupervised geological mapping framework utilizing stacked autoencoders for dimensionality reduction, showing improved feature extraction compared to conventional methods; (3) We proposed a privacy preserving method for autonomous driving datasets using on face detection and image inpainting; (4) Utilizing Stable Diffusion based image inpainting for replacing the detected face and license plate to advancing privacy-preserving techniques and ethical considerations in the field.; and (5) An adapted diffusion model for art restoration that effectively handles multiple types of degradation through unified fine-tuning.

This thesis advances the theoretical understanding of generative models and their practical applications, demonstrating significant improvements in efficiency, scalability, and real-world utility across multiple domains.  

\vspace{\baselineskip}


 \tableofcontents
 \listoffigures
 \listoftables

 \clearpage
 \pagenumbering{arabic}
 
 \justifying
\part{Introduction}
 \chapter{Introduction}
Recent advances in generative modeling have led to various powerful models, each offering distinct advantages for synthesizing high-quality data. Normalizing Flow (NF) models, in particular, have gained attention for their ability to provide exact likelihood estimation and invertible transformations, making them highly interpretable and capable of learning complex data distributions. Unlike Generative Adversarial Networks (GANs) \cite{gan_high, gans}, which rely on adversarial training and are often difficult to stabilize, NF models use a series of invertible transformations to map data to a latent space, enabling direct maximum likelihood training. Similarly, Variational Autoencoders (VAEs) \cite{yu2021autoencoder, protopapadakis2021stacked} offer efficient latent variable models but suffer from limitations in expressiveness due to the need for approximating the actual posterior distribution. Diffusion models \cite{saharia2022photorealistic, fei2023generative,  zhu2023denoising}, on the other hand, generate data by simulating a gradual transformation of noise into structured data, showing significant promise in image synthesis tasks, though at the cost of computational inefficiency. This thesis focuses on addressing the challenges of computational efficiency and expressivity within NF models through novel architectural advancements and the application of Image Generative Models in real-world domains, alongside exploring the broader landscape of generative modeling techniques for practical tasks such as image generation, restoration, and enhancement.

\begin{figure*}[!ht]
\centering
	\includegraphics[width=0.999\textwidth]{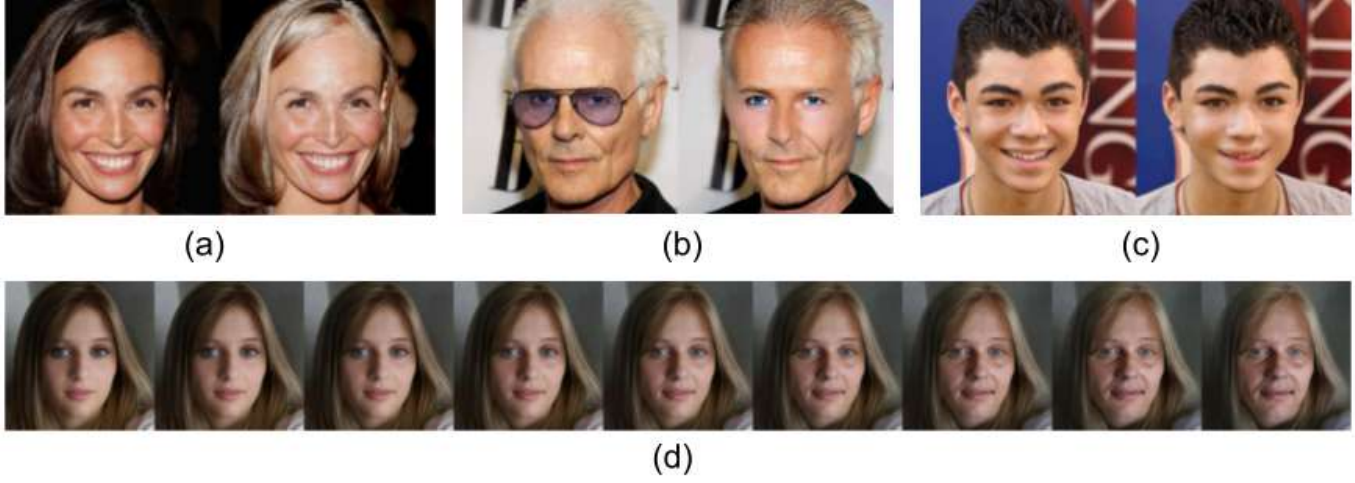}
	\caption{Images interpolation inspired by work of this thesis:  (a) regenerated image when using \emph{CInC Flow} to change hair color, (b) remove glasses, (c) change visage shape. (d) Result of gradually modifying the age parameter, original image: middle.}
\end{figure*}

\section{Normalizing Flow Models}
Advances in Flow models \cite{glow} have revolutionized generative modeling, with various architectures emerging to tackle the challenge of synthesizing realistic data. Flow models \cite{emerging,keller2021self} represent a distinctive approach that offers unique advantages through invertible transformations and direct maximum likelihood training. Unlike GANs and VAEs, which rely on adversarial or variational training techniques, NFs establish a one-to-one mapping between images and their latent representations, enhancing interpretability and control in generation tasks.

The evolution of flow-based models has been marked by significant milestones, particularly the introduction of Glow \cite{glow} with its innovative 1×1 invertible convolutions. While this advancement demonstrated the potential for efficient parallel computations invariant to spatial dimensions, extending these capabilities to more expressive k×k convolutions remained a crucial challenge. This thesis addresses this limitation through multiple contributions: CInC Flow \cite{nagar2021cinc}, which optimizes parameter utilization in padded 3×3 convolutions, and Inverse Flow, which leverages channel independence within convolution matrices for accelerated sampling without sacrificing model complexity. These innovations represent substantial progress toward achieving computational efficiency and high expressivity in flow-based generative modeling.

\section{Challenges in Normalizing Flow Models}

While powerful and flexible for learning complex distributions, NF models present several challenges that can hinder their practical implementation and application. These challenges can be categorized into the following key areas:

\paragraph{Training Instability:} Training NFs can be unstable due to poor initialization, inappropriate learning rates, or numerical instabilities during optimization. Since NFs rely on the invertibility and differentiability of the transformation between the base and target distributions, minor errors in the Jacobian determinant calculation or gradient propagation can significantly affect the model's convergence.

\paragraph{Computational Complexity:} NFs require the Jacobian determinant's computation for each flow transformation, which can be computationally expensive, especially for high-dimensional data. In some cases, these calculations grow quadratically or cubically with the input size, making it difficult to scale to large datasets or complex models. Designing flows with tractable Jacobian determinants remains an active area of research, with methods such as \textit{masking} or \textit{autoregressive flows} attempting to alleviate the burden.

\paragraph{Interpretability and Understanding:} NFs, especially intense ones with complex architectures, can be challenging to interpret. Understanding the learned transformation from the base distribution to the target distribution requires understanding the underlying mapping, which is not always straightforward. This lack of interpretability can be a barrier when these models are applied in domains where understanding the learned transformations is crucial, such as healthcare or finance.

\paragraph{Flow Design and Flexibility:} While NFs are highly flexible, designing the appropriate flow structure for a given task can be challenging. The choice of transformation functions limits the expressiveness of the flow, and designing flows that can capture highly complex data distributions without over-fitting or becoming intractable is still an open problem. Research into more flexible architectures, such as \textit{invariant} or \textit{parameterized} flows, is ongoing to address this limitation.

\paragraph{Scalability to High Dimensions:} NFs often face challenges when scaling to high-dimensional data. As the dimensionality of the input increases, the complexity of learning an invertible transformation also increases. Moreover, ensuring that the model generalizes well to high-dimensional spaces while maintaining tractable computations is a significant challenge. Some recent advancements focus on incorporating structured flows and hierarchical models to address the scalability issue in high-dimensional domains.

\paragraph{Theoretical Guarantees:}
While NFs are empirically adequate in many tasks, theoretical guarantees about their performance are limited. It is still not fully understood under what conditions NFs will optimally approximate a given target distribution or whether they can universally represent all distributions. Further theoretical work is needed to establish stronger foundations for NF models and their limitations.

\paragraph{Addressing Normalizing Flow Challenges in This Thesis:} We address several critical challenges in NF models, which, while powerful for exact likelihood estimation and invertibility, encounter practical obstacles in implementation and application. To counteract training instability, often caused by poor initialization and gradient propagation, the work introduces innovative architectural improvements like the CInC Flow, which enhances parameter efficiency in convolutional transformations. Computational complexity, mainly due to the demanding Jacobian determinant calculations, is addressed by methods that lower dimensional requirements, including multi-scale architectures and strategies that leverage channel independence. To improve interpretability and flexibility, this thesis advances NFs architectures that optimize the balance between model expressivity and ease of understanding. Additionally, structured flows enhance scalability in high-dimensional spaces, enabling the model to handle larger inputs effectively. These contributions collectively advance the theoretical foundation and practical utility of NFs, with demonstrated applications in image interpolation, faster sampling, and complex real-world modeling tasks.

\section{Applications of Image Generative Models}
Beyond theoretical advancements, our work demonstrates the practical utility of generative models across diverse applications. In agricultural quality assessment \cite{nagar2021automated}, we developed a computer vision-based system that combines Conditional Generative Adversarial Networks (CGANs) to address data scarcity challenges in seed quality evaluation. For autonomous driving applications, we introduced \emph{Pvt-IDD}, the first publicly available Indian dataset annotated for faces and license plates, addressing critical privacy preservation needs in real-world scenarios.

We pioneered an unsupervised machine learning framework for geological mapping that integrates stacked autoencoders with k-means clustering \cite{nagar2024rs_geo}. This approach, evaluated across multiple multispectral remote sensing datasets, demonstrates superior performance in identifying potentially mineralized areas through efficient dimensionality reduction and feature learning. 

Furthermore, Much of our research focuses on advancing image restoration techniques through the use of diffusion models. We developed a comprehensive approach that begins with leveraging pre-trained weights and systematically evolves through fine-tuning processes, particularly in adapting StableSR for various restoration tasks \cite{nagar2023adaptation}. Our work categorizes ten distinct distortion types and their relationships with training strategies and degradation scenarios. Through extensive comparative analysis of existing super-resolution and restoration methods, we established a precise taxonomy of their capabilities and limitations. The evaluation framework we developed encompasses multiple metrics across three essential tasks: super-resolution, deblurring, and inpainting, demonstrating powerful results when training on specialized datasets tailored to specific degradation types.

\section{Contributions}
This thesis proposal explores two interconnected themes: advancing the theoretical foundations of generative models through more efficient and expressive architectures and demonstrating their practical impact across multiple domains. In Part II (Chapter \ref{chap:cinc},  \ref{chap:inv_flow}), the proposed work contributes to the fundamental understanding of NFs, and Chapter \ref{chap:affine_sr} demonstrates the adaptation of NF layers to make the Diffusion model faster and efficient. Part IV (Chapter \ref{chap:seeds}, \ref{chap:pvt_idd}, \ref{chap:geo_vae}, \ref{chap:art_restore}, and \ref{chap:missing_sign}) showcases the potential of Generative Models and Computer Vision to solve real-world problems in agriculture, privacy-preserving in autonomous driving, missing traffic-sign detection, geological mapping, cultural heritage preservation, and predict spin-state energetics of transition metal complexes using machine learning models. The frameworks and methodologies developed here establish new benchmarks for computational efficiency and practical utility in generative modeling applications.

\paragraph{The rest of this thesis proposal is organized as follows:}

\begin{itemize}
    \item \textcolor{blue}{Chapter \ref{chap:cinc}:} Propose invertible n×n convolution and a new coupling layer, Quad-coupling layer, for NF models.
    
    \item \textcolor{blue}{Chapter \ref{chap:inv_flow}:} Fast backpropagation algorithm for inverse of convolution with GPU implementation and a multi-scale architecture, Inverse-Flow, accelerating sampling in NF models.

    \item \textcolor{blue}{Chapter \ref{chap:affine_sr}:} Present Affine-StableSR framework, integrating encoder and decoder structures with affine-coupling layers and pre-trained weights of the Stable Diffusion model.
    
    \item \textcolor{blue}{Chapter \ref{chap:seeds}:} A novel computer vision-based automated system for corn quality testing and reducing agricultural produce.
    
    \item \textcolor{blue}{Chapter \ref{chap:pvt_idd}:}  Image Anonymization for Street Scenes Datasets using Image Inpainting.
    
    \item \textcolor{blue}{Chapter \ref{chap:geo_vae}:} A framework that combined stacked autoencoders with k-means clustering to generate geological maps and a novel dimensionality reduction and clustering method.
        
    \item \textcolor{blue}{Chapter \ref{chap:art_restore}:} Adaptation of image super-resolution method encompassing both image super-resolution and restoration domains.

    
    \item \textcolor{blue}{Chapter \ref{chap:missing_sign}:} Missing traffic sign detection using computer vision.

    
    \item \textcolor{blue}{Chapter \ref{chap:summary}}: Summary and future directions for future advancements in efficient generative modeling and its applications, contributing significantly to fields of computer vision, image enhancement, and image generation/sampling.
\end{itemize}


\part{Fast \& Efficient Normalizing Flow Models}
 \chapter{Invertible  3 x 3 Convolution layer and Normalizing Flow}\label{chap:cinc}

\section{Introduction}\label{sec:intro}
The availability of large datasets has resulted in improved machine-learning solutions for more complex problems using Convolution Neural Networks (CNNs) model. However, supervised datasets are expensive to create. Hence, unsupervised methods like generative models are increasingly being worked on. The proposed generative models can be broadly categorized under Likelihood-based methods and GANs. For example, an optimization algorithm could directly minimize the former's negative log-likelihood of the unsupervised examples. At the same time, in the latter, the loss function is modeled as a discriminator network that is trained alternately. Hence, likelihood-based methods directly optimize the probability of examples. In contrast, in GANs, the optimized function is implicit and complex to reason about.


An essential type of Likelihood-based Generative model is normalizing flow-based models. Normalizing flow-based models transform a latent vector, usually sampled from a continuous distribution like the Gaussian, by a sequence of invertible functions to produce the sample. Hence, even though the latent vector distribution is simple, the sample distribution could be highly complex, provided we use an expressive set of invertible transformations. Also, the invertibility of the model implies that one can find the exact latent vector corresponding to an example from a dataset by means of a one-to-one mapping. All other approaches to generative modeling can compute the latent vector, for example, only approximately.

The ability of a normalizing flow-based model to express complex real-world data distributions depends on the expressive power of the invertible transformations used. In supervised models in vision tasks, complex, multilayered CNNs with different window sizes are used. CNNs with larger window sizes help in the spatial mixing of information about the images, resulting in expressive features. Glow used invertible 1$\times$1 convolutions to build normalizing flow models \cite{kingma2018glow}. For a 1$\times$1 convolution (if it is invertible), the inverse is also a 1$\times$1 convolution. We show that this approach does not generalize to larger window sizes. In particular, the inverse of an invertible 3$\times$3 convolution necessarily depends on all the feature vector dimensions, unlike CNNs, which only require local features.

Emerging convolutions proposed a way of inverting convolutions with large window sizes \cite{hoogeboom2019emerging}. The inverse is not a convolution and is computed by a linear equation system that can efficiently be solved using back substitution. However, they required 2 CNN filters to obtain an invertible convolution. Hence, every effective invertible convolution is required to do two convolutions back to back. 

We propose a simple padding approach to obtain invertible convolutions, which only uses a single convolutional filter. Furthermore, we can give a characterization (necessary and sufficient conditions) for the convolutions to be invertible. This allows us to optimize the space of invertible convolutions directly during training. We have compared our method with Emerging convolution (\cite{hoogeboom2019emerging}) and Autoregressive convolution (\cite{germain2015made}). Code available on GitHub\footnote{\label{footnote}\url{https://github.com/Naagar/Normalizing_Flow_3x3_inv}}.

\paragraph{Main Contributions.}
\begin{itemize}
    \item We give necessary and sufficient conditions for a 3$\times$3 convolution to be invertible by making some modifications to the padding (see Section \ref{sec:inv-conv}).
    \item We also propose a more expressive coupling mechanism called Quad-coupling (see Section \ref{sec:quad-coupling}).
    \item We use our characterization and Quad-coupling to train flow-based models that give samples of similar quality as previous works while improving upon the run-time compared to the other invertible 3$\times$3 convolutions proposed (see Section \ref{sec:results_2}).
\end{itemize}

\section{Related works}\label{sec:related_work}
\paragraph{Normalizing flows.}
A normalizing flow aims to model an unknown data distribution (\cite{kingma2018glow,cubic_spline_flow,neural_spline_flow}), that is, to be able to sample from this distribution and estimate the likelihood of an element for this distribution.

To model the probability density $p$ of a random variable $x$, a normalizing flow applies an invertible change of variable $x=g_\theta(z)$ where $z$ is a random variable following a known distribution, for instance $z\sim \mathcal{N}(0, I_d)$ where $d$ is the dimension. Then, we get the probability of $x$ by applying the change of variable formula
\[  p_\theta\left(x\right)= p\left(f_\theta(x)\right)\left(\left|\frac{\partial f_\theta\left(x\right)}{\partial x^T}\right| \right) \]
where $f_\theta$ denotes the inverse of $g_\theta$ and $\left|\frac{\partial f_\theta\left(x\right)}{\partial x^T}\right|$ its Jacobian.

The parameters $\theta$ are learned by maximizing the actual likelihood of the dataset. At the same time, the model is designed so that the function $g_\theta$ can be inverted and have its Jacobian computed in a reasonable amount of time.

\paragraph{Glow.}
RealNVP \cite{rezende2015variational} defines a normalizing flow composed of a succession of invertible steps. Each of these steps can be decomposed into layers of steps. Improvements for some of these layers were proposed in later articles \cite{kingma2018glow}, \cite{hoogeboom2019emerging}.

\emph{Actnorm}: The actnorm layer performs an affine transformation similar to batch normalization. First, its weights are initialized so that the output of the actnorm layer has zero mean and unit variance. Then, its parameters are learned without any constraint.

\emph{Permutation}: RealNVP proposed to use a fixed permutation to shuffle the channels, as the coupling layer only acts on half of the channels. Later, \cite{kingma2018glow} replaced this permutation with a 1$\times$1 convolution in Glow. These can easily be inverted by inverting the kernel. Finally, \cite{hoogeboom2019emerging} replaced this 1 x 1 convolution with the so-called emerging convolution. These have the same receptive convolution with a kernel of arbitrary size. However, they are computed by convolving with two successive convolutions whose kernel is masked to help the inversion operation.

\emph{Coupling layer}: The coupling layer provides flexibility to the architecture. The Feistel scheme (\cite{feistelGeneralized}) inspires its design. They are used to build an invertible layer from any given function $f$. Here, $f$ is learned as a convolutional neural network. 
\[ y = [y_1, y_2], \quad y_1=x1, \quad y_2=(x_2 + f(x_1))*\exp(g(x_1)) \]
We get $x_1$ and $x_2$ by splitting the input $x$ along the channel axis.

\paragraph{Invertible Convolutional Networks.}
Complementary to normalizing flows, some work has been done designing more flexible invertible networks. For example, \cite{revnet} proposed reversible residual networks (RevNet) to limit the memory overhead of backpropagation, while (\cite{JacobsenBZB19}) built modifications to allow an explicit form of the inverse, also studying the ill-conditioning of the local inverse. \cite{ho2019flow++} proposed a flow-based model that is a non-autoregressive model for unconditional density estimation on standard image benchmarks

\emph{Deconvolution vs Inverse of Convolution:} There is a fundamental difference between a \textbf{deconvolution} (also known as a transposed convolution) and a true \textbf{inverse of a convolution} as used in our work. A deconvolution is a learned, approximate inverse typically used in the decoder part of an autoencoder or a generator in a GAN. It learns a kernel via backpropagation to up-sample a feature map, but it does not guarantee a perfect reconstruction of the original input. In contrast, an "inverse of a convolution" in our context refers to a mathematically exact inverse operation. This is only possible when the forward convolution is explicitly designed to be invertible, allowing the original input to be perfectly recovered from the output, as illustrated in Figure~\ref{fig:deconv_vs_invconv}.

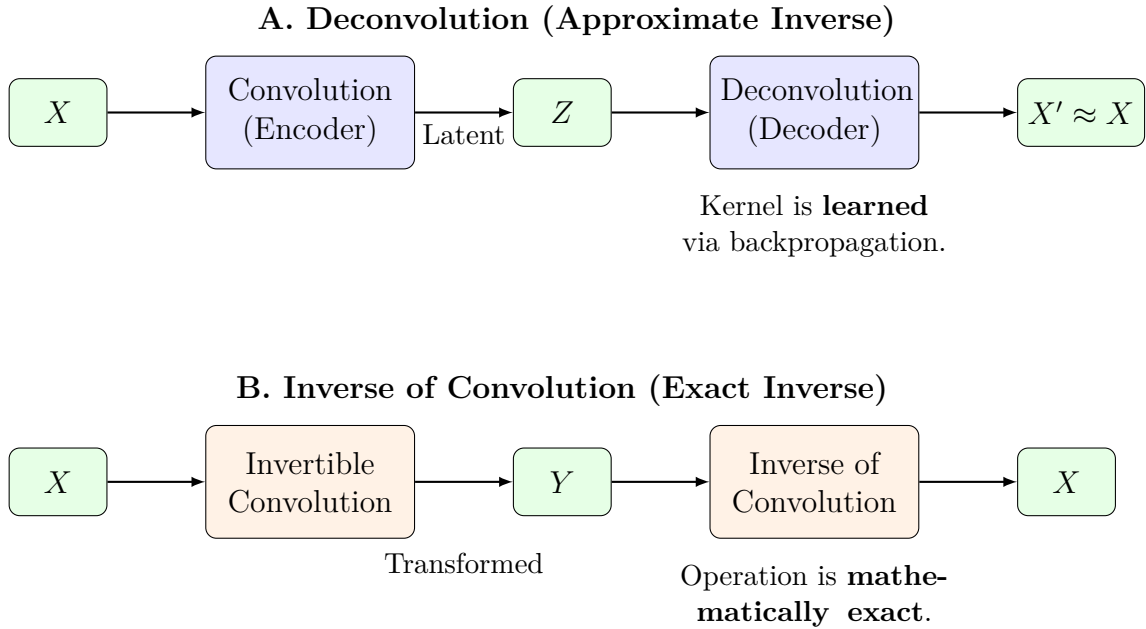
\begin{figure}[ht!]
        \centering
        \begin{tikzpicture}[
            node distance=1cm and 1.3cm,
            block/.style={rectangle, draw, fill=blue!10, text width=2.5cm, text centered, rounded corners, minimum height=1.5cm},
            data/.style={rectangle, draw, fill=green!10, text centered, rounded corners, minimum height=0.9cm, minimum width=1.3cm},
            arrow/.style={-latex, thick},
            title/.style={font=\bfseries}
        ]
            \node[data] (x1) {$X$};
            \node[block, right=of x1] (conv) {Convolution (Encoder)};
            \node[data, right=of conv] (z) {$Z$};
            \node[block, right=of z] (deconv) {Deconvolution (Decoder)};
            \node[data, right=of deconv] (x_prime) {$X' \approx X$};
            
            \draw[arrow] (x1) -- (conv);
            \draw[arrow] (conv) -- node[below, font=\small] {Latent} (z);
            \draw[arrow] (z) -- (deconv);
            \draw[arrow] (deconv) -- (x_prime);
            
            \node[below=0.2cm of deconv, text width=5cm, text centered, font=\small] {Kernel is \textbf{learned} via backpropagation.};
            
            \node[fit=(x1)(x_prime), inner sep=0.1cm, label={[yshift=0.3cm]above:\textbf{A. Deconvolution (Approximate Inverse)}}] (frame1) {};

            \node[data, below=4cm of x1] (x2) {$X$};
            \node[block, fill=orange!10, right=of x2] (inv_conv) {Invertible Convolution};
            \node[data, right=of inv_conv] (y) {$Y$};
            \node[block, fill=orange!10, right=of y] (inv_conv_inv) {Inverse of Convolution};
            \node[data, right=of inv_conv_inv] (x_recon) {$X$};

            \draw[arrow] (x2) -- (inv_conv);
            \draw[arrow] (inv_conv) -- node[below=0.8cm, font=\small] {Transformed} (y);
            \draw[arrow] (y) -- (inv_conv_inv);
            \draw[arrow] (inv_conv_inv) -- (x_recon);
            
            \node[below=0.2cm of inv_conv_inv, text width=5cm, text centered, font=\small] {Operation is \textbf{mathematically exact}.};

             \node[fit=(x2)(x_recon), inner sep=0.1cm, label={[yshift=0.3cm]above:\textbf{B. Inverse of Convolution (Exact Inverse)}}] (frame2) {};

        \end{tikzpicture}
        \caption{(A) a learned Deconvolution, which provides an approximate reconstruction, and (B) a mathematical Inverse of a Convolution, which provides an exact reconstruction.}
        \label{fig:deconv_vs_invconv}
\end{figure}

\emph{Invertible 1$\times$1 Convolution:}  
\cite{kingma2018glow} proposed the invertible 1$\times$1 convolution replacing fixed permutation (\cite{realNVP}) that reverses the ordering of the channels. \cite{hoogeboom2019emerging} proposed a normalizing flow method to do the inversion of 1$\times$1 convolution with some padding on the kernel and two distinct auto-regressive convolutions, which also provide a stable and flexible parameterization for invertible 1$\times$1 convolutions. The weights of a 1x1 convolution with c input channels and c output channels, is given by a cxc matrix W. The c dimensional vector at every pixel of the input is multiplied by this matrix to get the corresponding c dimensional output vector at every pixel. What we mean by invertible here, is that W is invertible as a matrix (as in linear algebra).

\emph{Invertible n$\times$n Convolution:} Reformulating n$\times$n convolution using the invertible shift function proposed by \cite{glow_nxn} to decrease the number of parameters and remove the additional computational cost while keeping the range of the receptive fields. In our proposed method, there is no need for the reformulation of standard convolutions. \cite{hoogeboom2019emerging} proposed two different methods to produce the invertible convolutions : (1) Emerging Convolution and (2) Invertible Periodic Convolutions. Emerging requires two autoregressive convolutions to do a standard convolution. Still, our method requires only one convolution compared to the method proposed by \cite{hoogeboom2019emerging} and increases the flexibility of the invertible n$\times$n convolution. 
\section{Our approach}\label{sec:our_approach}
We propose a novel approach for constructing invertible 3$\times$3 convolutions and coupling layers for normalizing flows. We propose two modifications to the existing layers used in previous normalizing flow models:
\begin{itemize}
	\item Convolution layer: instead of using 1$\times$1 convolutions or emerging convolutions, we propose to use standard convolutions with a kernel of any size with a specific padding. 
	\item Coupling layer: We propose to use a modified version of the coupling layer designed to have a bigger receptive field.
\end{itemize}
We also show how invertibility can be used to manipulate images semantically.

\paragraph{Invertible 3x3 Convolution} \label{sec:inv-conv}
We give necessary and sufficient conditions for an arbitrary convolution with some simple modifications on the padding to be invertible. Moreover, an efficient back substitution algorithm can also compute the inverse.

\begin{definition}[Convolution]
The convolution of an input $X$ with shape $H\times W \times C$ (H: height, W: width, C: channels, X: input image) with a kernel $K$  with shape $k\times k\times C\times C$ is $Y=X*K$ of shape $(H-k+1)\times (W-k+1)\times C$ then first channel of which is equal to
\begin{equation}
   Y_{i,j,c_o} = \sum_{l,h < k}\sum_{c_i=1}^{C}I_{i+l,j+k,c_i}K_{l,k,c_i,c_o}
\end{equation}.
\end{definition}

In this setting, the output $Y$ is smaller than the input to prevent the input from being padded before applying the convolution.
\begin{definition}[Padding]
Given an image $I$ with shape $H\times W \times C$, the $(t-top, b-bottom, l-left, r-right)$ padding of $I$ is the image $\hat{I}$ of shape $(H + t + b)\times (W+l+r) \times C$ defined as
    \begin{equation}
        \hat{I}_{i,j,c} = \begin{cases}
        I_{i-t,j-l,c}\ &\text{if } i-t < H\text{ and } j-l < W\\
        0&\text{otherwise}
        \end{cases}
    \end{equation}
\end{definition}

\begin{figure*}[ht]
	\centering
		\includegraphics[width=0.99\linewidth]{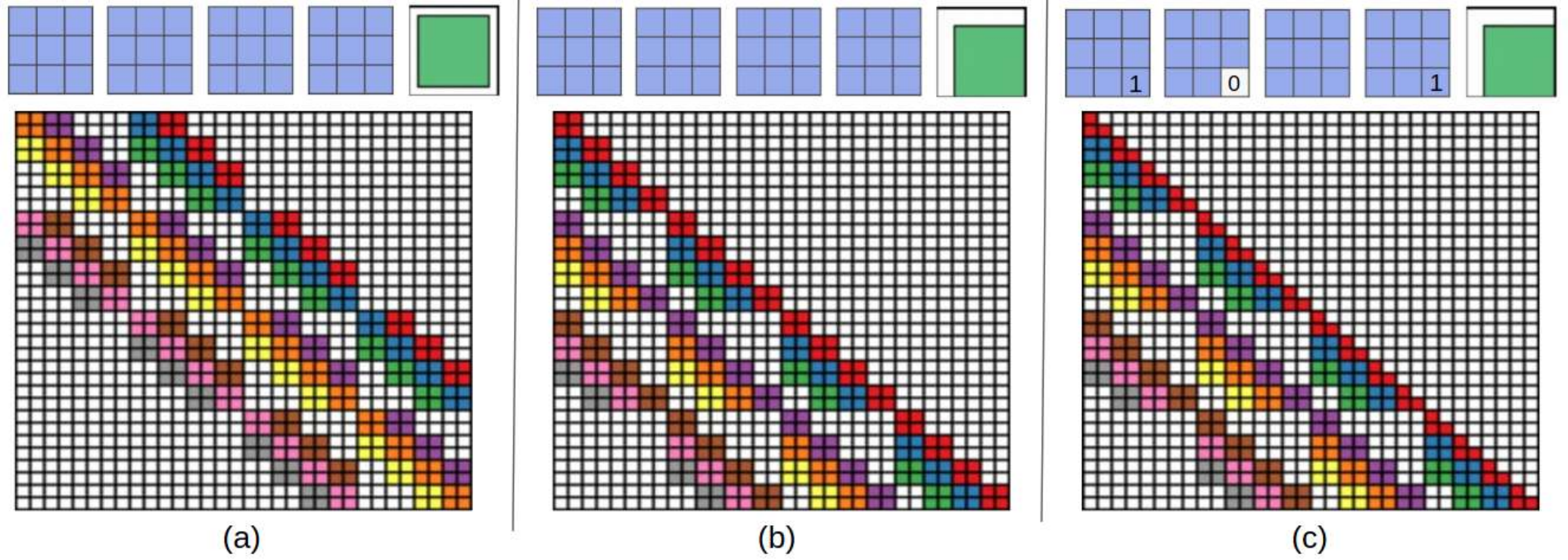}	
	\caption{(a).Top: The first four are the kernel matrix, and the fifth is the input matrix with the \emph{standard} padding that gives the bottom convolution matrix. Bottom: the convolution matrix corresponding to a convolution with a kernel of size three applied to an input of size $4\times4$, padded on both sides, and with two channels. Zero coefficients are drawn in white; other coefficients are drawn using the same color if applied to the same spatial location, albeit on different channels. (b) Top: an \emph{alternative} padding scheme that results in a block triangular matrix $M$, Bottom: The matrix corresponding to a convolution with a kernel of size three applied to an input of size $4\times4$ padded only on one side and with two channels. (c) Top: an \emph{masked alternative} padding scheme that results in a triangular matrix $M$, Bottom: the matrix corresponding to a convolution with a kernel of size three applied to an input of size $4\times4$ padded only on one side and with two channels. One of the weights of the kernel is masked. Note that the equivalent matrix $M$ is \emph{triangular}.}
	\label{fig:matrix}
\end{figure*}

As zero padding does not add any bias to the input, the convolution between a padded input $\hat{I}$ and a kernel $K$ is still a linear map between the input and the output. As such, it can be described as matrix multiplication.

An image $I$ of shape $(H,W,C)$ can be seen as an vector $\vec{I}$ of $\mathbb{R}^{H\times W\times C}$. In the rest of this chapter, we will always use the following basis 
$I_{i,j,c} = \vec{I}_{c + Cj + CHi}$.  For any index $i\leq HWC$, let $(i_y,i_x,i_c)$ denote the indexes that satisfy $\vec{I}_i=I_{i_y,i_x,i_c}$. Note that $i < j$ iff $(i_y,i_x,i_c) \prec (j_y,j_x,j_c)$ where $\prec$ denotes the lexicographical order over $\mathbb{R}^3$. If $C=1$ this means that the pixel $(j_y,j_x)$ is on the right or below the pixel  $(i_y,i_x)$.

\begin{definition}[Matrix of a convolution.]
Let $K$ be a kernel of shape $k\times k \times C \times C$. The matrix of a convolution of kernel $K$ with input $X$ of size $H\times W\times C$ with padding $(t,b,l,r)$ is a matrix describing the linear map $X\mapsto \hat{X}*K$.
\end{definition}

\paragraph{Characterization of invertible convolutions:}
We consider convolution with top and left padding only. We give necessary and sufficient conditions for such convolutions to be invertible. 
Let $K$ be the kernel of the convolution with shape $3\times 3 \times N \times N$ where $3\times3$ is the window size, and $N$ is the number of channels. Note that the number of input channels should be equal to the number of output channels to be invertible.
\begin{lemma}\label{lem:lower-triang} Let $y$ = $M\hat{x}$, 

\textbf{$M$} is a lower triangular matrix with all diagonal entries $=K_{3,3}$
\end{lemma}
Where the matrix $M$ is that produces the equivalent result when multiplied with a vectorized input ($\hat{x}$).


\begin{proof}
Consider any entry in the upper right half of $M$. That is $(i,j)$ such that $i < j$ according to the ordering given in the definition of $M$. $M_{i,j}$ is nothing but the scalar weight that needs to be multiplied by the $j$th pixel of input when computing the $i$th pixel of the output. The linear equation relating these two variables is as follows:
$$ y_i = \sum_{l=0}^3\sum_{k=0}^3K_{3-l,3-k}x_{i_x-l,i_y-k} $$
From this equation follows that if $j_x>i_x$ or $j_y>i_y$, then the $i$th pixel of the output does not depend on the $j$th pixel of the input, and thus $M_{i,j}=0$. This also justifies that all diagonal coefficients of $M$ are equal to $K_{3,3}$
\end{proof}

We first describe our conditions for the case when $N=1$. We prove the following theorem.
\begin{theorem}[Characterization for $N=1$] \label{thm:inv-conv-1}
$$M \text{ is invertible iff } K_{3,3} \neq 0.$$
\end{theorem}

\begin{proof}
The proof of the theorem uses Lemma \ref{lem:lower-triang}.  Since $M$ is lower triangular, the determinant is nothing but the product of diagonal entries, which is $=K_{3,3}^{h*w}$ where $h,w$ are the dimensions of the input/output image. 
\end{proof}
At its core, the convolution layer is a linear operation. However, we have no guarantees regarding its invertibility. The result $z$ of the convolution of input $x$ with kernel $k$ can be expressed as the product of $x$ with a matrix $M$. When zero-padding is used around the input so that $x$ and $z$ have the same shape, the matrix $M$ is not easily invertible because the determinant of $M$ can be zero (see matrix $M$ in Figure \ref{fig:matrix}(a) caption for it and Lemma 1).

However, when padding only on two sides (left and top), the corresponding $M$ is blocked triangular (see Figure \ref{fig:matrix}(b)). To further ensure invertibility and speed up the inversion process, we also mask part of the kernel so that the matrix corresponding to the convolution is triangular; see Figure \ref{fig:matrix}(c) and for more details, which $K_{(n,n)}$ we need to mask, see paragraph \ref{sec:masking}. In this configuration, the Jacobian of the convolution can also be easily computed. For further details of the padding of the input, see Figure-\ref{sec:padding}.

\paragraph{Input padding}\label{sec:padding}
To make sure the matrix ($M$) is block triangular, padding of $(k+1)/2$ on top and left of the input ($x$) is applied, where $k$ is the size of the kernels. 
\paragraph{Masking of kernels}\label{sec:masking}
Considering the assumption that the number of input channels equals the number of output channels ($N$), masking the kernels depends only on the number of channels ($N$); let the kernel size be $k\times k$. The $K^{a,b}_{k,k}$ are the entries corresponding to block ($N\times N$) diagonals entire in the convolution matrix ($M$), where $a$ and $b$ are  $a, b \in {1, 2, ..., N}$.   $K^{a,b}_{k,k}$ is a square matrix ($C$) of size $N\times N$) and the diagonal entries of the block of matrix $M$ are the diagonal of $C$ and to ensure the invariability of the $M$, we have to set all the entries $C_{a,b}$ to zero when $a > b$ and one when $a = b$.

Comparison of our method with the existing invertible Normalizing flow methods \label{sec:comparison} in Figure \ref{fig:cincflow_compare}.

\begin{figure}[!ht]
    \centering
    \includegraphics[width=0.99\textwidth]{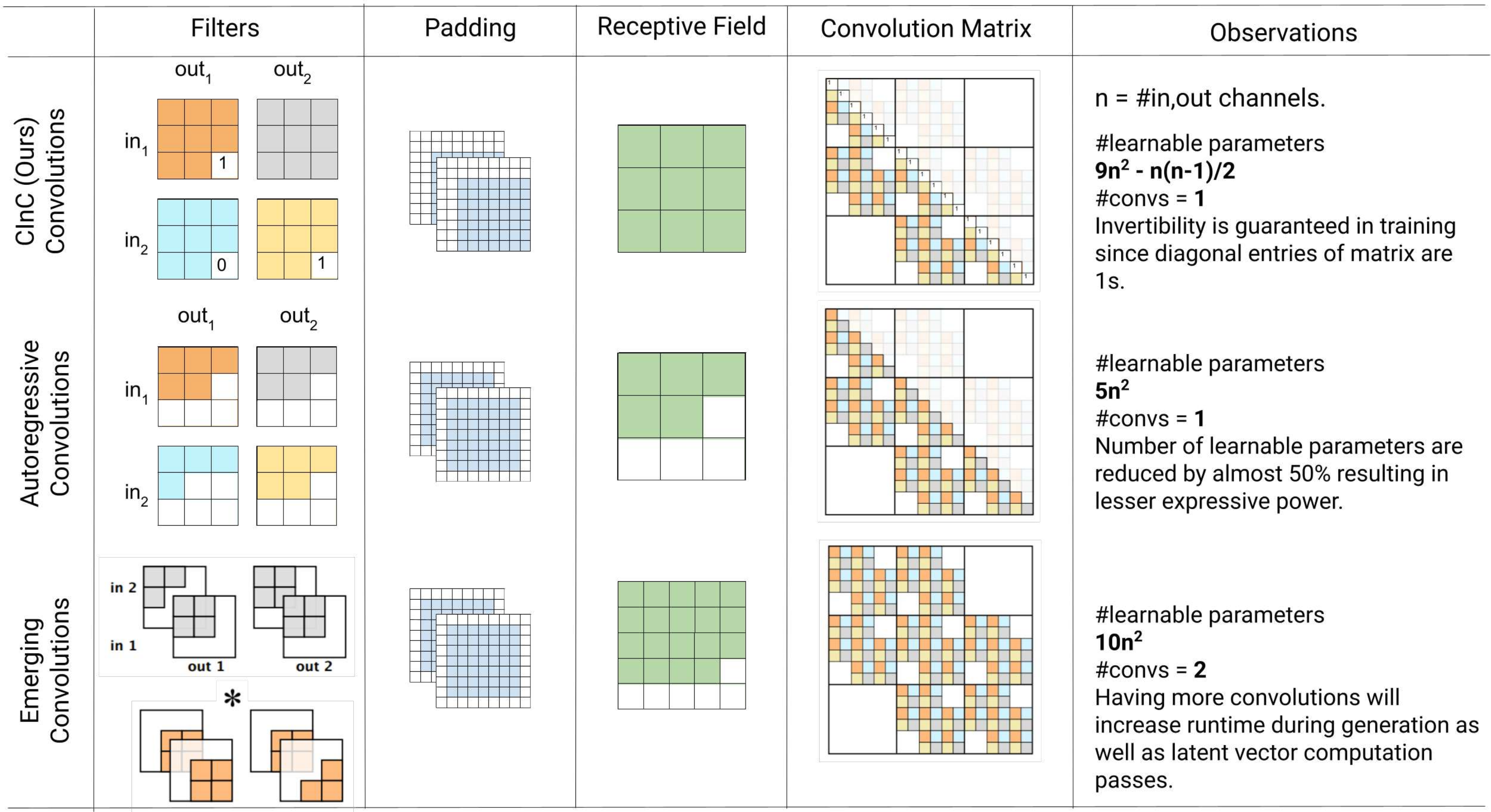}
    \caption{CInC Flow \cite{nagar2021cinc} comparison of the speed and utilization of parameters with Autoregressive convolutions and Emerging convolutions.}
    \label{fig:cincflow_compare}
\end{figure}

\paragraph{Quad-coupling}
\label{sec:quad-coupling}
The coupling layer is used to have some flexibility, as its functions can be of any form. However, it only combines the effects of half channels. 
To overcome this issue, we designed a new coupling layer inspired by generalized Feistel (\cite{feistelGeneralized}) schemes. Instead of dividing the input $x$ into two blocks, we divide it into four $x = \left[x_1,x_2,x_3,x_4\right]$ along the feature axis. Then we keep $x_1$ unchanged  and use it to modify the other blocks in an autoregressive  manner (see Figure \ref{fig:coupling}):

\begin{align}
    y_1 =& x_1\\
    y_2 =& (x_2 + f_1(x_1)) * \exp(g_1(x_1))\\
    y_3 =& (x_3 + f_2(x_1,x_2)) * \exp(g_2(x_1,x_2))\\
    y_4 =& (x_4 + f_3(x_1,x_2,x_3)) * \exp(g_3(x_1,x_2,x_3))
\end{align}

where $(f_i)_{i\leq 3}$ and $(g_i)_{i\leq 3}$ are learned. The layer output is obtained by concatenating the $(y_i)_{i\leq 4}$.

\begin{figure}[ht!]
    \centering
    \includegraphics[width=0.35\linewidth]{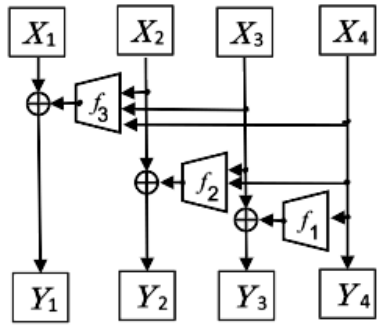}
       \caption{The Quad-coupling layer, each input block $X_i$ has the same spatial dimension as the input $X$ but only one-quarter of the channels. Each function $f_1$, $f_2$, and $f_3$ is a 3-layer convolutional network. $\bigoplus$ symbolizes a component-wise addition. The multiplicative actions are not represented here.}
    	\label{fig:coupling}
\end{figure}

\begin{figure}[ht!]
    \centering
    \includegraphics[width=0.75\linewidth]{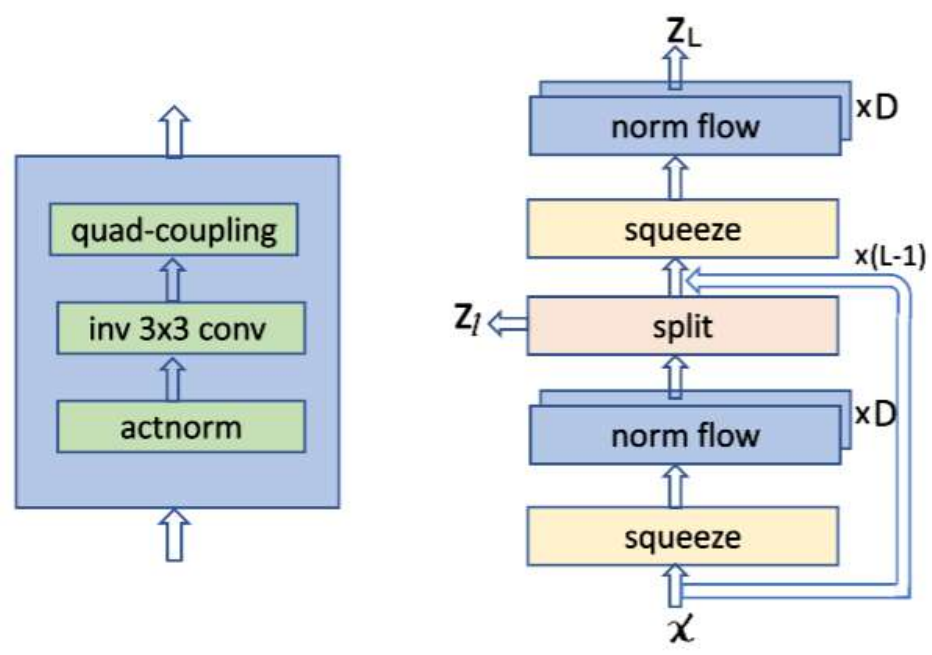}
       \caption{Overview of the model architecture. Left: proposed flow module: containing inv $3\times3$ convolution. Right: complete model architecture, where the flow module is now grouped. The squeeze module reorders pixels by reducing the spatial dimensions by half and increasing the channel depth by four. A hierarchical prior is placed on part of the intermediate representation using the split module as in (\cite{kingma2018glow}). $x$ and $z$ denote input and output. The model has L levels and D flow modules per level.}
    	\label{fig:norm_flow}
\end{figure}

\section{Experimental results}\label{sec:results_2}
The architecture is based on \cite{hoogeboom2019emerging}. We modified the emerging convolution layer to use our standard convolution. We also introduced the Quad-coupling layer in place of the affine coupling layer. Finally, we evaluate the model using various methods and provide images sampled from the model. See Figure \ref{fig:norm_flow} for a detailed architecture overview.

\paragraph{Training setting:} To train the model on CIFAR10, we used the three levels (L) and depth (D) of 32 and lr $0.001$ for the 500 epochs. To train on ImageNet32, $L = 3$, $D = 48$, lr $0.001$ for the 600 epochs and for ImageNet64, $L = 4$, $D = 48$, lr $0.001$ for the 1000 epochs. See Figure \ref{fig:norm_flow} for the model architecture.
\paragraph{Quantitative results:}
The comparison of the performance of  $3\times3$ invertible convolution with the emerging convolution \cite{hoogeboom2019emerging} for the CIFAR10 dataset in Table \ref{tab:cifar}. The performance of our layers was tested on CIFAR10 \cite{cifar10}, ImageNet \cite{imgnet_data}, as well as on the galaxy dataset \cite{galaxy}, see Table \ref{tab:performance}. We also tested our architecture on networks with a smaller depth ($D=4$ or $D=8$), see Table \ref{tab:smallnet}, which could be used when computational resources are limited, as their sampling time is much lower. In this case, using standard convolution and Quad-coupling offers a more considerable performance improvement than bigger models (see Table \ref{tab:smallnet}). 

\begin{table}[ht!]
    \centering
    \begin{tabular}{ccc}
    		\toprule
    		 Layer &  Emerging 3$\times$3 Inv. conv & Our 3$\times$3 Inv. conv \\
    		\hline
    		Affine & 3.3851 & 3.4209 \\
    		Quad & \textcolor{blue}{3.3612} & \textcolor{blue}{3.3879} \\
                \bottomrule
    	\end{tabular} 
    	\caption{Comparison of the performance in bits per dimension (BPD) achieved on the CIFAR10 dataset with different coupling architectures.}
    	\label{tab:cifar}
\end{table}    

\begin{table}[ht!]
    \centering
     \begin{tabular}{lcccc}
            \toprule
        Dataset & Glow & Emerging & $3\times 3$ & Quad\\
        \hline 
        CIFAR10 & 3.35 & 3.34 & \textcolor{blue}{3.3498} & \textcolor{blue}{3.3471} \\ 
        ImageNet32 & 4.09 & 4.09 & \textcolor{blue}{4.0140} & 4.0377 \\ 
        ImageNet64 & 3.81 & 3.81 & 3.8946 & 3.8514 \\ 
        Galaxy & --- & 2.2722 & 2.2739 & \textcolor{blue}{2.2591} \\ 
        \bottomrule
    \end{tabular} 
    \caption{Performance achieved on the CIFAR10 and ImageNet datasets after a limited number of epochs (500 for CIFAR10, 600 for ImageNet32, ImageNet64, and 1000 for Galaxy). Emerging results were obtained by using the code provided in \cite{hoogeboom2019emerging}, $3\times 3$ is replacing the emerging convolutions by our $3\times 3$ invertible convolutions, and Quad uses Quad-coupling on top of this.}
    \label{tab:performance}
\end{table}

\begin{figure}[ht!]
    \centering
    \includegraphics[width=0.8\linewidth]{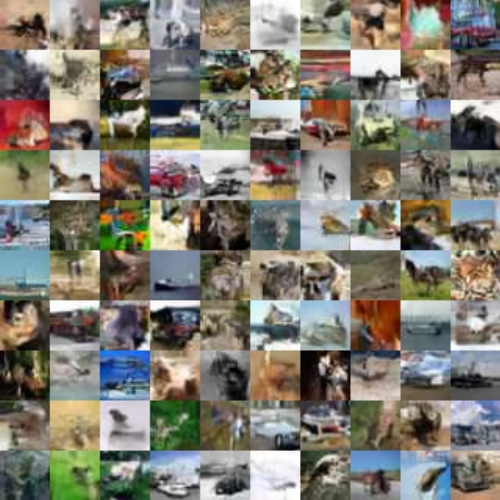}
    \caption{Sample images generated after training on the CIFAR10 dataset.}
    \label{fig:sample_cifar}
\end{figure}

\begin{table*}
\centering
    \begin{tabular}{lccccc}
        \toprule
        \multirow{2}{*}{Dataset} & \multicolumn{2}{c}{Emerging} & \multicolumn{2}{c}{Ours} & \multirow{2}{*}{Depth}  \\
               & BPD & ST & BPD & ST &   \\ 
        \hline
        CIFAR10 &  3.52 & \multirow{2}{*}{2.45} & \textcolor{blue}{3.49} & \multirow{2}{*}{\textcolor{blue}{1.31}} & \multirow{2}{*}{4} \\ 
        {Imagenet32} & 4.30 &  & \textcolor{blue}{4.25} &  & \\ 
        CIFAR10 &  3.47 & \multirow{2}{*}{4.94} & \textcolor{blue}{3.46} & \multirow{2}{*}{\textcolor{blue}{2.76}} & \multirow{2}{*}{8} \\ 
        Imagenet32 & 4.20 &   & \textcolor{blue}{4.18} &  & \\ 
        \bottomrule
    \end{tabular}
    \caption{BPD with smaller networks, when computational resources are limited. The performance is expressed in bits per dimension, and the sampling time (ST) is the time in seconds needed to sample 100 images. All networks were trained for 600 epochs.}
    \label{tab:smallnet}
\end{table*}

\begin{figure*}[ht]
    \centering
	\includegraphics[width=0.89\textwidth]{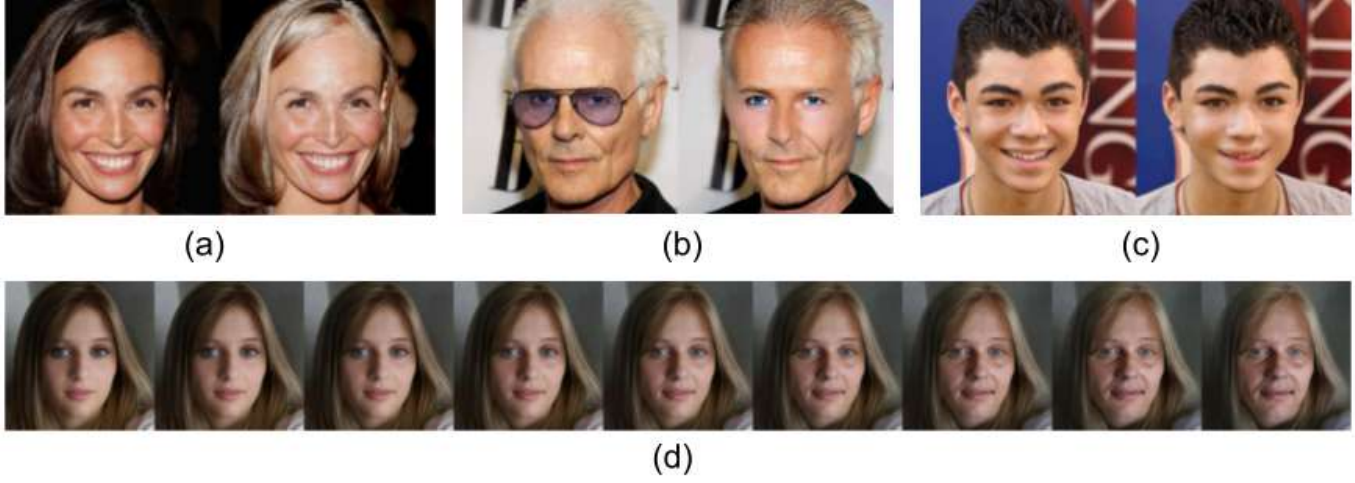}
	\caption{From left to right: the result obtained when using the network to change hair color (a), remove glasses (b), and visage shape (c). For each example, the original image is shown on the left. Fig.(d) Here, we can see the result of gradually modifying the age parameter. The original image is the fourth from the left (middle one).}
	\label{fig:midif}
\end{figure*}

\paragraph{Sampling Times:}
We compared our method's sampling time (Table \ref{tab:samplingtimes}) against Glow (\cite{kingma2018glow}) and Emerging Convolutions (\cite{hoogeboom2019emerging}). Our convolution still requires solving a sequential problem to be inverted and, as such, needs to be inverted on the CPU, unlike Glow, which can be inverted while performing all the computations on the GPU to show the relative computational cost of different architectures. This explains the gap between the sampling time of our model and Glow. However, it is roughly two times faster than emerging convolutions; this comes from the need to solve two inversion problems to invert one emerging convolution layer. The Quad-coupling layer does not affect sampling time too much.

\begin{table}[ht!]
\centering
    \begin{tabular}{lcccc}
        \toprule
        dataset & Glow & Emerging & $3\times 3$  & Quad \\ 
        \hline 
        CIFAR10 & 0.58 & 18.4 & 9.3 & 10.8 \\  
        Imagenet32 & 0.86 & 27.6 & 14.015 & 16.1 \\ 
        Imagenet64 & 0.50 & 160.72 & 82.04 & 84.06 \\ 
        \bottomrule
    \end{tabular}
    \caption{ Time to sample 100 images. Results were obtained with Glow running on a GPU and the other methods running on one CPU core. Seconds.}
    \label{tab:samplingtimes}
\end{table}

\paragraph{Interpretability results:}
To show the interpretability of our invertible network, we used the Celeba dataset (\cite{celeba}), which provides images of faces and attributes corresponding to these faces. Figure \ref{fig:sample_cifar} shows the randomly generated fake sample images for the CIFAR10 dataset. The covariance matrix between the attributes of images in the dataset and their latent representation indicates how to modify the latent representation of an image to add or remove features. Examples of such modifications can be seen in Figures \ref{fig:midif}(a, b, c) and \ref{fig:midif}(d).

\section{Summary}
In this work, we propose a new method for Invertible n$\times$n Convolution. Coupling layers solve two problems for normalizing flows: they have a tractable Jacobian determinant and can be inverted in a single pass. We propose a new type of coupling method, Quad-coupling. Our method shows consistent improvement over the Emerging convolutions method, and we only need a single CNN layer for every effective invertible convolution. This work shows that we can invert a convolution with only one effective convolution, and additionally, the inference time and sampling time improved notably. We show the inversion of 3$\times$3 convolution and the generalization of the inversion for the n$\times$n kernel. Furthermore, we demonstrate enhanced quantitative performance regarding log-likelihood on standard image modeling benchmarks.

 \chapter{Backpropogation Algorithm for Inverse of Convolution and Inverse-Flow Model}\label{chap:inv_flow}

\section{Introduction}
Large-scale neural network optimization using gradient descent is made possible due to efficient and parallel back-propagation algorithms \citep{bottou2010large}. Large models could not be trained on large datasets without such fast back-propagation algorithms. All operations for building practical neural network models need efficient back-propagation algorithms \citep{lecun2002efficient}. This has limited types of operations that can be used to build neural networks. Hence, it is important to design fast parallel backpropagation algorithms for novel operations that could make models more efficient and expressive.

Convolutional layers are very commonly used in Deep Neural Network models as they have fast parallel forward and backward pass algorithms \citep{lecun2002efficient}. The inverse of a convolution is a closely related operation with use cases in Normalizing Flows \citep{karami2019invertible}, Image Deblurring \citep{eboli2020end}, Sparse Blind Deconvolutions \citep{xu2014deep}, Segmentation, etc. However, the Inverse of a Convolution is not used directly as a layer for these problems since straightforward algorithms for the backpropagation of such layers are highly inefficient. Such algorithms involve computing the inverse of a very large-dimensional matrix. 

Fast sampling is crucial for Normalizing flow models in various generative tasks due to its impact on practical applicability and real-time performance \citep{papamakarios2021normalizing}. Rapidly producing high-quality samples is essential for large-scale data generation and efficient model evaluation in fields such as image generation, molecular design \citep{zang2020moflow}, image deblurring, and deconvolution. Normalizing flows have demonstrated their capability in constructing high-quality images \citep{kingma2018glow, meng2022butterflyflow}. However, the training and sampling process is computationally expensive due to the repeated need for inverting functions (e.g., convolutions). Existing approaches rely on highly constrained architectures and often impose limitations like diagonal, triangular, or low-rank Jacobian matrices and approximate inversion methods \citep{hoogeboom2019emerging, keller2021self}. These constraints restrict the expressiveness and efficiency of normalizing flow models. To overcome these limitations, fast, efficient, and parallelizable algorithms are needed to compute the inverse of convolutions and their backpropagation, along with GPU-optimized implementations. Addressing these challenges would significantly enhance the performance and scalability of Normalizing flow models. 

In this work, we propose a fast, efficient, and parallelized backpropagation algorithm for the inverse of convolution with running time $O(mk^2)$ on an $m\times m$ input image. We provide a parallel GPU implementation of the proposed algorithm  (together with baselines and experiments) in CUDA. Furthermore, we design \emph{Inverse-Flow}, using an inverse of convolution ($f^{-1}$) in forward pass and convolution ($f$) for sampling. Inverse-flow models generate faster samples than standard Normalizing flow models.  

In summary, our contribution includes:
\begin{enumerate}
    \item We designed a fast and parallelized backpropagation algorithm for the inverse of the convolution operation. (see section \ref{sec:bp_inv_conv})
    \item \sloppy Implementation of the proposed backpropagation algorithm for the inverse of convolution on the GPU CUDA. (see section \ref{sec:inv_flow})
    \item We propose a multi-scale flow architecture, \emph{Inverse-Flow}, for fast training of inverse of convolution using our efficient backpropagation algorithm and faster sampling with $k \times k$ convolution. (see section \ref{sec:inv_flow})
    
    \item Benchmarking of \emph{Inverse-Flow} and a small linear, 9-layer flow model on image datasets (MNIST, CIFAR10). (see section \ref{sec:if_results})
    
\end{enumerate}

\begin{figure}
    \centering
    \includegraphics[width=0.79\linewidth]{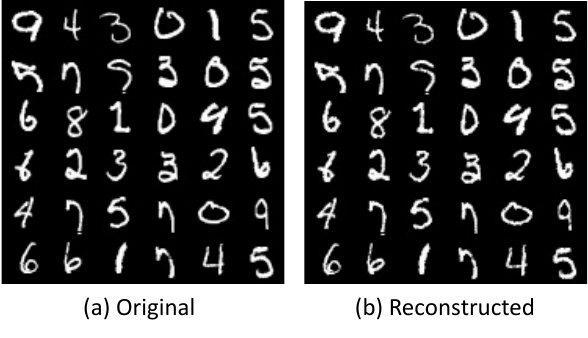}
    \label{fig:recon_img}
    \caption{a). Images from the MNIST dataset. b). Reconstructed images using an Inverse-Flow model based on the inv-conv layer for a forward pass.}
\end{figure}

\section{Related work}

\paragraph{Backpropagation for Inverse of Convolution:}
The backpropagation algorithm performs stochastic gradient descent and effectively trains a feed-forward neural network to approximate a given continuous function over a compact domain. Hoogeboom et al.\citep{hoogeboom2019emerging} proposed invertible convolution, Emerging, generalizing 1x1 convolution from Glow \citep{kingma2018glow}. \citep{finzi2019invertible} proposed periodic convolution with $k \times k$ kernels. Emerging convolution combines two autoregressive convolutions \citep{kingma2016improved}, and parallelization is not possible for its inverse. MaCow \citep{ma2019macow} uses four masked convolutions in an autoregressive fashion to get a receptive field of $3 \times 3$ standard convolution, which leads to slow sampling and training. To the best of our knowledge, this work is the first to propose a backpropagation algorithm for the inverse of convolution. Additionally, it is the first to utilize an inverse Normalizing flow for training and a standard flow for sampling, marking a novel approach in the field.

\paragraph{Normalizing flows:}
NF traditionally relies on invertible specialized architectures with manageable Jacobian determinants \citep{keller2021self}. One body of work builds invertible architectures by concatenating simple layers (coupling blocks), which are easy to invert and have a triangular Jacobian \cite{nagar2021cinc}. Many choices for coupling blocks have been proposed, such as MAF \citep{papamakarios2017masked}, RealNVP \citep{dinh2016density}, Glow \citep{kingma2018glow}, Neural Spline Flows \citep{durkan2019neural}. Self Normalizing Flow (SNF) \citep{keller2021self} is a flexible framework for training NF by replacing expensive terms in the gradient with learning approximate inverses at each layer. Several types of invertible convolution emerged to enhance the expressiveness of NF. Glow has stood out for its simplicity and effectiveness in density estimation and high-fidelity synthesis.

\paragraph{Autoregressive:}
\citep{kingma2016improved} proposes an inverse autoregressive flow and scales well to high-dimensional latent space, which is slow because of its autoregressive nature. \cite{papamakarios2017masked} introduced NF for density estimation with masked autoregressive. Sample generation from autoregressive flows is inefficient since the inverse must be computed by sequentially traversing through autoregressive order \citep{ma2019macow}

\paragraph{Invertible Neural Network:}
\citep{dinh2016density} proposed Real-NVP, which uses a restricted set of non-volume-preserving but invertible transformations. \citep{kingma2018glow} proposed Glow, which generalizes channel permutation in Real-NVP with $1 \times 1$ convolution. 
However, these NF-based generative models resulted in worse sample generation compared to state-of-the-art autoregressive models and are incapable of realistic synthesis of large images compared to GANs \citep{brock2018large} and Diffusion Models. CInC Flow \citep{nagar2021cinc} proposed a fast convolution layer for NF. ButterflyFlow \citep{meng2022butterflyflow} leverages butterfly layers for NF models. FInC Flow \citep{visapp23} leverages the advantage of parallel computation for the inverse of convolution and proposed efficient parallelized operations for finding the inverse of convolution layers and achieving $O(n\times k^2)$ run time complexity. We designed a backpropagation algorithm for the inverse of convolution layers. Then, a multi-scale architecture, Inverse-Flow, is designed using an inverse of convolution for the forward pass and convolution for the sampling pass and backward pass.

\paragraph{Fast Algorithms for Invertible Convolutions}
The complexity of inverting convolutional layers in NF is a key factor in their efficiency. Traditional methods of invertible convolutions rely on sequential operations that can be computationally expensive. We introduce our approach, FInC Flow, which significantly improves the inversion time compared to existing techniques.

CInC Flow \citep{nagar2021cinc} utilized padded CNNs to achieve invertibility, enabling parallel computation for faster Jacobian determinant computation. However, while this method reduces inversion time, it still requires a relatively large number of operations compared to the optimal performance achievable with a more parallelized approach. FInC Flow \cite{visapp23} builds on the padded CNN design of CInC Flow but introduces a parallel inversion algorithm that reduces the number of operations needed for inversion to just $(2n-1)k^2$, where $n$ is the input size and $k$ is the kernel size.

This algorithm is designed to take advantage of the inherent structure in convolutions, allowing for a more efficient inversion process. The key distinction of FInC Flow over existing models like MaCow \citep{ma2019macow}, which also improves inversion times by masking kernels, is the optimization of parallel operations across multiple channels. By dividing the convolution operations channel-wise, we reduce the computational burden and improve both the forward pass and the sampling time.

In contrast to methods like MintNet \citep{song2019mintnet} and SNF \citep{keller2021self}, which approximate the inversion process and thus achieve faster but less precise results, FInC Flow offers exact inversion while maintaining competitive performance. This approach strikes a balance between inversion speed and model expressiveness, achieving results comparable to previous methods on benchmark datasets like CIFAR-10, ImageNet, and CelebA while significantly reducing sampling time. This advancement not only improves the efficiency of NF but also retains its ability to model complex distributions.

\paragraph{Sampling Time:} NF requires large and deep architectures to approximate complex target distributions \citep{cornish2020relaxing} with arbitrary precision. \cite{jung2024normalizing} presents the importance of fast sampling for Normalizing flow models. To model distribution using NF models requires the inverse of a series of functions, the backward pass, which is slow. This creates a limitation of slow sample generation. To address this, we propose Inverse-Flow, which uses convolution (fast parallel operation, $O(k^2)$, $k\times k$ kernel size) for a backward pass and the inverse of convolution for a forward pass.

\section{Fast Parallel Backpropagation for Inverse of a Convolution}
\label{sec:bp_inv_conv}
We assume that the input/output of a convolution is an $m \times m$ image, with the channel dimension assumed to be 1 for simplicity. The algorithm can naturally be extended to any number of channels. We also assume that input to the convolution is padded on the top and left sides with $k-1$ zeros, where $k$ is the kernel size; see Figure \ref{fig:inv_conv}. Furthermore, we assume that the bottom right entry of the convolution kernel is 1, which ensures that it is invertible.

\begin{figure}
    \centering
    \includegraphics[width=0.7\linewidth]{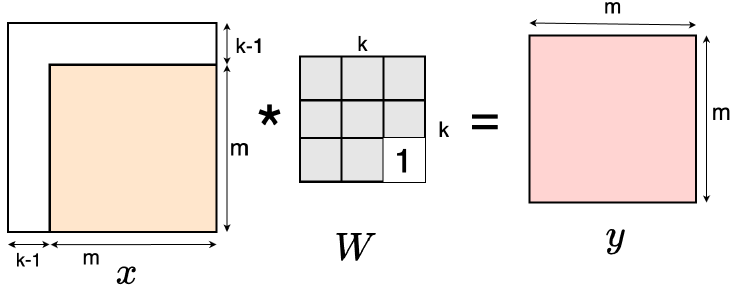}
    \caption{Invertible convolution with zero padding (top, left) on input $x$ and masking of kernel $W_{k,k} = 1$}
    \label{fig:inv_conv}
\end{figure}

The convolution operation is a Linear Operator (in Linear Algebraic terms; see Figure \ref{fig:inv_conv}) on the space of $m \times m$ matrices. Considering this space as column vectors of dimension $m^2$, this operation corresponds to multiplication by an $m^2 \times m^2$ dimensional matrix. Hence, the inverse of convolution is also a linear operator represented by an $m^2 \times m^2$ dimensional matrix. Suppose vectorization of $m \times m$ matrix to $m^2$ is done by row-major ordering; diagonal entries of Linear Operator matrix will be the bottom right entry of kernel, which we have assumed to be 1.

While the convolution operation has fast parallel forward and backpropagation algorithms with running time $O(k^2)$ (assuming there are $O(m^2)$ parallel processors), a naive approach for the inverse of the convolution using Gaussian Elimination requires $O(m^6)$. \citep{visapp23} gave a fast parallel algorithm for the inverse of convolution with running time $O(mk^2)$. In this section, we give a fast parallel algorithm for backpropagation of the inverse of convolution (inv-conv) with running time $O(mk^2)$, see Table \ref{tab:run_time}. Our backpropagation algorithm allows for efficient optimization of the inverse of convolution layers using gradient descent.

\begin{table}[!ht]
    \centering
    \caption{Running times of algorithms for Forward and Backward passes assuming there are enough parallel processors as needed. The forward pass algorithm for the inverse of convolution was improved by \citep{visapp23}. In this work, we give an efficient backward pass algorithm for the inverse of convolution.}
    \begin{tabular}{lrr}
    \toprule
       Layer & Forward & Backpropagation \\ 
    \midrule
    Std. Conv. & $O(k^2)$ & $O(k^2)$ \\
    inv-conv (naive)& $O((m^2)^{3})$ & $O((m^2)^{3})$ \\
    inv-conv. & $ O(mk^2)$  & $\bf O(mk^2)$ \\
    \bottomrule
    \end{tabular}
    \label{tab:run_time}
\end{table}

\paragraph{Notation:} We will denote input to the inverse of convolution (inv-conv) by $y \in \mathbb R^{m^2}$ and output to be $x \in \mathbb R^{m^2}$. We will be indexing $x,y$ using $p = (p_1, p_2) \in \{ 1, \cdots, n\} \times \{ 1, \cdots, n\}$. We define 
$$\Delta(p) = \{ (i,j):  0 \leq p_1 - i, p_2 -j < k  \} \setminus \{ p \}.$$
$\Delta(p)$ informally is  set of all pixels except $p$ which depend on $p$, when convolution is applied with top, left padding. We also define a partial ordering $\leq$ on pixels as follows
$$ p \leq q \quad \Leftrightarrow \quad p_1 \leq q_1 \text{ and } p_2 \leq q_2.$$

The kernel of $k\times k$ convolution is given by matrix $W \in \mathbb R^{k \times k}$. For the backpropagation algorithm for inv-conv, the input is 
$$x \in \mathbb R^{m^2} \text{ and } \frac{\partial L}{\partial x} \in \mathbb R^{m^2},$$ 
where $L$ is the loss function. We can compute $y$ on $O(m^2k^2)$ time using the parallel forward pass algorithm of \citep{visapp23}. The output of backpropagation algorithm is $$\frac{\partial L}{\partial y} \in \mathbb R^{m^2} \text{ and } \frac{\partial L}{\partial W} \in \mathbb R^{k^2}$$
which we call input and weight gradient, respectively. We provide the algorithm for computing these in the next 2 subsections.

\emph{Computing Input Gradients:}
Since $y$ is input to inv-conv and $x$ is output, $y = \text{conv}_W(x)$, and we get the following $m^2$ equations by definition of the convolution operation. 
\begin{equation} \label{eqn:conv}
  y_p = x_p + \sum_{q \in \Delta(p)} W_{(k,k) - p + q} \cdot  x_q 
\end{equation}
Using the chain rule of differentiation, we get that
\begin{equation}
    \frac{\partial L}{ \partial y_p} = \sum_{q} \frac{\partial L}{ \partial x_q} \times \frac{\partial x_q}{ \partial y_p}.
\end{equation}
Hence if we find $\frac{\partial x_q}{ \partial y_p}$ for every pixels $p,q$, we can compute $ \frac{\partial L}{ \partial y_p}$ for every pixel $p$.

\begin{theorem}\label{the:dx_dy}
$$
\frac{\partial x_q}{ \partial y_p} = 
\begin{cases}
\quad    1 - \sum_{q \in \Delta(p)} W_{(k,k) - p + q} \cdot  \frac{\partial x_q}{\partial y_p}  & \text{ if } p = q\\
   \quad 0 & \text{ if } q \not \leq p\\
    - \sum_{r \in \Delta(p)} W_{(k,k) - r} \frac{\partial x_{p-r'}}{\partial y_p} & \text{ otherwise. }
\end{cases}
$$
\end{theorem}

\emph{Formal proof:}

The theorem presents computing $\frac{\partial x_q}{\partial y_p}$, which represents how a change in input pixel $y_p$ affects output pixel $x_q$ in the inverse of the convolution operation. Let's break down each case:

\emph{Case 1:} When $p = q$, take partial derivative with respect to $y_p$ on both sides of Equation \ref{eqn:conv} and rearranging. 
$$\frac{\partial x_p}{\partial y_p} = 1 - \sum_{q \in \Delta(p)} W_{(k,k) - p + q} \cdot  \frac{\partial x_q}{\partial y_p}  $$

So if $\frac{\partial x_q}{\partial y_p}$ is known for all $q \leq p$, we can compute $\frac{\partial x_p}{\partial y_p}.$ Since the off-diagonal entries are unrelated in the $\leq$ partial order, we can compute all of them in parallel, provided the previous off-diagonal entries are known.

\emph{Case 2:} From Equation \ref{eqn:conv}, when $q \not\leq p$, we have: $\frac{\partial x_q}{\partial y_p} = 0$. This case uses the partial ordering defined earlier. If $q$ is not less than or equal to $p$ in this ordering, it means that output pixel $x_q$ is not influenced by input pixel $y_p$ in the inverse of the convolution operation \ref{eqn:conv}. Therefore, the derivative is 0.

\emph{Case 3:} For all other cases:

$$\frac{\partial x_q}{\partial y_p} = -\sum_{r \in \Delta(p)} W_{(k,k) - r} \frac{\partial x_{p-r'}}{\partial y_p}$$

\begin{itemize}
    \item $\Delta(p)$ is set of all pixels (except $p$) that depend on $p$ in a regular convolution operation.
    \item $W_{(k,k) - r}$ represents weight in convolution kernel corresponding to relative position of $r$.
    \item $\frac{\partial x_{p-r'}}{\partial y_p}$ is a recursive term, representing how changes in $y_p$ affect $x$ at a different position.
\end{itemize}

The negative sign and summation in this formula account for the operation's inverse nature and the convolution kernel's cumulative effects.

\paragraph{Computing Weight Gradients:}
From Equation \ref{eqn:conv}, we can say that computing the gradient of loss $L$ with respect to weights $W$ involves two key factors. Direct influence: how a specific weight $W_{a}$ in the convolution kernel directly affects output $x$ pixels, and Recursive Influence: how neighboring pixels, weighted by the kernel, indirectly influence output $x$ during the convolution operation. Similarly, to compute  gradient of  loss $L$ w.r.t filter weights \( W \), we apply  chain rule:

\begin{equation} \label{eq:dl_dw}
    \frac{\partial L}{\partial W} = \frac{\partial L}{ \partial x} * \frac{\partial x}{\partial W}
\end{equation}

where \( \frac{\partial L}{\partial x} \) is  gradient of  loss with respect to  output \( x \) and convolution operation is applied between \( \frac{\partial L}{\partial x} \) and  output \( x \). Computing the gradient of loss $L$ with respect to convolution filter weights $W$ is important in backpropagation when updating the convolution kernel during training. Similarly, $\partial L/ \partial W$ can be calculated as \ref{eq:dl_dw} and $\partial x/\partial W$ can be calculated as (\ref{eq:dx_dw}) for each $k_{i, j}$ parameter by differentiating \ref{eqn:conv} w.r.t $W$:

\begin{equation}
    \label{eq-dw}
    \frac{\partial L}{ \partial W_a} =  \sum {\frac{\partial L}{\partial x_q} * \frac{\partial x_q}{\partial W_a}}
\end{equation}

Equation \ref{eq-dw} states that to compute the gradient of the loss with respect to each weight $W_a$, we need to:
\begin{itemize}
    \item Compute how loss $L$ changes with respect to each output pixel $x_q$ (denoted by $\frac{\partial L}{ \partial x_q}$).
    \item Multiply this by gradient of each output pixel $x_q$ with respect to weight $W_a$ (denoted by $\frac{\partial x_q}{ \partial W_a}$)
\end{itemize}

We then sum over all output pixels $x_q$.

\begin{theorem} \label{eq:dx_dw}
$$
\frac{\partial x_q}{\partial W_a} = 
\begin{cases}
    0 & \text{if } a = q \\
    -\sum_{q' \in \Delta_q(a)} W_{q' - a} \cdot \frac{\partial x_{q-q'}}{\partial W_a} - x_{q-a} & \text{if } q > a
\end{cases}
$$
\end{theorem}

\emph{Formal proof:} Computation of $\partial x_q/\partial W_a$ depends on relative positions of pixel $q$ and kernel weight index. 

\emph{Case 1:} When $a = q $, if index of weight matches index of output pixel \ref{eqn:conv}, gradient is 0. This means that weight does not directly influence the corresponding pixel in this case. \\
\emph{Case 2:} When $ q \geq a$,  gradient is computed recursively by summing over neighboring pixel positions $q'$ in convolution window. The convolution kernel weights $W_{q'-a}$ and the shifted pixel value $x_{q-a}$ are used to calculate gradients.

\paragraph{Backpropagation Algorithm for Inverse of Convolution:}

The backpropagation algorithm for the inverse of convolution (inv-conv) computes gradients necessary for training models that use the inv-conv operation for a forward pass. Our proposed Algorithm \ref{algo:bp_cuda} efficiently calculates gradients with respect to both input ( $\frac{\partial L}{\partial Y}$) and convolution kernel ($\frac{\partial L}{\partial K}$) using a parallelized GPU approach.

Given the gradient of loss $L$ with respect to output ($\frac{\partial L}{\partial X}$ ), the algorithm updates input gradient $\frac{\partial L}{\partial Y}$ by accumulating contributions from each pixel in output, weighted by corresponding kernel values. Simultaneously,  kernel gradient $\frac{\partial L}{\partial K}$ is computed by accumulating contributions from spatial interactions between input and output. The process is parallelized across multiple threads, with each thread handling updates for different spatial and channel indices, ensuring efficient execution. This approach ensures that both input and kernel gradients are computed in a time-efficient manner, making it scalable for high-dimensional inputs and large kernels. A fast algorithm is key to enabling gradient-based optimization in models involving the inverse of convolution.

\emph{Complexity of Algorithm \ref{algo:bp_cuda}:} This computes $\frac{\partial L}{ \partial y}$ and $\frac{\partial L}{ \partial y}$ in $O(mk^2)$ utilizing independence of each diagonal of output $x$ and sequencing of $m$ diagonals. Diagonals are processed sequentially, but elements within each diagonal are processed in parallel. Each diagonal computation takes $O(k^2)$ time due to $k \times k$ kernel. This results in a time complexity of total $O(mk^2)$ and represents a substantial improvement over the naive $O(m^6)$ approach. It makes the algorithm highly efficient and practical for deep learning models with inverse convolution layers, even for large input or kernel sizes.

\SetKwComment{Comment}{/* }{ */}

\begin{algorithm}[!t] \caption{Backpropagation Algorithm for Inverse of Convolution (Input and Weight Gradients)} \label{algo:bp_cuda} 
\KwIn{$K$: Kernel of shape $(C, C, k_H, k_W)$\\ $Y$: output of conv of shape $(C, H, W)$\\ 
$\frac{\partial L}{\partial X}$: gradient of shape $(C, H, W)$} 
\KwOut{$\frac{\partial L}{\partial Y}$: gradient of shape $(C, H, W)$\\ $\frac{\partial L}{\partial K}$: gradient of shape $(C, C, k_H, k_W)$}
\textbf{Initialization:} \\
$\frac{\partial L}{\partial Y} \gets 0$ (initialize input gradient to zero)\\
$\frac{\partial L}{\partial K} \gets 0$ (initialize kernel gradient to zero)

\For{$d \gets 0, H + W - 1$}{
    \For{$c \gets 0, C - 1$}{
        \Comment{The below lines of code are executed in parallel on different threads on the GPU for every index $(c, h, w)$ on $d$th diagonal.}
        \For{$k_h \gets 0, k_H - 1$}{
        \For{$k_w \gets 0, k_W - 1$}{
        \For{$k_c \gets 0, C - 1$}{
            \If{pixel $(k_c, h - k_h, w - k_w)$ not out of bounds}{
                \Comment{Compute input gradient for every pixel $(c, h, w)$:}
                $\frac{\partial L}{\partial Y}[c, h, w] \gets \frac{\partial L}{\partial Y}[c, h, w] + \frac{\partial L}{\partial X}[c, h, w] \cdot K[c, k_c, k_H - k_h - 1, k_W - k_w - 1]$
                \Comment{Compute kernel gradient:}
                $\frac{\partial L}{\partial K}[c, k_c, k_h, k_w] \gets \frac{\partial L}{\partial K}[c, k_c, k_h, k_w] + \frac{\partial L}{\partial X}[c, h, w] \cdot X[k_c, h - k_h, w - k_w]$
            }
        }
        }
        }
        \Comment{synchronize all threads}
    }
}

\Return{$\frac{\partial L}{\partial Y}, \frac{\partial L}{\partial K}$}
\end{algorithm}

\section{Normalizing Flows}\label{sec:inv_flow}

Normalizing flows are generative models that enable exact likelihood evaluation. They achieve this by transforming a base distribution into a target distribution using a series of invertible functions.

Let $\mathbf{z} \in \mathcal{Z}$ be a random variable with a simple base distribution $p_Z(\mathbf{z})$ (e.g., a standard Gaussian). A Normalizing flow transforms $\mathbf{z}$ into a random variable $\mathbf{y} \in \mathcal{Y} $ with a more complex distribution $p_Y(\mathbf{y})$ through a series of invertible transformations: $\mathbf{y} = f(\mathbf{z}) = f_1(f_2(\cdots f_K(\mathbf{z})))$. Probability density of transformed variable $\mathbf{y}$ can be computed using change-of-variables formula:
\begin{equation}\label{eq:flow}
p_Y(\mathbf{y}) = p_Z(\mathbf{z}) \left|\det\frac{\partial f^{-1}}{\partial \mathbf{y}}\right| = p_Z(f^{-1}(\mathbf{y})) \left|\det\frac{\partial f}{\partial \mathbf{z}}\right|^{-1},
\end{equation}

where $\left|\det\frac{\partial f}{\partial \mathbf{z}}\right|$ is absolute value of determinant of Jacobian of $f$.

This relationship (\ref{eq:flow}) can be modeled as $y = f_\theta(z)$, called the change of variable formula, where $\theta$ is a set of learnable parameters. This formula enables us to compute the likelihood of $y$ as:
\begin{equation}
    \log p_Y(y) = \log p_Z(f_\theta(y)) + \log \left| \det \left( \frac{\partial f_\theta(y)}{\partial y} \right) \right|,
\end{equation}

where second term, $\log \left| \det \left( \frac{\partial f_\theta(y)}{\partial y} \right) \right|$, is log-determinant of Jacobian matrix of transformation $f_\theta$. This term ensures volume changes induced by transformation are properly accounted for in likelihood. For invertible convolutions, which are a popular choice for constructing flexible Normalizing Flows, the complexity of computing the Jacobian determinant can be addressed by making it a triangular matrix with all diagonal entries as $1$, and the determinant will always be one.

In this work, we leverage the fast inverse of convolutions for a forward pass (inv-conv = $f_{\theta}$) and convolution for a backward pass, and design \emph{Inverse-Flow} model to generate fast samples. To train  Inverse-Flow, we use our proposed fast and efficient backpropagation algorithm for the inverse of convolution.
\begin{figure}[!ht]
    \centering
    \includegraphics[width=0.7\linewidth]{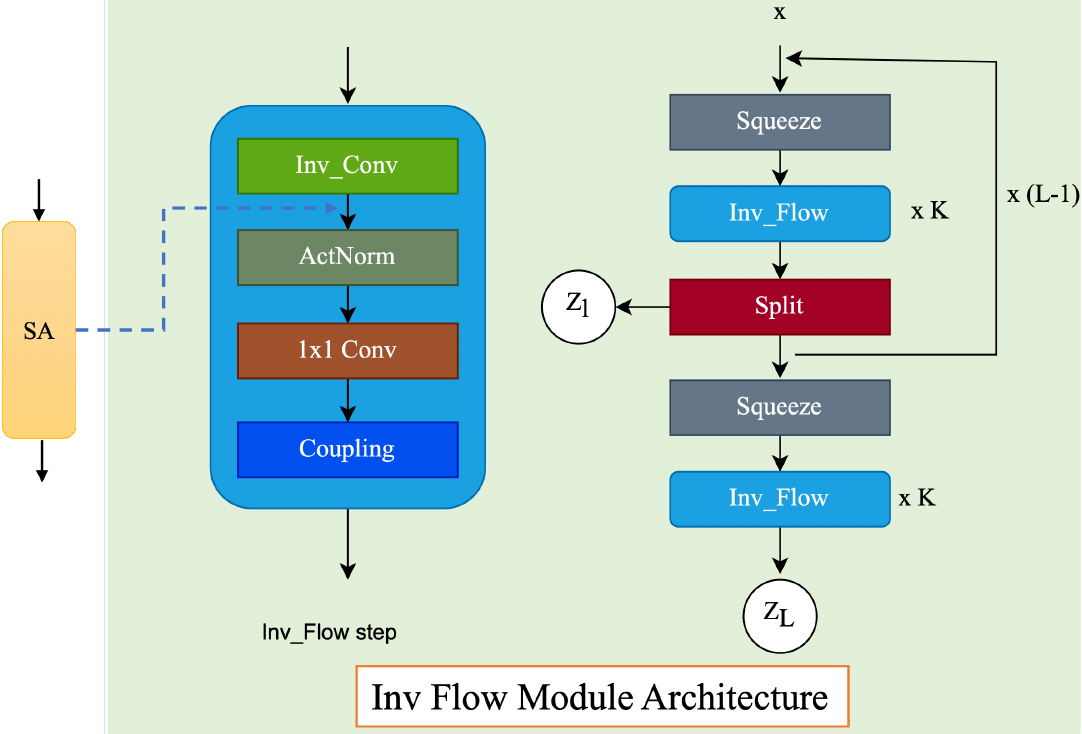}
    \caption{Multi-scale architecture of Inverse-Flow model and Inv Flow step.}
    \label{fig:ms_if_arch}
\end{figure}

\paragraph{Inverse-Flow Architecture:}
Figure \ref{fig:ms_if_arch} shows an architecture of Inverse-Flow. Designing flow architecture is crucial to obtaining a family of bijections whose Jacobian determinant is tractable and whose computation is efficient for forward and backward pass. Our model architecture resembles the architecture of Glow \citep{kingma2018glow}. Multi-scale architecture involves a block of Squeeze, an $Inv\_Flow$ Step repeated $K$ times, and a Split layer. The block is repeated $L - 1$ a number of times. A Squeeze layer follows this, and finally, the $Inv\_Flow$ Step is repeated $K$ times. At the end of each Split layer, half of the channels are ’split’ (taken away) and modeled as Gaussian distribution samples. These split channels are latent vectors. The same is done for output channels. These are denoted as $Z_{L}$ in Fig-\ref{fig:ms_if_arch}. Each $Inv\_Flow$ Step consists of an $Inv-Conv$ layer, an Actnorm Layer, and a $1 \times  1$ Convolutional Layer, followed by a Coupling layer.

\emph{$Inv\_Flow$ Step:} first we consider inverse of convolution and call it Inv\_Conv layer. Figure-\ref{fig:ms_if_arch} left visualizes the inverse of $k \times k$ convolution (Inv-Conv) block followed by Spline Activation layer.

\emph{SplineActivation (SA):} \cite{bohra2020learning} introduced a free-form trainable activation function for deep neural networks. We use this layer to optimize the Inverse-Flow model. Fig-\ref{fig:ms_if_arch}, leftmost: SA layer is added in Inv\_Flow step after Inv\_Conv block.

\emph{Actnorm:} Next, Actnorm, introduced in \citep{kingma2018glow}, acts as an activation normalization layer similar to that of a batch normalization layer. Introduced in Glow, this layer performs an affine transformation of the input using scale and bias parameters per channel.

\emph{ $1  \times 1$ Convolutional:} this layer introduced in Glow does a $1 \times 1$ convolution for a given input. Its log determinant and inverse are very easy to compute. It also improves the effectiveness of coupling layers.

\emph{Coupling Layer:} RealNVP \citep{dinh2016density} introduced a layer in which input is split into two halves. The first half remains unchanged, and the second half is transformed and parameterized by the first half. The output is the concatenation of the first half and the affine transformation by functions parameterized by the first half of the second half. The coupling layer consists of a $3 \times 3 $ convolution followed by a $1 \times 1$ and a modified $3 \times 3$ convolution used in Emerging.

\emph{Squeeze:} this layer takes features from spatial to channel dimension \citep{behrmann2019invertible}, i.e., it reduces feature dimension by a total of four, two across height dimension and two across width dimension, increasing channel dimension by four. As used by RealNVP, we use a squeeze layer to reshape feature maps to have smaller resolutions but more channels.

\emph{Split:} input is split into two halves across channel dimensions. This retains the first half, and a function parameterized by the first half transforms the second half. The transformed second half is modeled as Gaussian samples are latent vectors. We do not use the checkerboard pattern used in RealNVP and many others to keep the architecture simple.

\paragraph{Inverse-Flow Training:}
During training, we aim to learn the parameters of invertible transformations (including invertible convolutions) by maximizing the likelihood of data. Given input data $y$ and a simple base distribution $p_z$ (e.g., a standard Gaussian distribution), the training process aims to find a sequence of invertible transformations such that: $z = \text{inv-conv}(y)$, where $z$ is a latent vector from base distribution and $\theta$ represents the model. The likelihood of data under the model is computed using the change of variables formula:
    $$\log_{p_Y}(y) = \log_{p_z}(\text{inv-conv}(y)) 
    + \log\Big|\det(\frac{\partial \text{inv-conv}(y)}{\partial y})\Big|$$
Here $\det(\frac{\partial \text{inv-conv}(y)}{\partial y})$ represents a Jacobian matrix of transformation, which is easy to compute for \emph{inv-conv}.
\paragraph{Sampling for Inverse-flow}

We use the reverse process to generate samples from the model after training: Sample from the base distribution $z \sim p_{z}(z)$ from a Gaussian distribution. Apply the inverse of the learned transformation to get back data space: 
$y =  \text{conv}_{\theta}(z) $
This process involves performing the inverse of all transformations in the flow, including \emph{inv-conv}. This sampling procedure ensures that generated samples are drawn from the distribution that the model has learned during training, utilizing the invertible nature of convolutional layers.

\section{Results} \label{sec:if_results}
In this section, we compare the performance of Inverse-Flow against other flow architectures. We present Inverse-Flow model results for bit per-dimension ($\log$-likelihood), sampling time (ST), and forward pass time (FT) on two image datasets. To test the modeling of Inverse-Flow, we compare bits-per-dimension (BPD). To compare ST, we generate 100 samples for each flow setting on a single \emph{NVIDIA GeForce RTX 2080 Ti GPU} and take an average of 5 runs after warm-up epochs. For comparing FT, we present forward pass time with a batch size of $100$, averaging over 10 batch runs after warm-up epochs. Due to computation constraints, we train all models for 100 epochs, compare BPD with other state-of-the-art models, and show that Inverse-Flow outperforms based on model size and sampling speed. 

\begin{table}[ht]
    \centering
        \begin{tabular}{lrrrrr}
        \toprule
        \textbf{Method} & \textbf{ST (ms)}  & \textbf{FT (ms)}& \textbf{NLL} & \textbf{BPD} & \textbf{param (M)} \\
        \midrule
        Emerging     & 332.7 $\pm 2.7$  & 121.0 $\pm 1.5$ & 630 & 1.12 & 0.16 \\
        FIncFlow     & 47.3 $\pm 2.3$  & 95.1 $\pm 2.5$ & 411 & 0.73 & 5.16 \\
        SNF          & 33.5 $\pm 2.2$ & 212.5 $\pm 7.3$ & 557   & 1.03 & 1.2\\
        Inverse-Flow & \textcolor{blue}{12.2 $\pm 1.1$} & \textcolor{blue}{77.9 $\pm 1.3$} & \textcolor{blue}{350}  & \textcolor{blue}{0.62} & \textcolor{blue}{0.6}\\
        \bottomrule
        \end{tabular}
    \caption{Performance comparison  for MNIST dataset with $4$ block size and $2$ blocks, small model size. ST = sampling time, FT = Forward pass, NLL is negative-$\log$-likelihood. All times are in milliseconds (ms) and parameters in millions (M).}
        
    \label{tab:M_st_2_4}
\end{table}

\begin{table}[ht]
    \centering
    \begin{tabular}{lrrrrr}
        \toprule
        \textbf{Method} & \textbf{ST} & \textbf{NLL} & \textbf{BPD} & \textbf{Param} & \textbf{Inverse}\\
        \midrule
        SNF          & 99 $\pm 2.1$ & 699 &  1.28  & 10.1  & approx \\
        FIncFlow     & 90 $\pm 2.2$ & 655 & 1.15  &  10.2 & exact \\
        MintNet      & 320$\pm 2.8$ & 630 & 0.98 & 125.9 &  approx \\
        Emerging     & 814$\pm 6.2$ & 640 & 1.09 & 11.4 &  exact \\
        Inverse-Flow & \textcolor{blue}{52 $\pm 1.3$} & 710 & 1.31 & \textcolor{blue}{1.6}  & exact \\
        \bottomrule
        \end{tabular}
    \caption{Performance comparison for MNIST with block size ($K = 16$) and number of blocks ($L = 2$).}
    \label{tab:M_st_2_16}
\end{table}

\paragraph{Modeling and Sample time for MNIST:}
We compare sample time (ST) and number of parameters for small model architecture ($L = 2$, $K = 4$) on small image datasets, MNIST \citep{lecun1998gradient} with image size $1 \times28\times28$ in Table \ref{tab:M_st_2_4}. It may not be feasible to run huge models in production because of the large computations. Therefore, it is interesting to study the behavior of models when they are constrained in size. So, we compare Inverse-Flow with other Normalizing flow models with the same number of flows per level ($K$), for $K = 4, 16$, and $L = 2$. In Table \ref{tab:M_st_2_4}, Inverse-Flow demonstrates the fastest ST of $12.2$, significantly outperforming other methods. This advantage is maintained in Table \ref{tab:M_st_2_16}, where Inverse-Flow achieves the second-best ST of $52\pm 1.3$, only behind SNF but with a much smaller parameter count. Inverse-Flow gives competitive forward time. Table \ref{tab:M_st_2_4} shows that Inverse-Flow has a best forward time of $77.9 $, indicating efficient forward pass computations compared to other methods. 

In Table \ref{tab:M_st_2_4}, Inverse-Flow achieves the lowest NLL ($350$) and BPD ($0.62$), suggesting superior density estimation and data compression capabilities for the MNIST dataset with a small model size. Inverse-Flow consistently maintains a low parameter count for all model sizes. Table \ref{tab:M_st_2_4} uses only $0.6$M parameters, which is significantly less than FInc Flow (5.16M) while achieving better performance. In Table \ref{tab:M_st_2_16}, Inverse-Flow has the smallest model size among all methods, demonstrating its efficiency. Inverse-Flow consistently shows strong performance across multiple metrics (ST, FT, NLL, BPD) while maintaining a compact model size. The following observations highlight Inverse-Flow's efficiency in sampling, density estimation, and parameter usage, making it a competitive method for generative modeling on the MNIST dataset.

For small linear flow architecture, our Inv\_Conv demonstrates the best sampling time of $19.7 \pm 1.2$, which is significantly faster than all other methods presented in Table \ref{tab:if_linear}. This indicates that Inv\_Conv offers superior efficiency in generating samples from the model, which is crucial for many practical applications of generative models. Inv\_Conv achieves the fastest forward time of $100$, outperforming all other methods. It has the smallest parameter count of 0.096 million, making it the most parameter-efficient approach. This combination of speed and compactness suggests that Inv\_Conv offers an excellent balance between computational efficiency and model size, which is valuable for deployment in resource-constrained environments or applications requiring real-time performance.

\begin{table}[ht]
    \centering
        \begin{tabular}{lrrrr}
        \toprule
        \textbf{Method} & NLL & \textbf{ST}  & \textbf{FT} & \textbf{Param} \\
        \midrule
        Exact Conv.       & 637.4 $\pm 0.2$ & 36.5 $\pm 4.1$ & 294 & 0.103 \\
        Exponential Conv. & 638.1 $\pm 1.0$ & 27.5 $\pm 0.4$ & 160 & 0.110 \\
        Emerging Conv.    & 645.7 $\pm 3.6$ & 26.1 $\pm 0.4$ & 143 & 0.103 \\
        SNF Conv.         & 638.6 $\pm 0.9$ & 61.3 $\pm 0.3$ & 255 & 0.364 \\
        Inv\_Conv (our)  & 645.3 $\pm 1.2$ & \textcolor{blue}{19.7 $\pm 1.2$} & \textcolor{blue}{100} & \textcolor{blue}{0.096} \\
        \bottomrule
        \end{tabular}
        \caption{Runtime comparison of small planer models with $9$ layers with different invertible convolutional layers for MNIST.} 
    \label{tab:if_linear}
\end{table}

\begin{table}[ht]
    \centering
    \begin{tabular}{lrrrr}
    \toprule
    \textbf{Method} & \textbf{BPD}  & \textbf{ST} & \textbf{FT} & \textbf{Param} \\
    \midrule
    SNF              & 3.47  & 199.0 $\pm 2.2$ & 81.8 $\pm 3.6$  & 0.446\\
    Woodbury         & 3.55  & 2559.4 $\pm 10.5$ & 31.3 $\pm 1.5$ & 3.125 \\
    FIncFlow         & 3.52  & 47.3 $\pm 2.3$ & 125.5 $\pm 4.2$ & 0.589 \\
    Butterfly Flow   & 3.36  & 155 $\pm 4.6$ & 394.6 $\pm 3.4$ & 3.168 \\
    Inverse-Flow     & 3.56  & \textcolor{blue}{23.2 $\pm 1.3$} & 250.2 $\pm 2.9$ & \textcolor{blue}{0.466} \\
    \bottomrule
    \end{tabular}
    \caption{Performance comparison for CIFAR10 dataset with $L=2$ blocks and block size of $K=4$.}
    
    \label{tab:C_bpd_2_4}
\end{table}

\paragraph{Modeling and Sample time for CIFAR10:} 

In Table \ref{tab:C_bpd_2_4}, Inverse-Flow demonstrates the fastest sampling time of $23.2 \pm 1.3$, significantly outperforming other methods. This advantage is maintained in Table \ref{tab:bpd_2_16_C}, where Inverse-Flow achieves the second-best sampling time of $91.6 \pm 6.5$ among methods with exact inverse computation, only behind SNF, which uses an approximate inverse. While not the fastest in forward time, Inverse-Flow shows balanced performance. In Table \ref{tab:C_bpd_2_4}, its forward time of 250.2 ± 2.9 is in the middle range. In Table \ref{tab:bpd_2_16_C}, its forward time of $722 \pm 7.0$ is competitive with other exact inverse methods.

While not the best, Inverse-Flow maintains competitive BPD scores. In Table \ref{tab:C_bpd_2_4}, it achieves $3.56$ BPD, which is comparable to other methods. In Table \ref{tab:bpd_2_16_C}, its BPD of $3.57$ is close to the performance of other exact inverse methods. Inverse-Flow consistently maintains a low parameter count. In Table \ref{tab:C_bpd_2_4}, it uses only 0.466M parameters, which is among the lowest. In Table \ref{tab:bpd_2_16_C}, Inverse-Flow has the second-smallest model size ($1.76$M param) among methods with exact inverse computation, demonstrating its efficiency. Inverse-Flow demonstrates a good balance between sampling speed and BPD. Comparing Tables \ref{tab:C_bpd_2_4} and \ref{tab:bpd_2_16_C}, we can see that Inverse-Flow scales well when increasing the block size from 4 to 16. It maintains competitive performance across different model sizes and complexities. Table \ref{tab:bpd_2_16_C} highlights that Inverse-Flow provides exact inverse computation, a desirable property shared with several other methods like MaCow, CInC Flow, Butterfly Flow, and FInc Flow.

\begin{table}[!ht]
    \centering
    \begin{tabular}{lrrrr}
    \toprule
    \textbf{Method} & \textbf{BPD} & \textbf{ST} & \textbf{FT} & \textbf{Param} \\
    \midrule
    SNF                 & 3.52 & 16.8 $\pm 2.7$ & 609 $\pm 5.4$ & 1.682  \\
    MintNet             & 3.51 & $25.0^*$ $\pm 1.5$ & 2458 $\pm 6.2$ & 12.466  \\
    \hline
    Woodbury            & 3.48 & 7654.4 $\pm 13.5$ & 119 $\pm 2.5$ & 12.49 \\
    MaCow               & 3.40 & 790.8 $\pm 4.3 $ & 1080 $\pm 6.6 $ & 2.68 \\
    CInC Flow           & 3.46 & 1710.0 $\pm 9.5$ & 615 $\pm 5.0$ & 2.62 \\
    Butterfly Flow      & 3.39 & 311.8 $\pm 4.0 $ & 1325 $\pm 7.5$ & 12.58\\
    FInc Flow           & 3.59 & 194.8 $\pm 2.5$ & 548 $\pm 6.2$ & 2.72 \\
    Inverse-Flow        & 3.57 &  \textcolor{blue}{91.6 $\pm 6.5$} & 722 $\pm 7.0$ & \textcolor{blue}{1.76} \\
    \bottomrule
    \end{tabular}
    \caption{Performance comparison for CIFAR dataset with block size ($K = 16$) and number of blocks ($L = 2$). SNF uses approx for inverse, and MintNet uses autoregressive functions. *time in seconds.}
    
\label{tab:bpd_2_16_C}
\end{table}

\section{Summary}
In this work, we give a fast and efficient backpropagation algorithm for the inverse of convolution. Also, we proposed a flow-based model, Inverse-Flow, that leverages convolutions for efficient sampling and the inverse of convolutions for learning. Our key contributions include a fast backpropagation algorithm for the inverse of convolution, enabling efficient learning and sampling; a multi-scale architecture accelerating sampling in Normalizing flow models; a GPU implementation for high-performance computation; and extensive experiments demonstrating improved training and sampling timing. Inverse-Flow significantly reduces sampling time, making it competitive with other generative approaches. Our efficient inverse convolution backpropagation opens new avenues for training more expressive and faster Normalizing flow models. The publicly available implementation, optimized for GPU performance, facilitates further research and computer vision applications. Inverse-Flow represents a substantial advancement in efficient, expressive generative modeling, addressing key computational challenges and expanding the practical applicability of flow-based models. This work contributes to the ongoing development of generative models and their real-world applications, positioning flow-based approaches as powerful tools in the machine-learning landscape.

 \chapter{Image Super-resolution and Normalizing Flow layers}\label{chap:affine_sr}
\section{Introduction}

\emph{Single image super-resolution (SR)} is an active research area focused on reconstructing high-resolution (HR) images from their low-resolution (LR) counterparts. Since the pioneering work of SRCNN \cite{dong2015image}, deep convolutional neural network (CNN) approaches have significantly advanced the field, leading to numerous breakthroughs and innovations in SR technology. \emph{Deep generative models} have demonstrated remarkable performance in image generation, particularly with autoregressive models \cite{salimans2017pixelcnn++} and Generative Adversarial Networks (GANs) \cite{goodfellow2014generative}. Additionally, Variational Autoencoders (VAEs) and Normalizing Flows (NFs) \cite{kingma2018glow, nagar2021cinc} have been capable of producing high-quality samples. However, a significant drawback of deep models like StableSR \cite{wang2023exploiting} is their substantial computational and training time requirements. To address this, transfer learning can be employed, leveraging pre-trained models to accelerate the process. \emph{Adapting generative priors and Normalizing Flow layers} for real-world image super-resolution presents an intriguing yet challenging problem due to the complexities of model size and image generation time. In this work, we propose a novel approach to address these issues. Our method involves fine-tuning pre-trained diffusion models without making explicit assumptions about degradation. We tackle key challenges such as model size and the number of large trainable parameters needed for fine-tuning Stable Diffusion for super-resolution (SR) by introducing simple yet effective modules. Specifically, we use an encoder and decoder with affine-coupling layers, incorporating more parameters for focused SR training. Our Affine-StableSR model (see Figure \ref{affine-stablesr}) serves as a robust baseline, paving the way for future research in integrating diffusion priors and Normalizing Flow layers for image restoration tasks.

\begin{enumerate}
\item We propose an Affine-StableSR framework that uses the encoder and decoder with affine-coupling layers and pre-trained weights for the stable diffusion model. (see section \ref{para:Affine-SR})
\item We fine-tuned the pre-trained weights of Autoencoder from the StableSR\cite{rombach2022high} for the diffusion model (U-net). (see section \ref{sec:exp_train})
\item We construct new datasets named \emph{'four-in-one}, which contain images from four different sets of images to increase the diversity in the training images, and our experiments show the advantage of a combined dataset for training an SR model. (see section \ref{sec:exp_train})
\end{enumerate}


\begin{figure*}[!ht] 
  \centering
  \includegraphics[width=0.99\linewidth]{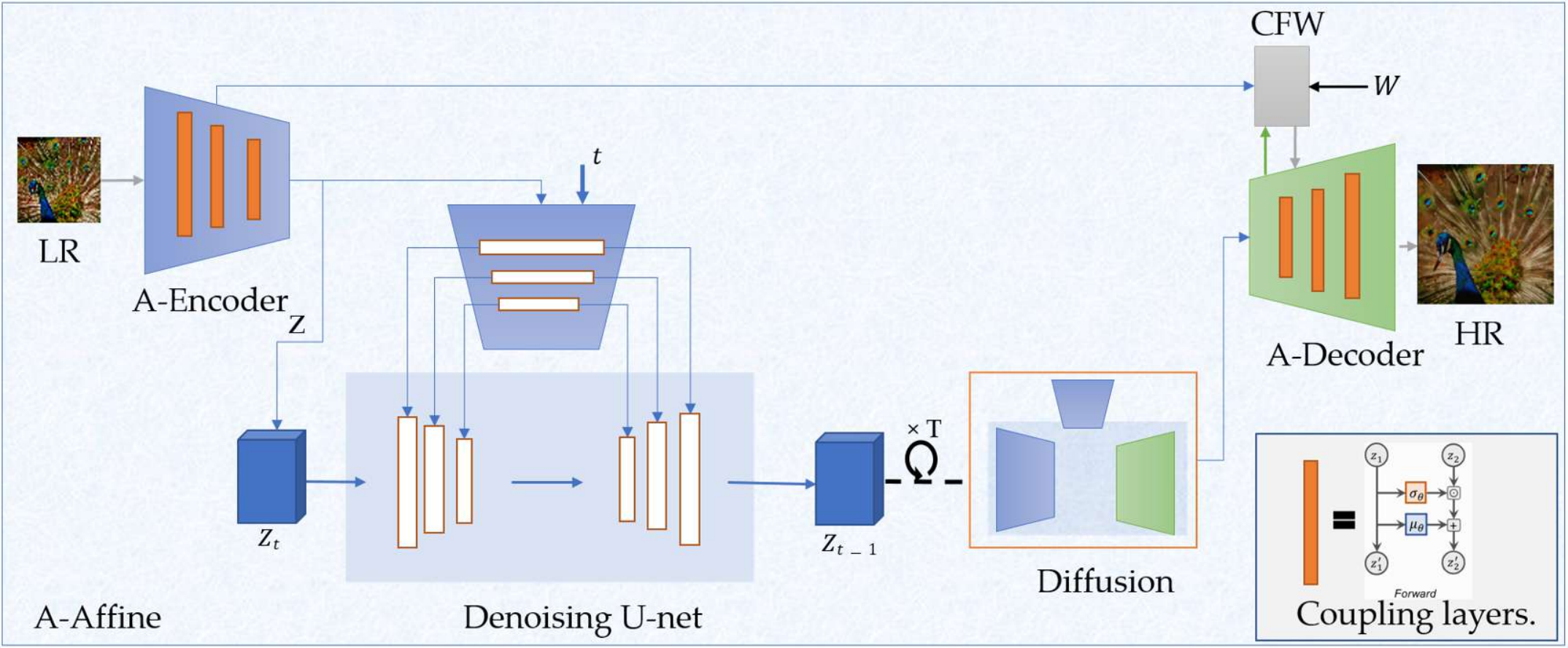}
  \caption{Flow diagram of the proposed SR method, Affine-StableSR. We replaced the ResNet Block with Affine Coupling layers (AffineNet blocks).}
  \label{affine-stablesr}
\end{figure*}



\section{Related work}
\label{sec:related}

\paragraph{Single Image SR:}
One of the first successes in image super-resolution (SR) is SR-CNN \cite{dong2015image}, which utilizes a three-layer CNN structure and employs a bicubic interpolated low-resolution image as input to the network. Building on this, Very Deep SR (VDSR) \cite{kim2016accurate} introduced additional CNN layers, significantly improving SR-CNN. The introduction of Generative Adversarial Networks (GANs) further advanced the field, with SRGAN \cite{ledig2017photo} being a notable example of a GAN applied to image super-resolution. For deep architectures, numerous works have adapted ResNet-based models for SR, such as SR-ResNet \cite{ledig2017photo} and SR-DenseNet \cite{tong2017image}. One frequently used CNN-based super-resolution model in comparative studies is EDSR \cite{lim2017enhanced}, which utilizes ResNets without batch normalization in the residual blocks. EDSR achieved impressive results, securing first place in the NTIRE2017 Super-Resolution Challenge \cite{Timofte_2017_CVPR_Workshops}.

\paragraph{Diffusion Models and SR:}
Denoising Diffusion Probabilistic Models (DDPM) \cite{ho2020denoising} is a new generative framework designed to model complex data distributions and generate high-quality images. Building on DDPM, the Denoising Diffusion Implicit Model (DDIM) \cite{song2020denoising} introduces non-Markovian diffusion processes to accelerate the sampling process. For conditional generation, \cite{dhariwal2021diffusion} conditions the pre-trained diffusion model during sampling with the gradients of a classifier, while an alternative approach explicitly introduces conditional information as input to the diffusion models \cite{saharia2022photorealistic}. Leveraging conditional generation, diffusion models have been successfully applied to various image restoration tasks \cite{fei2023generative}, including denoising, deblurring, inpainting, and reconstruction \cite{zhu2023denoising, lugmayr2022repaint, kawar2023imagic, nagar2023adaptation}. However, to our knowledge, only \cite{li2023ntire} has applied a lightweight diffusion model to image enhancement in a supervised manner. In contrast, our proposed model does not train the diffusion model from scratch but instead utilizes the inherent natural image priors in a pre-trained diffusion model. 

\paragraph{Vision Transformer:}
Recently, the Transformer architecture \cite{vaswani2017attention} has garnered significant attention in the computer vision community due to its success in natural language processing. For instance, ResShift employs the VQ-GAN (Taming Transformer) \cite{esser2021taming} for image super-resolution, while SwinIR \cite{liang2021swinir} utilizes Transformers for image restoration. Although vision Transformers have demonstrated their superiority in modeling long-range dependencies \cite{touvron2021training}, many studies indicate that convolutional layers can enhance Transformers' visual representation capabilities. However, using Transformers remains computationally expensive, and sample generation can be slow.

\paragraph{Normalizing flows:}
Flow-based methods, known as normalizing flows, were first introduced in \cite{dinh2014nice, germain2015made} to learn an exact mapping from data distribution to latent distribution by leveraging the tractability of exact log-likelihood. Later developments like Glow \cite{kingma2018glow} and CInC Flow \cite{nagar2021cinc} demonstrated high-quality and fast image sampling by FInC Flow \cite{visapp23}. Unlike GANs \cite{goodfellow2014generative}, which learn an implicit density, normalizing flows explicitly compute the probability density. In the context of super-resolution (SR), SRFlow \cite{lugmayr2020srflow} employed a conditional normalizing flow method with affine coupling layers, initially limited to $\times2$ SR tasks. SRFlow has since outperformed state-of-the-art GAN-based approaches in both PSNR and perceptual quality metrics, and it has been successfully applied to $\times4$ and $\times8$ SR tasks, generating diverse SR results.

However, flow-based models are constrained by the need to maintain the same input and output dimensions, which complicates changing the resolution of input images and adds computational complexity. Our proposed Affine-StableSR addresses these issues by applying unsupervised super-resolution methods that can accommodate inputs of arbitrary size.

\section{Method} 
\label{sec:method}
This section introduces single-image super-resolution (SR) methods that generate high-quality output images. We use StableSR \cite{wang2023exploiting} as our baseline, as it has demonstrated strong performance and can be applied to input images of any size for a single-image SR. Normalizing Flow models map a simple distribution to a complex distribution using the \emph{Change of Variables Formula} \cite{dinh2014nice}. For Affine-StableSR, we do not need the inverse part ($z = {f}^{-1}(x)$) of the change of variable, and this gives the freedom of using the Affine-coupling layer directly as the standard convolution layers. Affine Coupling \cite{dinh2014nice} (see Figure \ref{fig:affine_vs_resnet}) is a method for implementing a normalizing flow (where we stack a sequence of invertible bijective transformation functions). Affine coupling is one of these bijective transformation functions. Specifically, it is an example of a reversible transformation where the forward function, the reverse function, and the log-determinant are computationally efficient. For the forward function, we split the input dimension into two parts:

\begin{equation}
	x_1, x_2 = split(x, dim=3, chunks=2) 
\end{equation}

The first part (passive, $x_1$) stays the same. In contrast, the second part (active, $x_2$) undergoes an affine transformation, where the parameters for this transformation are learned using the first part ($x_1$) being put through a convolution network. Together, we have:

\begin{equation}
    \begin{aligned}
        s, t = Conv\_2D(x_1) \\
        z_1 = x_1,
        z_2 = (x_2 * s) + t, \\
        z = concat(z_1, z_2)
    \end{aligned}
\end{equation}

We extend the powerful StableSR towards practical super-resolution (SR), which is faster and smaller. StableSR is currently not suitable for real-life applications. Reducing the total number of parameters while increasing trainable parameters helps improve the fine-tuning of Stable Diffusion for super-resolution. Permuting the channels to ensure learning from all channels. Efficient parameter decomposition method. Concept of flow indication embedding. Parameter-efficient solution for scalable NFs with significant sublinear parameter complexity
\paragraph{AffineNet v/s ResNet block:}

We used an Affine coupling layer instead of expensive ResNet blocks to reduce model complexity. Affine net splits its input into passive ($x_b$) and active ($x_a$) dimensions. Improve generalization and robustness to distribution shift. Typically trained using the Maximum Likelihood (ML) loss. Equivalent to the Kullback-Leibler (KL) divergence. Conservation of probability.

\begin{figure}[!pht]
        \centering
	\includegraphics[width=0.89\linewidth ]{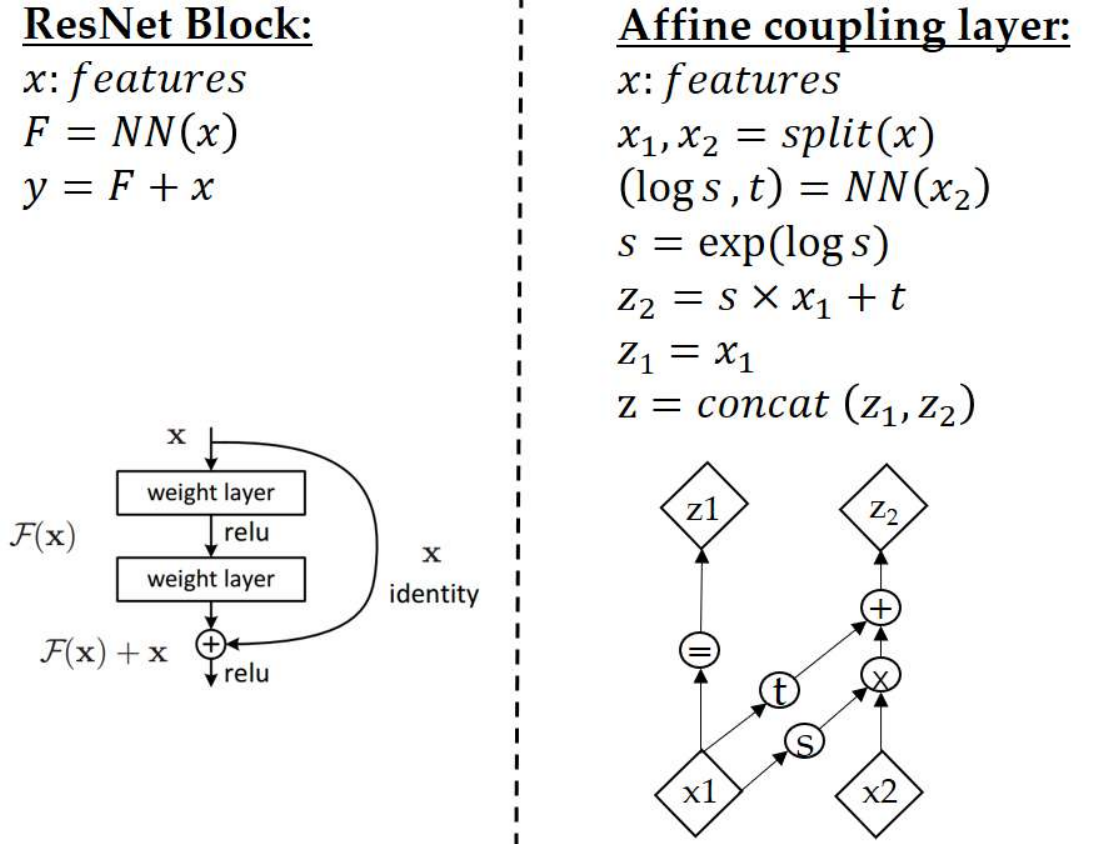}
	\caption{ResNet block vs AffineNet block. Computational graph of an Affine Coupling layer, where $s$ and $t$ are the convolution outputs.}
	\label{fig:affine_vs_resnet}
\end{figure}

\paragraph{Affine-encoder and Affine-decoder:}

In our proposed Affine-Encoder and Affine-Decoder for Latent Diffusion Models (LDM) (see Figure \ref{fig:affine_enc_dec_arch}), we integrate Affine nets (utilizing affine-coupling layers) in place of the standard down-sample and up-sample layers in the encoder and decoder. Normalizing Flow layers typically require the constraint of $input\_size = output\_size$. The StableSR model uses pre-trained weights for the encoder and decoder. However, in our Affine-StableSR, we freeze the encoder and decoder layers except for the Affine net. This strategy focuses on training the Affine net within the encoder and decoder, enhancing their performance for image super-resolution tasks. This focused training helps the latent vector to capture features relevant to super-resolution better and similarly aids the decoder in generating high-quality images. We achieve a reduced model size and fewer trainable parameters by replacing the conventional layers with affine layers. We train the Affine-Encoder and Affine-Decoder for the super-resolution task in a supervised manner, ensuring they are optimized for this specific application.

\begin{figure}[!pht]
    \centering
	\includegraphics[width=0.9751\linewidth ]{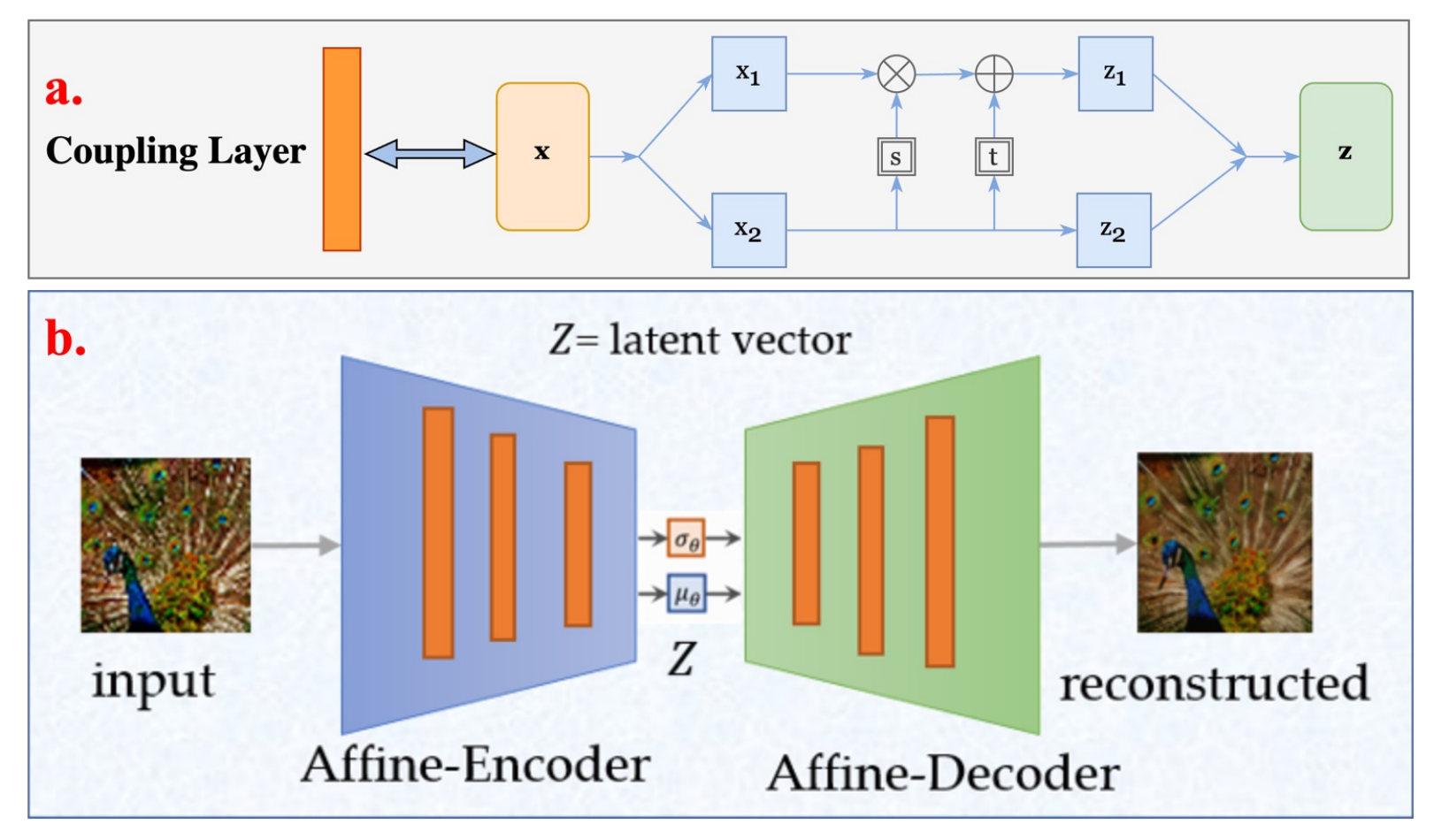}
	\caption{a. Coupling layer flow architecture. b. Affine-Encoder and Affine-Decoder architecture for latent ($4\times64\times64$) with a coupling layer.}
	\label{fig:affine_enc_dec_arch}
\end{figure}

\paragraph{Decoder with coupling layer:}
Replacing convolutional layers with coupling layers in the encoder of a Stable Diffusion model for image super-resolution can enhance the model's ability to capture complex spatial dependencies and maintain high-frequency details. Coupling layers, as used in flow-based models, allow for reversible transformations that preserve information throughout the network. By integrating these layers into the encoder, the model can better learn the intricate mappings between low-resolution and high-resolution images. This approach can lead to improved performance in generating clearer and more detailed super-resolved images, leveraging the ability of coupling layers to manage the flow of information more effectively than traditional convolutional layers.

\paragraph{Affine Coupling layer with permute:}
To make sure all the input channels have a chance to be altered, we use the ordering in each layer so that different components are left unchanged. Following such an alternating pattern, half of the channel that remains identical in one transformation layer is always modified in the next. When hyperparameter \emph{permute is False} in the model arch, then $i\_block\% 2= 0$. This means that the even-numbered layer will use the first half of the channels for learning. If \emph{permute is True}, then $i\_block\%2 = 1$, which means the second half of the channels will be used for learning in the odd-numbered layer.

\begin{table}[!pht]
	\centering
	\begin{tabular}{lccc} \toprule
		Recon. model & \# Total (T) & \# Learnable (L) & \% (L/T)   \\ \midrule
		Affine-AE (our) 		          & \textcolor{blue}{68.3}  & \textcolor{blue}{53.6} & \textcolor{blue}{78.47\%} \\ 
		AE-KL \cite{wang2023exploiting}   & 145   & 47.2 & 32.55\% \\ 
		VQ-GAN \cite{yue2023resshift}     & 101.2 & 101.2   & 0\%    \\ 
		\bottomrule
	\end{tabular}
	\caption{Comparison of the Learnable (fine-tuned for the SR) and total parameters in millions. FLOPs are measured under the setting of upscaling SR images to $512 \times 512$ resolution on the $\times 2$ scale. In the case of AE-KL (encoder and decoder for the StableSR), the decoder includes a Content Feature Wrapping module with learnable layers. In the case of VQ-GAN (Encoder-Decoder used in ResShift), there is no fine-tuning, and the pre-trained weights are used.}
	\label{params}
\end{table}

\begin{table}[!pht]
	\centering
	\begin{tabular}{lccc} \toprule
		reconstruction model & Encoding & Latent & Decoding   \\ \midrule
		Affine-AE (our) & \textcolor{blue}{2.11} & \textcolor{blue}{0.0068} & \textcolor{blue}{0.132 } 
		\\ \midrule
		AE-KL \cite{wang2023exploiting} & 3.32 & 0.0098 & 0.138  \\ 
		VQ-GAN \cite{yue2023resshift}   & 2.81 & 0.0106 & 0.139  \\ 
		VQ-VAE \cite{rombach2022high}   & 4.74 & 0.0143 & 0.148  \\ 
		\bottomrule
	\end{tabular}
	\caption{Computational overhead of different encoders and decoders. Our Affine-Autoencoder (Affine-AE), AE-KL is used in StableSR, VQ-GAN is used in ResShift, and VQ-VAE is used in LDM. Input \& output size is $512\times512$.  Time in seconds on Tesla-P40 24GB GPU}
    
	\label{time_in_sec}
\end{table}

\paragraph{Affine-StableSR:}\label{para:Affine-SR}
For SR image latent vector and reconstruction, we use a StableSR inspired by a latent text-to-image generation, the Latent Diffusion  Model (LDM). The configuration is outlined in Figure 1. \emph{Model Size and Complexity:} is a fundamental problem in diffusion models, and the overhead of the encoder and decoder poses an extra complexity. Our proposed Affine-StableSR targets this problem using fast and efficient layers. This is reflected in the Table \ref{params}. In Table \ref{params}, we present the model size and the number of trainable parameters for the SR. In Table \ref{time_in_sec}, we present the encoding and decoding time of three different AE models for the latent generation used for diffusion models.

\section{Experiments and training:} 
\label{sec:exp_train}
\paragraph{Dataset}  We conducted our experiments using a large dataset comprising images from DIV2K-train, DIV8K, NEOCR, and DTD datasets. This dataset includes 8,599 images at 2K and 8K resolutions and corresponding low-resolution versions generated using bicubic interpolation at scaling factors of ×2, ×3, and ×4. Additionally, we evaluate the performance of our approach on the validation sets of several benchmark datasets: DIV2K \cite{Agustsson_2017_CVPR_Workshops}, Set5 \cite{bevilacqua2012low}, Set14 \cite{zeyde2012single}, B100 \cite{agustsson2017ntire}, and Urban100 \cite{huang2015single}. Evaluation metrics include peak signal-to-noise ratio (PSNR), providing insights into the fidelity of our super-resolution results across different scales and datasets.

\paragraph{Implementation details:}
Affine-StableSR is constructed based on Stable Diffusion 2.1-base2. Our Affine-Encoder is akin to the LDM's encoder U-Net in Stable Diffusion, but is significantly more lightweight (105M parameters, including Spatial Feature Transform). In instances where the input size matches the output size, residual blocks are replaced with Affine-net blocks in both the Encoder and Decoder to facilitate efficient and rapid encoding and decoding processes.

We fine-tuned the diffusion model of Affine-StableSR for 200 epochs using a batch size of 32, with the prompt fixed as null. Following Stable Diffusion practices, we employed the Adam optimizer \cite{kingma2014adam} with a learning rate set to $5 \times 10^{-5}$. The training was conducted at a resolution of $512 \times 512$ across 8 NVIDIA Tesla 24G-P100 GPUs. During inference, we utilized DDPM sampling \cite{ho2020denoising} with 200 time steps. We adopted the aggregation sampling strategy from StableSR for images larger than $512 \times 512$ to handle images of arbitrary sizes. We first enlarged the LR images for images smaller than $512 \times 512$ so that the shorter side matched the network's input resolution.

\emph{Stage-1:}
We focused on training the Affine-Encoder and Affine-Decoder with affine layers during this stage. Results presented in this section were obtained at a resolution of 128 × 128, which matches the training resolution. For images sourced from datasets (\cite{cai2019ntire, wei2020component, ignatov2017dslr}), patches were centered and cropped to achieve 128 × 128 resolution. For other real-world images, we resized them so that their shorter sides were 128 pixels before center cropping. The Affine-Encoder and Affine-Decoder training is supervised for the super-resolution (SR) task.

\emph{Stage-2}
In this stage, Affine-StableSR was trained on a large dataset (8K) and fine-tuned for the SR task using the diffusion model. The results reported in this section were obtained at a resolution of 256 × 256, consistent with the training resolution. Similar to Stage-1, images from specific datasets (\cite{cai2019ntire, wei2020component, ignatov2017dslr}) were cropped to 128 × 128 after centering. Real-world images were resized, so their shorter sides were 128 pixels before center cropping.

\paragraph{Experiment settings:}

\emph{Training Dataset:} Most of the public datasets do not have a diverse set of images. We created a large dataset \emph{SR-Dataset}, a more diverse and big dataset for a single image super-resolution compromise of $8600$ images.  SR dataset includes images from DIV2K \cite{agustsson2017ntire},  DIV8K (Gu et al., 2019), NEOCR, and DTD. SR-dataset comprises 865. We adopt the degradation pipeline of Real-ESRGAN (Wang et al., 2021c) to synthesize LR/HR pairs on the SR-dataset.

\paragraph{training}
With GPU: Tesla P40 24451 MB In Table \ref{tab:aekl}, we present the total validation loss for the training of Affine-AE (proposed encoder-decoder with affine coupling layers) and AE-KL (encoder and decoder used in StableSR). 

\begin{table}[!pht]
	\centering
	\begin{tabular}{lcccc} \toprule
		Model & Rec loss & AE loss & NLL loss &  Total loss \\ \midrule
		Affine-AE (our) & \textcolor{blue}{0.1548} & \textcolor{blue}{0.1941} & \textcolor{blue}{0.1545} & \textcolor{blue}{0.1617} \\ 
		AE-KL \cite{wang2023exploiting} & 0.1617 & 0.2119 & 0.1670 & 0.1685 \\
		\bottomrule
	\end{tabular}
	\caption{Validation loss of AE-KL (encoder and decoder for StableSR) and proposed Affine-AE (Encoder and decoder for Affine-StableSR ) after training for 100 epochs. Where rec-loss is reconstruction loss, ae-loss is autoencoder loss, and nll loss is negative log-likelihood or bits per dimension (BPD).}
	\label{tab:aekl}
\end{table}

In Table \ref{tab:vqgan}, we present the total validation loss for the training of Affine-VQGAN (transformer with affine coupling layers) and VQGAN (encoder and decoder used in ResShift). The results in this table show that using an affine coupling layer helps the transformer learn the encoding and decoding efficiently with half of the model size. As we can see in Figure \ref{fig:affinVQE_vs_VQGAN_loss}, Affine-VQGAN gives almost the same loss value as VQGAN and can be used for the generation and reconstruction of the image.

\begin{table}[!pht]
	\centering
	\begin{tabular}{lcccc} \toprule
		Model & Rec loss & AE loss & NLL loss &  Total loss \\ \midrule
		Affine-VQGAN(our) & 0.4964 & 0.5154 & 0.4840 & 0.5021 \\
		VQ-GAN \cite{yue2023resshift} 	& 0.4827     & 0.4972    & 0.4742     & 0.4896  \\ 
		\bottomrule
	\end{tabular}
	\caption{Validation loss of VQGAN (encoder and decoder for ResShift\cite{yue2023resshift}) and Affine-VQGAN (encoder and decoder for Affine-StableSR) after training for 31 epochs. Where rec-loss is reconstruction loss, ae-loss is autoencoder loss, and NLL loss is negative log-likelihood or bits per dimension (BPD).}
	\label{tab:vqgan}
\end{table}

\begin{figure}[!pht]
    \centering
    \includegraphics[width=0.891\linewidth ]{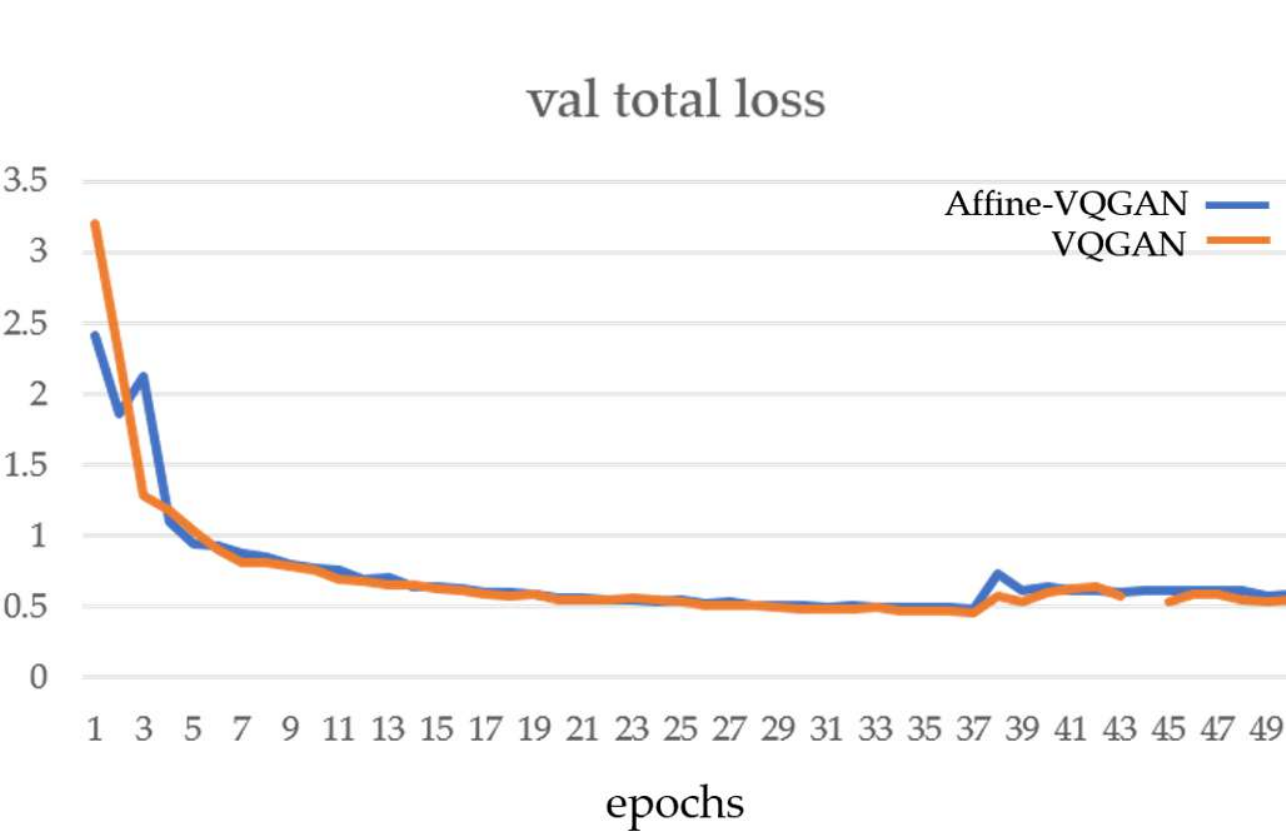}
    \caption{Validation total loss (rec-loss, NLL-loss, and AE-loss) plot affine-VQGAN vs VQGAN (taming transformers, AE used in ResShift) for training 50 epochs.}
    \label{fig:affinVQE_vs_VQGAN_loss}
\end{figure}

\begin{figure*}[!pht]
    \centering
    \includegraphics[width=0.9991\linewidth ]{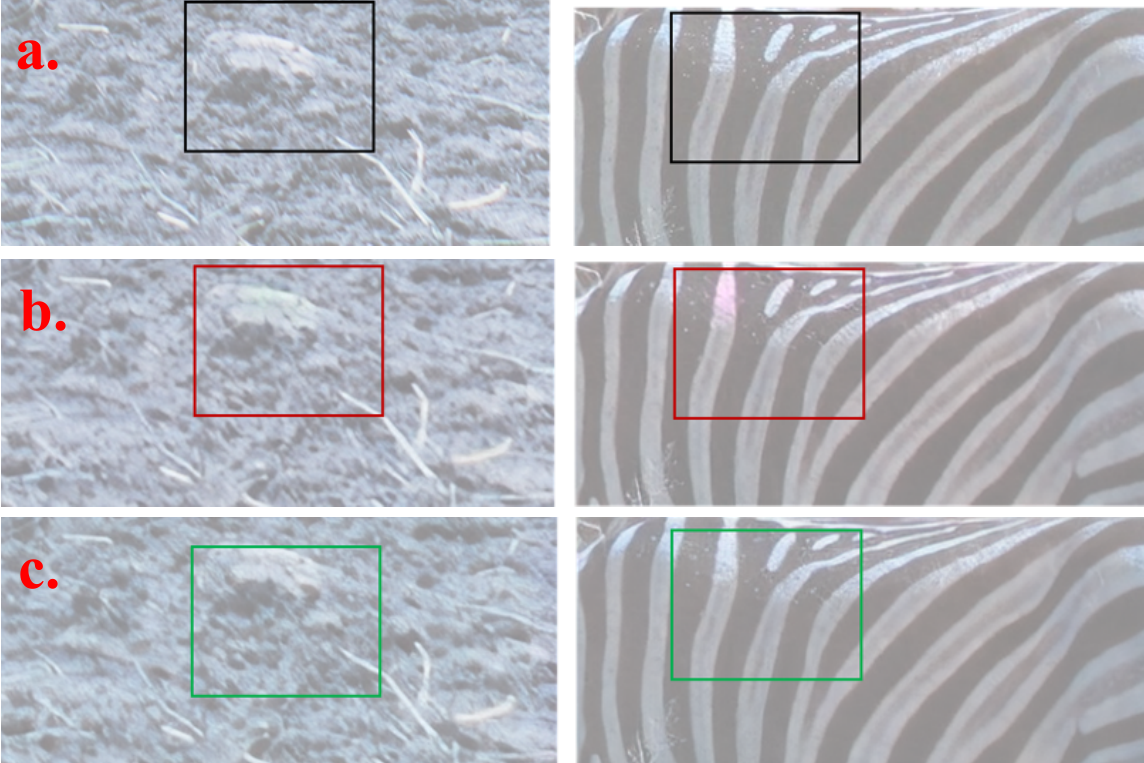}
    \caption{Columns: two different input images, a. input image, b. reconstructed images using Affine-AE (our) and c. AE-KL (used in StableSR). The green box shows a high-feature-rich and detailed reconstruction. Images zoomed and cropped for this comparison.}
    \label{fig:recon_AffineAE_vs_AE}
\end{figure*}

\paragraph{Ablation study and Discussion:}
For the ablation study, we train Affine-StableSR on a large dataset (8K) for classical image SR (×2) and test it on a DIV2K-val set. Qualitative comparison of different backbone networks for $\time 2$ SR. Figure \ref{fig:recon_AffineAE_vs_AE}, \ref{fig:qualti_comp}, see more reconstruction using Affine-VAE.

\begin{figure*}[!pht]
    \includegraphics[width=1\linewidth ]{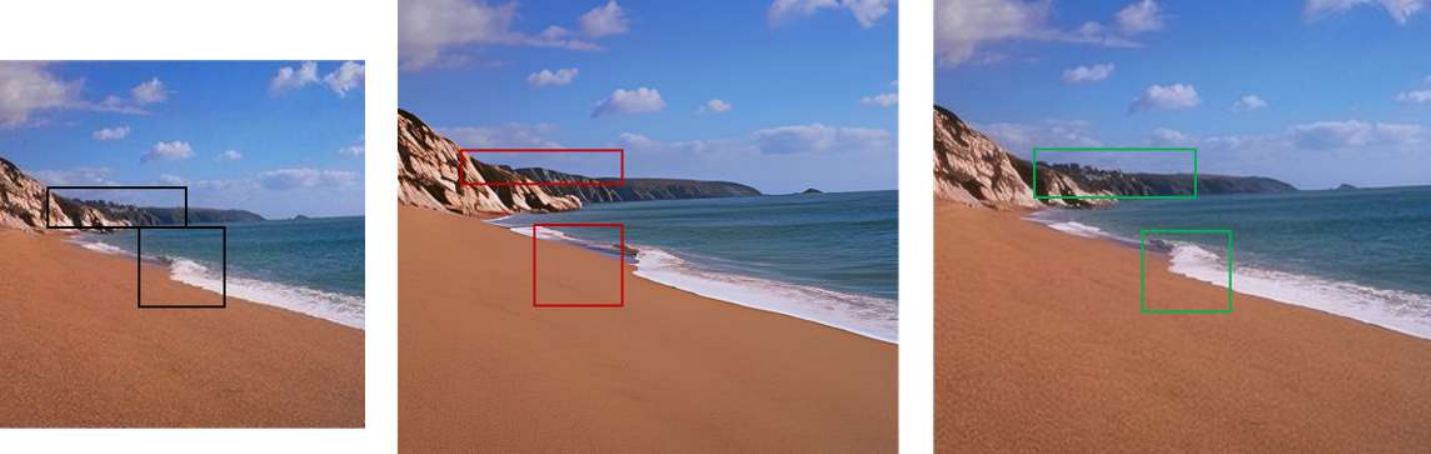}
    \caption{Qualitative comparison on Affine-StableSR and StableSR for $\time 2$ SR. Left: LR image, Centre: StableSR, and Right: Affine-StableSR. The green box shows a high-feature-rich and detailed reconstruction. Zoomed and cropped for this comparison.}
    \label{fig:qualti_comp}
\end{figure*}

\section{Summary} 
\label{sec:conclusion}
This work introduces Affine-StableSR, a novel approach harnessing diffusion priors for real-world Super-Resolution (SR). We undertake targeted training of both encoder and decoder from scratch, addressing prevalent challenges such as high computational demands and model size. Our proposed solutions incorporate affine-coupling layers within the encoder and decoder architectures. We envision our contributions as foundational steps in this domain, with Affine-StableSR offering valuable insights for future research, amalgamating concepts from normalizing flows and diffusion models. Furthermore, we present the Affine-StableSR framework, integrating encoder and decoder structures with affine-coupling layers and pre-trained weights for a stable diffusion model. Leveraging pre-trained weights from StableSR for the diffusion model (U-net), we introduce a new four-in-one dataset comprising images from four distinct sets to enhance training image diversity. We foresee the potential extension of our method to stable diffusion, such as replacing the ResNet block of the U-net and even the transformer for vision tasks.


\part{Applications of Generative Models \& Computer Vision}
 \chapter{Solving Class Imbalance in Corn Image Dataset}\label{chap:seeds}

\section{Introduction}
Quality checking is an essential step to ensuring food grain supplies. It ensures that stocks with different compositions of defective grains are not mixed and can be processed, packaged, and sold for appropriate uses, from high-quality, economy packaging to animal feeds. However, quality checking for small grains and seeds like corn and rice is a tedious task done mainly by experts through visual inspection. Manual inspections are not scalable because of the human resources required, inconsistency between inspectors, and a slower processing pipeline. An automated approach to seed/grain quality testing can solve many problems, leading to better usage and distribution of food stocks. 

\begin{figure}[!ht]
    \centering
    \includegraphics[width=.99\linewidth]{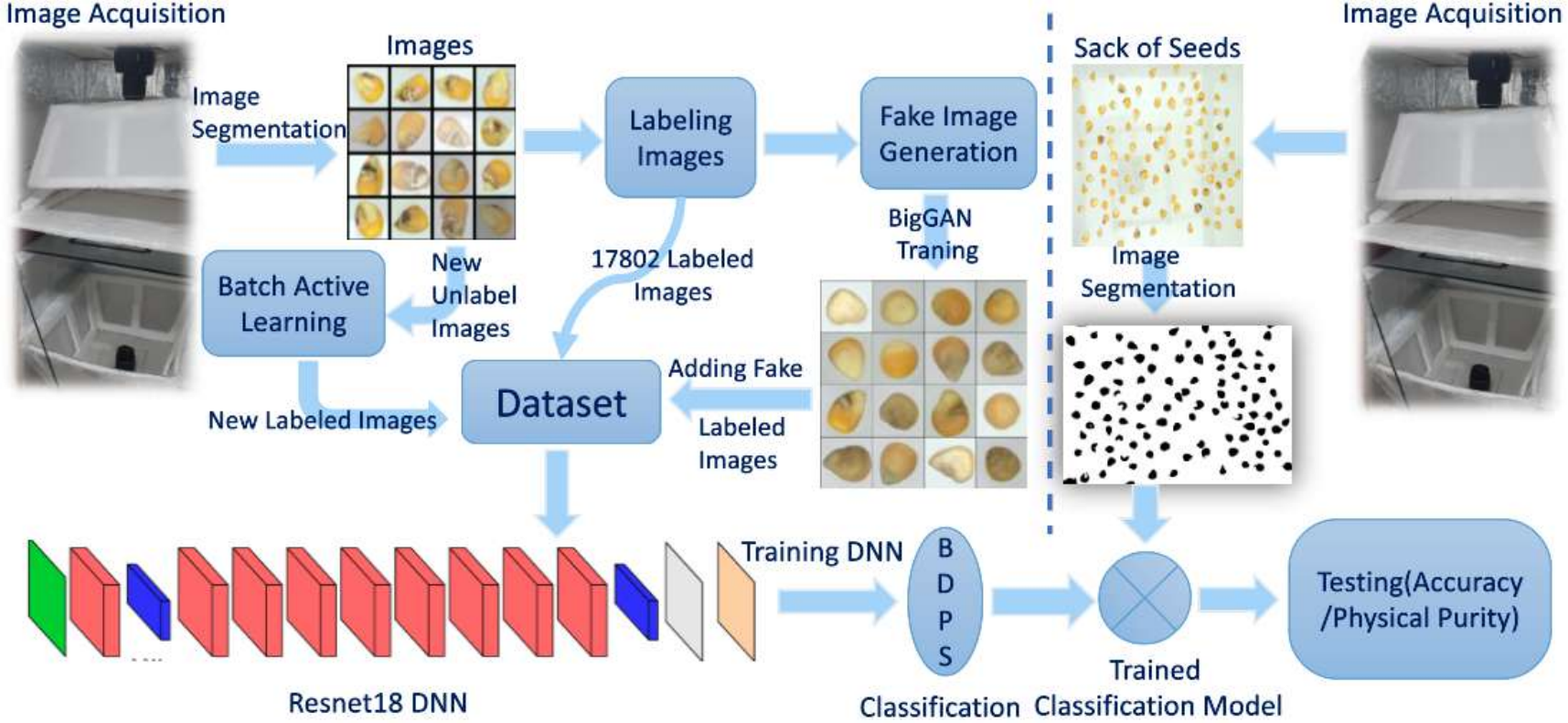}
    \caption{Overview of our proposed system.}
    \label{fig:my_label}
\end{figure}

With the advent of Deep Neural Networks (DNNs), data-driven models have become increasingly adept at image classification/detection tasks. DNN models have been used in literature for seed quality testing problems \cite{seeds_corn,seeds_corn2,cornseed}. However, there are some significant impediments to their widespread adoption. DNNs require large-scale datasets to give high accuracies, which match human inspectors. However, creating large-scale, good-quality datasets is challenging. Annotating seed data sets with defects requires experts with specialized knowledge. Also, a considerable imbalance of seeds can occur in a sample with a specific defect, which results in small datasets highly skewed to non-defective classes, with minimal images in some defective categories. Another main problem is that the defective part of the seed might not be visible from the top view.

We propose addressing class imbalance and annotation costs using machine learning techniques such as generative methods and \textcolor{blue}{Batch Active Learning (BAL).} We use active learning to select the most informative images from unlabeled data to be labeled by experts, leading to less annotation. We use generative methods, specifically conditional generative adversarial networks, to generate images of each class to address the class imbalance problem. We also developed a novel hardware design, which examines the seeds from the top and bottom, leading to better classification accuracy.  

Active learning is used to select an initial batch of images to be annotated. After a round of annotation, the next batch is determined based on the uncertainty of classification remaining, conditioned on the initial batch of annotation. We built an annotation tool, which shows the next batch obtained by Active Learning \cite{nagar2021automated}. Using this tool, we built a dataset of 26K corn seed images, classified into pure seeds and three other defective classes (broken, silkcut, and discolored).

However, the dataset created is highly imbalanced, with images of pure seeds being $40\%$ while some defective classes are as low as $9\%$, with $4$ classes in total. We use a Conditional Generative Adversarial Network (CGAN) to address the class imbalance problem. A CGAN (BigGAN) trained on the labeled base dataset, conditioned on the labels, to generate good-quality seed images. Then,  Conditional GAN is used to generate images that are indistinguishable from real images for each class. Hence, this results in a more balanced dataset. Finally, an image classifier DNN is trained on this dataset (with the fake images added), leading to a better accuracy of 80\%.

\paragraph{Main Contributions:}
We built a computer vision-based system for large-scale automated quality checking of corn seeds. 
\begin{enumerate}
    \item We give a better hardware design, which takes images from the top and bottom to inspect defective/pure seeds (see Section \ref{sec:hard-prim-data}).
    \item We use Conditional Generative Adversarial Networks to generate fake images of each class, leading to a larger and more balanced training dataset (see Section \ref{sec:cgan}).
    \item We build a dataset of 26K corn seed images labeled as pure, broken, silkcut, and discolored (see Section \ref{subsec:dataset}).
    \item We train an image classifier on the dataset with generated and real images, leading to improved accuracy (see Section \ref{sec:results_4}).
\end{enumerate}

\section{Related Works}

\paragraph{ML for Seed Quality Testing and Agriculture:}

Machine vision for precision agriculture has attracted research interest in recent years \cite{ml_in_agri, TIAN20201, rice, cornseed}. Plant health monitoring approaches are addressed, including weed, insect, and disease detection \cite{ml_in_agri}. With the success of DNNs, different approaches have been proposed to tackle problems of corn seed classification \cite{cornseed,seeds_corn,seeds_corn2}. Fine-grained objects (seeds) are visually similar at a rough glimpse, and details can correctly recognize them in discriminative local regions.

\paragraph{Generative Methods for Class Imbalance:}
Generative models can not only be used to generate images \cite{gan_for_class}, but adversarial learning showed good potential in restoring balance in imbalanced datasets \cite{gan_plant}. Generative models can generate samples that are difficult for the object detector to classify. Creating a balanced dataset is a problem because the availability of one type of sample (seed) with defects or impurities compared to others is not always the same. While creating a new dataset, we applied fake image generation to overcome the imbalance of the instances. We use a generative model for the Image-to-image translation with conditional adversarial networks \cite{gan_l}.


\section{Approach}
We approach the problem of seed quality testing by first building a camera setup and a preprocessing pipeline to obtain individual seed images from a sample, which are then labeled (see Section \ref{sec:hard-prim-data}). Then, since the data is highly imbalanced and to minimize the expert labeling effort, we propose two methods: i.) use an Active learning-based UI tool to aid the creation of a larger dataset with the least effort from the expert human intervention, and ii.) Use Conditional Generative Adversarial Networks (CGAN) to generate images to solve the class imbalance problem (see Section \ref{sec:cgan}).
\paragraph{Hardware Setup and Primary Dataset Creation:} \label{sec:hard-prim-data}
Our seed quality testing approach makes use of a camera setup shown in Fig. \ref {fig:setup} below. In this setup, we use two cameras, one on the top and one on the bottom. We place a bunch of seeds in the middle of a transparent glass, take a picture of one from each camera, and save it to the computer connected to the camera. We use a thin white film to block the other side of the transparent glass. While capturing the top view of the seeds, we put the thin film below the glass, which works as a white background, similar to the bottom view. Thus, the top camera gives the picture a top view of each seed, and similarly, the bottom camera captures the bottom picture of the seeds to train the classification model. We use the seeds' top and bottom images as individual input images, which can get two independent predictions, increasing the accuracy.

\begin{figure}
    \includegraphics[width=0.79\linewidth]{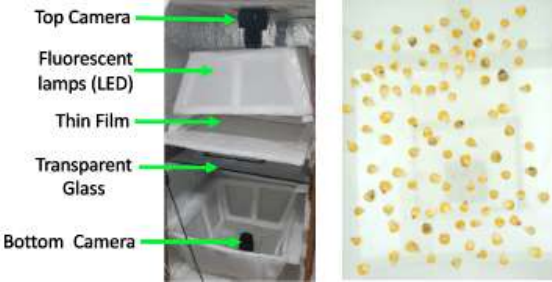} 
    \caption{Image capturing setup and sample image.}
    \label{fig:setup}
\end{figure}
For seed classification, individual corns in the captured image should be accurately segmented first. While placing the seeds on the transparent glass, we mind the gap between the neighboring seeds to avoid the overlapping of seeds segmentation \cite{seg} of the image done by the Watershed method, because this transform decomposes an image completely and assigns each pixel to a region or a watershed, and image segmentation can be accomplished simultaneously. After the segmentation, the expert's labeling of each seed image is done for the 17802, considering the top and bottom two seed images as different. To classify new seeds after the training model, we take a picture of a sack of seeds and use the segmentation to detect the location \cite{seg} of each seed and give it as input to the model classification.  

\paragraph{Conditional Image Generation to Balance Dataset:}\label{sec:cgan}

\begin{figure}
 \centering
 \includegraphics[width=0.99\linewidth]{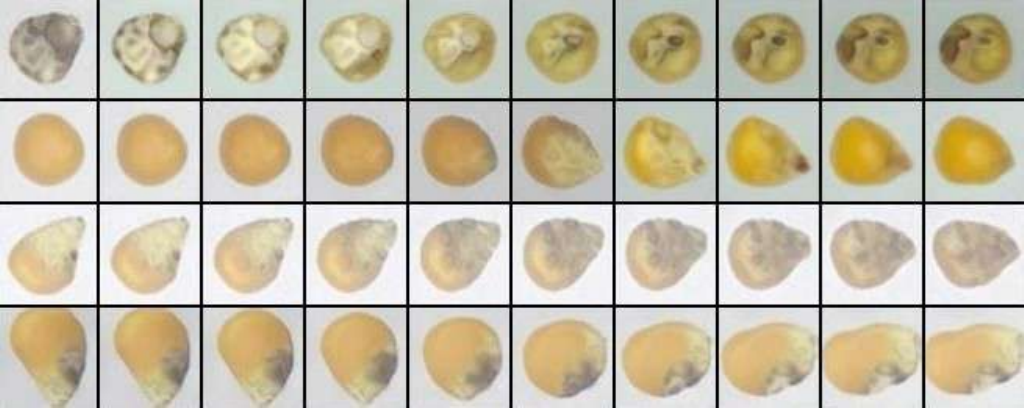} 
 \caption{Example of interpolation of two GAN-generated images. A gradual linear transformation from one seed to another using latent vector(z) interpolation.}
\label{fig:inter}
\end{figure}
Generative Adversarial Networks (GANs) are one of the state-of-the-art approaches for image generation. Existing works primarily aim at synthesizing data of high visual quality \cite{gan_l,gan_high}. We use BigGAN \cite{gan_l}, as it gives good-quality images in similar datasets. Training generative adversarial networks (GANs) using too little data typically leads to discriminator (D) over-fitting \cite{limitdata}. GAN training is dynamic and sensitive to nearly every aspect of its setup (from optimization parameters to model architecture). \textbf{Image Interpolation:} Exploring the structure of the \emph{latent space} $(z)$ for a GAN model is interesting for the problem domain. It helps develop an intuition for what has been learned by the generating model (see Figure \ref{fig:inter}). We use image interpolation to evaluate whether the inverted codes are semantically meaningful. We interpolate one type of seed to other classes in a large diversity. As we can see in Figure \ref{fig:inter} left to right, how smoothly the interpolation works, which validates that the GAN has learned a good latent representation for images in the dataset.

\section{Results}\label{sec:results_4}

\paragraph{Dataset \& Experimental Details:} 
\label{subsec:dataset}
Our primary dataset contains four classes: \emph{broken}, \emph{discolored}, \emph{pure}, and \emph{silkcut} (B, D, P, S), having different instances for each class (Imbalance). Images are taken for both sides of the corn, \emph{top-view}($8901$) and \emph{bottom-view}($8901$), as explained above in the hardware setup section-\ref{sec:hard-prim-data}. The primary dataset is highly imbalanced. We explore the confusion matrix in the result section \ref{sec:results_4} to analyze the class confusion for seeds. Instances of each class and their \% in the primary dataset are as follows: \emph{broken} $5670$ $(32\%)$, \emph{discolored} $3114$ $(17.4\%)$, \emph{pure} $7267 (40.8\%)$, \emph{silkcut} $1751 (09.8\%)$. We use two methods to add more images to the primary dataset: i.) adding GAN-generated images to balance the dataset, and ii.) Adding more captured images labeled using the Batch Active Learning method.

In case of adding fake images to balance the dataset, we split the primary dataset into two parts, training and validation, in a $70:30$ ratio, ensuring that each class has the same \% of instances on each set. The train set is used to train the BigGAN model, and after adding fake images generated by BigGAN into the train set, this new train set is used to train the classification model. Finally, the Validation set is used for testing the classification model only. 

We generated fake images as follows: \emph{broken} $2937$, \emph{discolored} $5823$, \emph{pure} $2937$, \emph{silkcut} $5823$ instances, and added them into the train set to balance the dataset. In the case of adding newly captured images labeled using the \emph{Batch Active Learning} method after image segmentation, new $9000$ labeled images are added to the \emph{primary} dataset. This new dataset contains 26,802 images split into train and validation sets $80:20$, respectively. The train set is used to train the classification model, and the Validation part is used to test the classification model. 

We train different Convolutional Neural Networks (CNN) models on this dataset in Section \ref{subsec:classification}. First, we use transfer learning; the models are initialized with weights learned from ImageNet Classification \cite{imnetclass}. Then, we trained the DNNs on the train set of the primary dataset to fine-tune the model and compare the validation accuracy for a different model (see table \ref{tab:freq}). Since Resnet$18$ has the highest accuracy, we trained it on three different datasets: i.) the primary dataset, ii.) with fake images generated using BigGAN, and iii.) after adding the newly annotated images labeled by the Batch Activation Learning method.

\paragraph{Conditional Image Generation (CGAN):}\label{subsec:CGAN}
We trained the conditional GANs to generate fake images to add to the dataset and reduce the class instances' imbalance. To train BigGAN, the \emph{primary} dataset is split into two sets, train and validation $70:30$, respectively. The train set is used to train the BigGAN and, after adding the fake images, to train the classification model. The validation set is used only to test the classification model. We trained the Generator (G) and Discriminator (D) of the BigGAN for 250 epochs, and after training, we passed random noise and a label to the G to generate the fake images. Hyper-parameters  used are as follows: input image resolution$=256\times256\times3$, learning rate $=2e^{-4}$, batch size $=16$, dimension of the latent vector space $\text{dim}(z)=128$. BigGAN is trained in alternate phases for D and G, and we ensure that each input batch contains an equal no. of images from each class. Figure \ref{fig:sample_img} shows and compares that the image generated is indistinguishable from real images.
\begin{figure}[ht!]
  \centering
  \includegraphics[width=\linewidth]{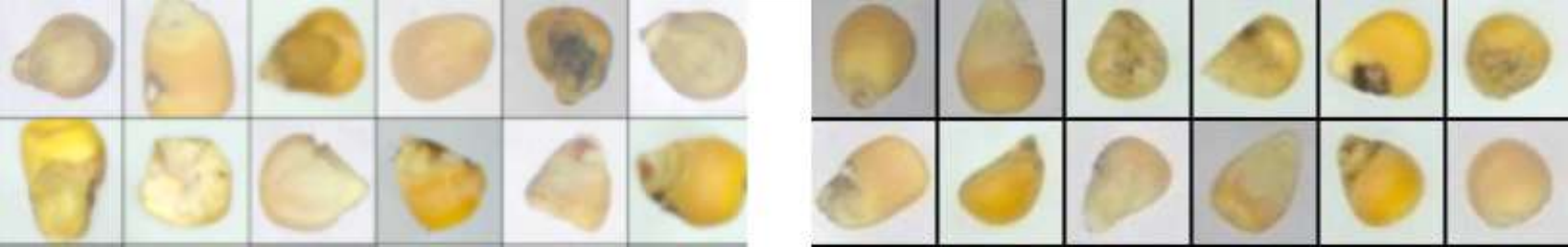}
  \caption{Sample images from primary dataset (left) \& ones generated using BigGAN (right).}
  \label{fig:sample_img}
\end{figure}
\paragraph{Classification:}\label{subsec:classification}
We have validated the effectiveness of Active Learning (\cite{nagar2021automated}) and analyzed the quality of images generated by the Conditional GAN in Section \ref{subsec:CGAN}. Here, we discuss the accuracy of the classification model trained on the dataset obtained.

\begin{table}[!th]
\centering
  \begin{tabular}{lc}
    \toprule
    Model &  Acc. \%   \\
    \hline
    Resnet18 &   \textcolor{blue}{71}\\
    Squeezenet &  71 \\
    ResNet50 &    70 \\
    Mobilenet &   68 \\
    WideResnet50  &  69 \\
  \bottomrule
  \end{tabular}
  \caption{Validation accuracy for different CNN models on \emph{primary} dataset. }
  \label{tab:freq}
\end{table}

\begin{table}[!ht]
\centering
  \begin{tabular}{lcccccc}
    \toprule
    \# Fake images &$\qquad$ & Acc. \% & $\qquad$ & Physical Purity \%  \\
    \midrule
    Zero && 71.00 && 80.58 \\
    20K && \textcolor{blue}{79.24} && \textcolor{blue}{88.25} \\ 
    40K && 78.23 && 87.68 \\
    100K && 78.88 && 87.24 \\
  \bottomrule
  \end{tabular}
  \caption{Validation accuracy comparison before and after adding the fake images to the dataset (for ResNet18 DNN,) solving imbalance. \emph{Accuracy:} four-class classification accuracy. \emph{Physical Purity}: Pure vs Impure (two-class classification).}
  \label{tab:fake}
\end{table}

\begin{figure}[ht!]
  \centering
  \includegraphics[width=0.99\textwidth]{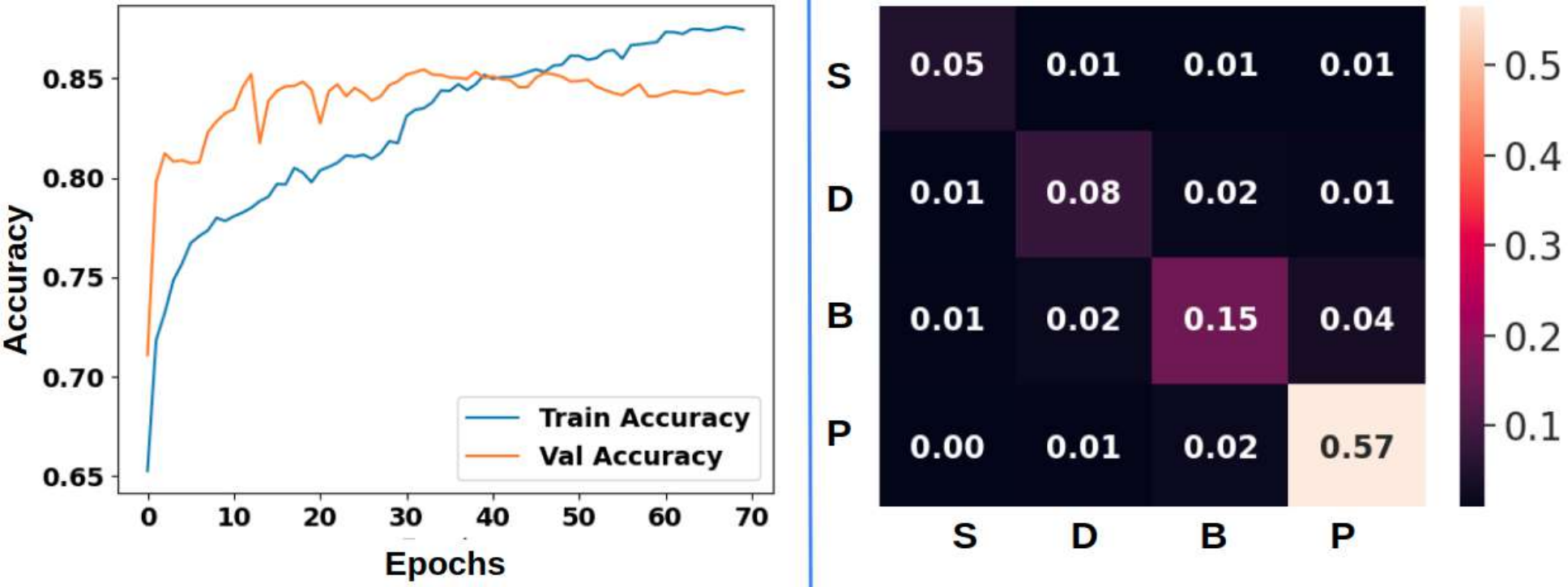}
  \caption{Left: Train and validate the accuracy of the ResNet18 using a pre-trained model. Right: Confusion Matrix: heat map of validation images (4946) after training the ResNet18 on a train set of new datasets (26802). Ground Truth: rows, Predicted: columns. B: broken, D: discolored, S: silkcut, P: pure.}
  \label{fig:heatmap}
\end{figure}

The highest validation accuracy of the primary dataset is $71.02\%$ (see Table \ref{tab:freq}) among different deep learning models. In the case of adding fake images into the train set of the primary dataset to balance the dataset, the validation accuracy after training the classification model on the new train set increases the form $71.02\%$ to $79.24\%$ (see Table \ref{tab:fake}), thereby validating the approach of using CGAN to solve the class imbalance problem. The graph in Figure \ref{fig:heatmap} plots the training and validation accuracy for Resnet 18 on the dataset after adding new images labeled by the BAL. This further improves the accuracy to $85.24\%$, and we also get the \emph{physical purity} (\emph{pure} vs \emph{impure} classification) accuracy to be $91.62\%$. The classwise accuracies are as follows: \emph{broken} $71.20\%$, \emph{discolored} $69.08\%$, \emph{pure} $94.94\%$, \emph{silkcut} $75.82\%$.

Some images in the dataset have high-class ambiguity. To analyze, we used the confusion matrix given in Figure \ref{fig:heatmap} for the validation set. Each matrix entry gives the \% of images of specific ground truth (rows) and a specific predicted class (columns). As can be seen from Figure \ref{fig:heatmap}, the classes \emph{pure} and \emph{broken} are the most confusing, followed by \emph{discolored} and \emph{broken}.

\section{Summary \& Discussion}
In this work, we propose a novel computer vision-based automated system that can be used for corn seed quality testing. A novel image acquisition setup is used to obtain two different viewpoints for every seed. We also address the class imbalance problem by using Conditional GANs (BigGAN) to generate more images of classes with a small dataset. We believe similar approaches can be used for quality testing of various seeds and vegetables, and can decrease wastage and human intervention.

 \chapter{Image Anonymization for Street Scenes Datasets using Image Inpainting}\label{chap:pvt_idd}

\section{Introduction}

As autonomous driving technology advances, large volumes of data are being collected on public roads globally. While this data is essential for model training, it also heightens privacy concerns, particularly regarding visible faces and vehicle license plates. This makes developing effective detection and anonymization models essential in real-world settings. However, existing public datasets for face and license plate detection are either limited to parking lot scenes or lack focus on the autonomous driving context.

Recent advances in computer vision have enabled its wide application in different domains, and one of the most exciting applications is autonomous vehicles (AVs), which have encouraged the development of several ML algorithms, from perception to prediction to planning. However, training AVs usually requires a lot of training data collected from different driving environments. Although existing privacy protection approaches have achieved theoretical and empirical success, there is still a gap when applying them to real-world applications such as autonomous vehicles \cite{xie2022privacy}. By scaling this idea from face to the entire image, prior studies adopted down-sampling and pixelation to
conceal Volatile Personally Identifiable Information (V-PII), which otherwise is apparent in high-resolution images.

A straightforward method for preserving privacy in visual data is image obfuscation, such as face blurring, before cloud storage or processing \cite{ilia2015face}. Rajput et al. \cite{rajput2020privacy} proposed a variation of this approach by introducing Gaussian noise into the image before applying down-sampling, thereby obscuring facial details.

\begin{figure*}[!ht]
     \centering
    \includegraphics[width=0.99\linewidth]{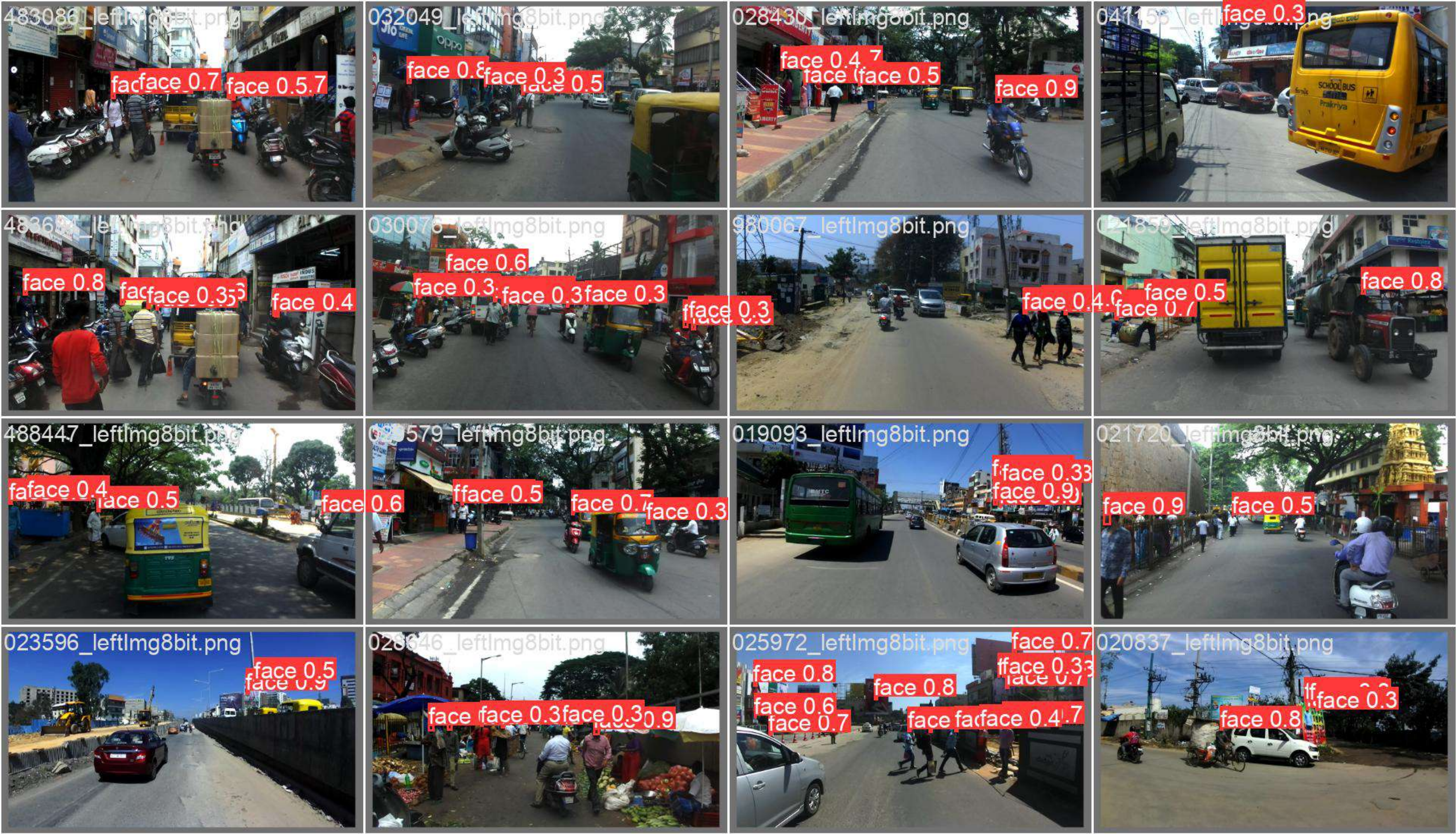}
    \caption{Sample predictions from the RetinaFace model \cite{deng2019retinaface} on the Pvt-IDD dataset. The predicted bounding boxes and their corresponding confidence scores and class labels are shown, demonstrating the model’s ability to localize faces in real-world autonomous driving scenarios.}
    \label{fig:pred_labels2}
\end{figure*}

However, the lack of publicly available, annotated datasets and standardized benchmark models significantly hinders progress in developing robust anonymization techniques for autonomous driving scenarios. Many commercial systems utilize proprietary datasets for training and evaluation, which are inaccessible to the research community. Furthermore, academic attention to privacy-preserving techniques in autonomous driving has been limited. No publicly released face detection datasets are tailored explicitly for in-vehicle or on-road use cases with anonymized faces. While numerous public datasets exist for general face detection, such as PASCAL FACE, FDDB, UFDD, MALF, and WIDER FACE \cite{everingham2015pascal, jain2010fddb, nada2018pushing, yang2015fine, yang2016wider}, none of them address the unique constraints and challenges of privacy in autonomous driving environments. 

We propose a method to preserve privacy in the road image dataset using face detection and image inpainting models. The results of real-dataset experiments validate that our method can achieve privacy protection while maintaining image recognition accuracy, which provides a more effective tradeoff between privacy and utility than the state-of-the-art.

Main contributions: 
\begin{itemize}
    
    \item We propose an automated anonymization pipeline by adapting a face detection model (RetinaFace) and a generative inpainting model (Inpaint-Anything) to replace facial regions while preserving scene context (see Section \ref{sec:face_anno}).
    
    \item We introduce a practical evaluation framework to assess anonymization quality based on its impact on face detection performance, ensuring task-relevant information is not lost (see Section \ref{sec:eval_anno}).
    
    \item We conduct comparative experiments to evaluate detection and segmentation accuracy on datasets with real versus anonymized faces, highlighting the trade-offs and effectiveness of our anonymization method (see Section \ref{sec:results_idd}).
\end{itemize}

\section{Related Work}
Existing public datasets for face detection and autonomous driving generally treat these two tasks separately, with few datasets tailored to driving scenarios. Furthermore, most face detection datasets do not account for dynamic driving contexts. In contrast, face datasets focus on stationary environments, rather than urban or highway scenes.

\paragraph{IDD Dataset:}
IDD \cite{varma2019idd} comprises high-resolution images captured from a front-facing camera mounted on a car driven through urban and semi-urban regions of Hyderabad and Bangalore. It features semantic segmentation and object detection annotations, offering rich and diverse scenes typical of Indian driving conditions. The segmentation dataset includes 10,003 images with unique labels such as billboards, auto-rickshaws, animals, and drivable muddy areas beside roads, capturing the complexities of unstructured traffic environments. Annotations follow a four-level hierarchical labeling scheme and are split into training, validation, and test sets. The detection subset includes over 46,000 images, with well-defined train, validation, and test partitions. This dataset stands out for its real-world variability, including occlusions, varied lighting, and the lack of strict lane discipline. It is a valuable resource for training robust perception models for autonomous driving in developing-world settings.

\paragraph{ECP Dataset:}
The EuroCity Persons (ECP) \cite{braun2018eurocity} dataset is a comprehensive and large-scale benchmark designed to support person detection in urban traffic scenes. It comprises over 47,300 images and more than 238,000 manually annotated person instances across various European cities, offering nearly ten times the data volume of previously established pedestrian datasets. The dataset provides high-quality bounding box annotations for fully visible and partially occluded individuals, with estimated extents and occlusion levels, making it particularly valuable for tasks like tracking. It distinguishes between pedestrians and riders (e.g., cyclists and motorcyclists) and includes part-based annotations for rider-vehicle pairs. Additional vehicle-related objects, such as bicycles, scooters, and wheelchairs, are annotated, along with ignore regions that help manage uncertain or ambiguous instances, improving robustness during training. The dataset’s scale, diversity, and detailed annotations make it a critical resource for advancing person detection in real-world, safety-critical applications like autonomous driving.

We also use the ECP \cite{braun2018eurocity} dataset to test the anonymization pipeline \ref{fig:samples2}, \ref{fig:samples3}. After filtering based on the visible face, we get 7712 training and 1533 validation images. For this ECP folder dataset, we did the detection only, as the ECP dataset has only detection labels. ECP dataset images are low-light day images.

\paragraph{Face Detection in Autonomous Driving Datasets:}
Most public face detection datasets are compiled from the internet or natural images, ranging from small collections (e.g., Annotated Faces in the Wild (AFW) \cite{zhu2012face} dataset with 205 images or PASCAL FACE with 851 images \cite{YAN2014790}) to extensive datasets (e.g., WIDER FACE with over 32,203 images \cite{yang2016wider} and  MALF with 5,250 images \cite{yang2015fine}). These datasets provide bounding box annotations and attributes like yaw, pitch, and roll under varied conditions, such as celebrity images, weather-based degradations, and motion blur. However, none specifically address the unique demands of driving scenarios, where faces in street scenes are typically small due to distance and captured from multiple viewing angles. Such distinctions pose specific challenges for anonymization models in autonomous driving.

\paragraph{Face Anonymization Models:}

Several commercial solutions have recently emerged to address privacy compliance in autonomous driving through data anonymization. Notably, Brighter AI introduced Deep Natural Anonymization (DNAT), a deep learning model based on R-CNN that targets anonymization of faces, license plates, and human bodies across standard and fisheye camera images, claiming $99\%$ accuracy on their private dataset. Celantur employs Mask-RCNN and keypoint detection to anonymize people, vehicles, and license plates. Similarly, Facebook’s Mapillary platform offers anonymization services and has published comparative evaluations highlighting its superiority over public APIs from Amazon, Google, and Microsoft \cite{mapillary2019privacy}. Other companies like dSpace’s Understand AI and NavInfo have also launched anonymization products \cite{understandai2022anonymizer}, with NavInfo uniquely reporting performance metrics on public datasets such as IJB-C \cite{maze2018iarpa} and CCPD \cite{ccpa2018}. However, most commercial offerings do not share their models or datasets, limiting reproducibility and academic benchmarking. 

To address the scarcity of open-source solutions for data anonymization in autonomous driving, PP4AV \cite{trinh2023pp4av} adopts the community version of dSpace Understand AI as a baseline. Additionally,  develop and release your baseline model using a knowledge-distillation approach, leveraging soft labels from state-of-the-art teacher models to overcome the limited availability of annotated driving data. In contrast, the IDD dataset represents real-world driving scenarios, presenting unique challenges for current state-of-the-art models developed on datasets without this level of complexity. This dataset offers an essential benchmark for evaluating privacy-preserving models in actual autonomous driving environments, bridging the gap between conventional datasets and realistic, diverse driving conditions. For face anonymization, we used the pretrained image inpainting model Inpaint-Anything \cite{yu2023inpaint} with the bounding boxes from the face detection model RetinaFace \cite{deng2019retinaface}




\section{Method}
\label{sec:impainting}

To ensure privacy in street-level imagery while preserving contextual scene information, we adopted a targeted face anonymization approach using the Inpaint-Anything model \cite{yu2023inpaint}, a state-of-the-art generative inpainting framework capable of filling masked regions with semantically consistent content. The core idea behind our method is to replace identifiable facial regions in images with visually coherent inpainted content that blends seamlessly with the surrounding environment.

\subsection{Face Anonymization}\label{sec:face_anno}
The anonymization pipeline begins with face detection using the RetinaFace model \cite{deng2019retinaface}, known for its high accuracy in detecting human faces under challenging real-world conditions. RetinaFace produces tight bounding boxes around detected faces, which are then used to generate binary masks indicating regions to be anonymized.

\begin{figure}
    \centering
    \includegraphics[width=0.99\linewidth]{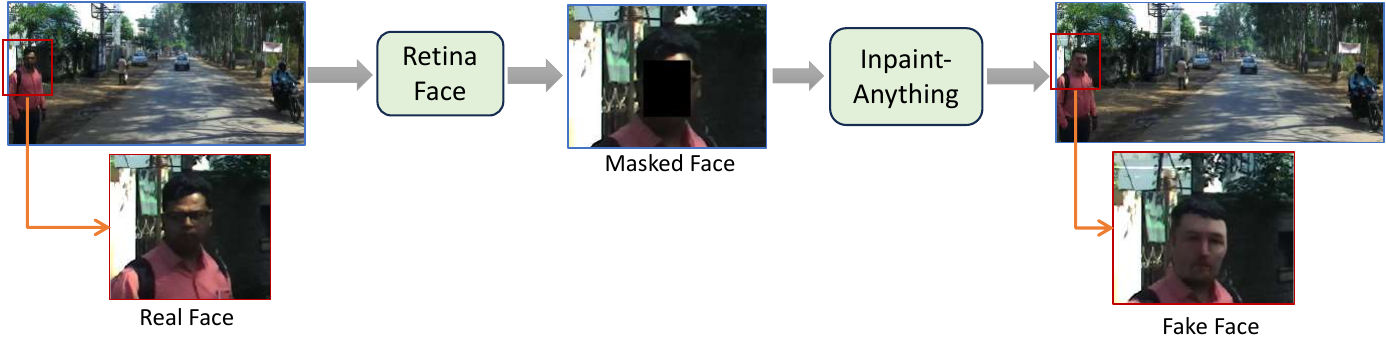}
    \caption{An overview of the face anonymization pipeline used to generate the training dataset. Faces are first detected using the RetinaFace model, followed by face region masking. The masked regions are then inpainted using the Inpaint-Anything model to generate synthetic yet contextually consistent faces. This ensures privacy while preserving visual and semantic coherence in the dataset. The validation set retains original faces for evaluation purposes.}
    \label{fig:idd_pvt_flow}
\end{figure}

These masked regions are passed to the Inpaint-Anything model, which synthesizes realistic replacements for facial areas using the surrounding pixel context. Importantly, this inpainting is confined strictly to the detected bounding boxes to ensure that the rest of the image remains completely untouched. This preserves critical scene semantics, such as pedestrian posture, spatial layout, and environmental features, while removing identifiable information. Such a design makes the anonymized data highly valuable for downstream tasks like pedestrian detection, person re-identification, or behavior prediction, where scene understanding is essential but identity information is not.

\subsection{Evaluating Anonymization}\label{sec:eval_anno}

To assess the effectiveness of our face anonymization pipeline, we perform a comparative evaluation using two parallel training setups: one with real face images and another with anonymized face images. In both cases, we train the same face detection model (e.g., YOLOv11) but on different training datasets. The first model is trained on the dataset containing real face images and evaluated on a validation set with real, unaltered faces. The second model is trained on the anonymized face dataset, where all detected faces are inpainted using the Inpaint-Anything model, and is likewise evaluated on the same validation set. This controlled comparison allows us to examine how anonymization affects model performance while keeping the evaluation benchmark consistent. The goal is to ensure that anonymization preserves essential scene information for downstream tasks, such as face detection, without compromising privacy.

For this evaluation, we apply the pipeline to the IDD dataset. First, we filter the dataset to retain only images that contain at least one detectable face. This results in a training set of 608 images with 3100 anonymized faces. The validation set consists of 153 images with 889 real, unmodified faces. This configuration allows us to measure how well the model trained on anonymized data generalizes to detecting real faces, directly measuring the anonymization impact.

We follow a similar procedure for the ECP dataset. After filtering for images with detectable faces, we obtain 7712 training images. The test set includes 1533 images with 3509 real, unaltered faces used for evaluation. Given the large scale and diversity of the ECP dataset, capturing various urban environments across Europe, it is a strong benchmark to validate the robustness and generalizability of our anonymization pipeline across complex real-world scenes.

\section{Results}
\label{sec:results_idd}

In Table~\ref{tab:res}, we present a comparative analysis of various YOLO-based models and our proposed variants evaluated on the Pvt-IDD dataset. As expected, increasing model size generally improves performance; however, our customized variants \textbf{yolov11l10} and \textbf{yolov11l19} achieve notable improvements in both precision and recall while maintaining a relatively compact model size (25M parameters, 87 GFLOPs). These models outperform standard YOLOv5 variants across all metrics, particularly in mAP50 and mAP50-90 for real (R) and anonymized (A) data. Notably, \textbf{yolov11l19} achieves the highest mAP50-90 score of 0.70 (A), indicating strong localization and generalization even in challenging privacy-preserving settings. This highlights the effectiveness of our proposed architecture modifications for face and license plate detection in real-world autonomous driving scenarios.

Table \ref{tab:loss} reports the validation losses in terms of box loss and objectness loss for both real (R) and anonymized (A) data across various models. Our proposed models, \textbf{yolov11l10} and \textbf{yolov11l19}, achieve the lowest box and objectness losses among all variants, indicating better convergence and more precise localization of faces and license plates. Notably, \textbf{yolov11l19} achieves the lowest box loss of 0.011 on anonymized data and an objectness loss of 0.007, showcasing its robustness in privacy-preserving scenarios. While larger models like DINO with SWIN backbone achieve competitive objectness loss, their box loss remains higher, particularly for anonymized inputs, emphasizing the effectiveness of our lightweight yet accurate model designs.

\subsection{Face Detection}
\label{ssec:det}

\paragraph{Training model for face detection on real images and anonymized images:}
The two figures show the training and validation performance of the YOLOv11l19 model on face detection tasks using datasets with real and anonymized faces, respectively. The training converges smoothly in both cases, consistently decreasing box loss, classification loss, and distribution focal loss across epochs. Precision, recall, and mean Average Precision (mAP) scores also improve steadily throughout training. Notably, the model trained on real face data (Figure \ref{fig:res_wF}) achieves slightly more stable and higher mAP and recall metrics compared to its anonymized counterpart (Figure \ref{fig:res_FF}), suggesting that facial anonymization introduces some challenges in localization and classification performance. However, the overall trends indicate successful learning in both setups. Yolov11l model for 1280 image size takes 0.2ms pre-process, 13.5ms inference, 0.9ms post-process per image.

\begin{figure*}
    \centering
    \includegraphics[width=0.99\linewidth]{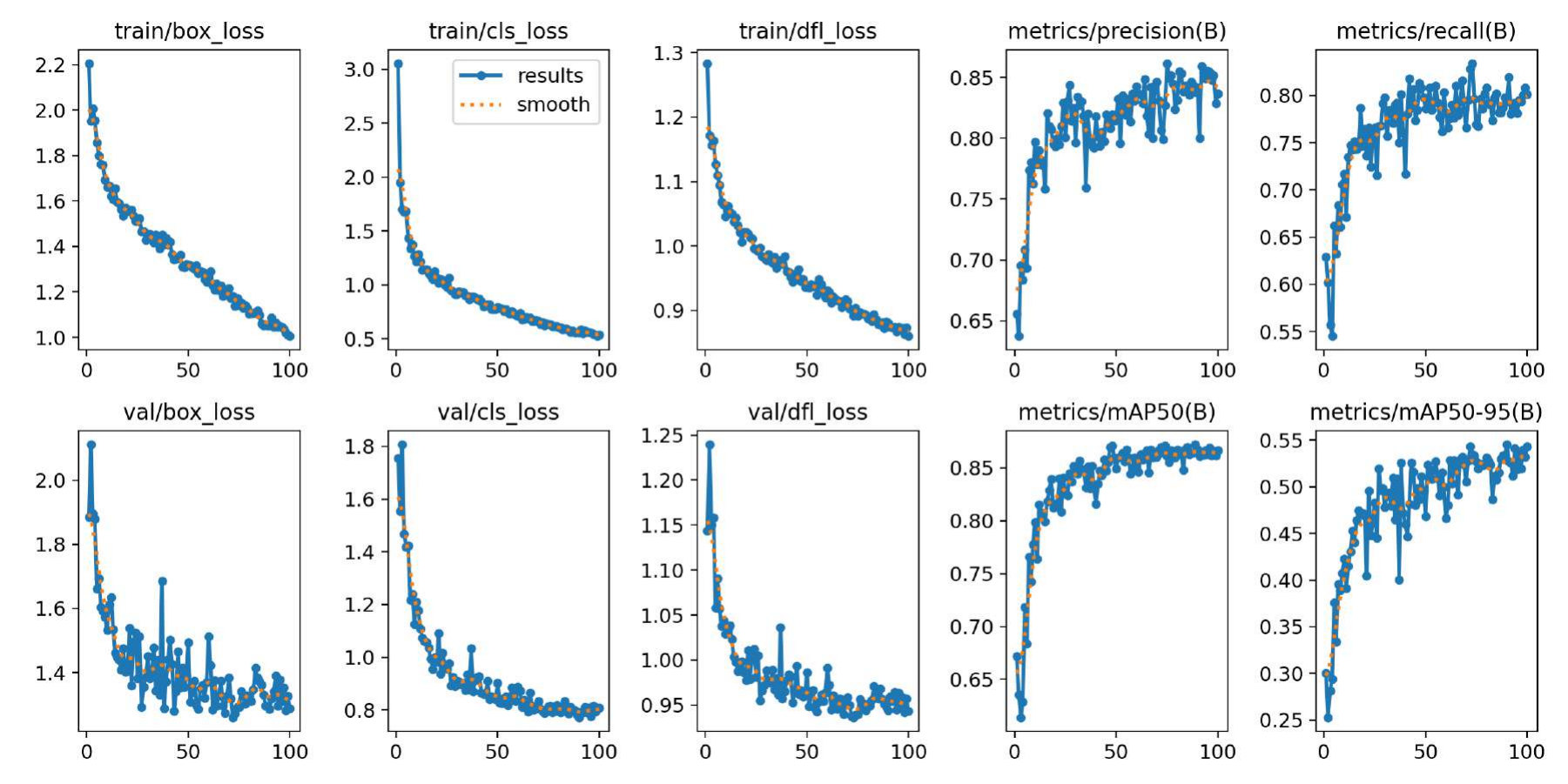}
    \caption{Face Detection training: validation result for the Yolov11l19 trained on the dataset with real faces and image size of $1920\times 1920$.}
    \label{fig:res_wF}
\end{figure*}

\begin{figure*}
    \centering
    \includegraphics[width=0.99\linewidth]{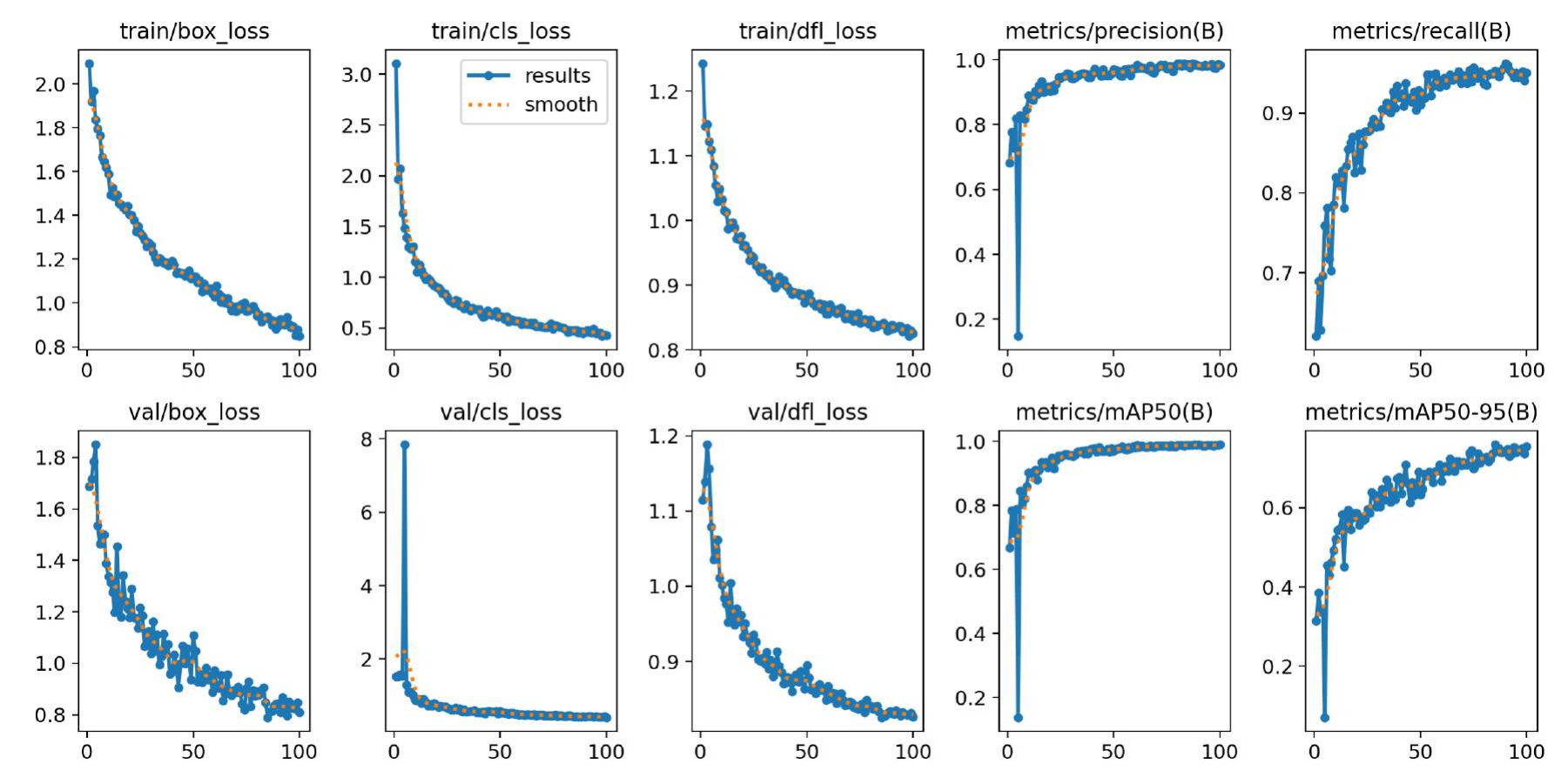}
    \caption{Face Detection training: validation result for the Yolov11l19 trained on the dataset with anonymized faces and image size of $1920\times 1920$.}
    \label{fig:res_FF}
\end{figure*}

\begin{table*}[!ht]
\centering
\resizebox{\columnwidth}{!}{
\renewcommand{\arraystretch}{1.5}
    \begin{tabular}{lcccccccccc}
    \toprule
     & \multicolumn{2}{c}{\textbf{Model Size}} & \multicolumn{2}{c}{\textbf{Precision}} & \multicolumn{2}{c}{\textbf{Recall}} & \multicolumn{2}{c}{\textbf{mAP50}}  & \multicolumn{2}{c}{\textbf{mAP50-90}}\\
    \cline{2-11} 
    \textbf{Model Name} & \#params (M)   & Flops (G)  & \textbf{R}   & \textbf{A}  & \textbf{R}   & \textbf{A}  & \textbf{R}   & \textbf{A} & \textbf{R}   & \textbf{A} \\ \hline
          yolov5n &    02  & 016 & 0.74 & 0.67 & 0.56 & 0.48 & 0.65 & 0.52  & 0.25 & 0.20\\ 
          yolov5s &    07  & 048 & 0.80 & 0.82 & 0.63 & 0.73 & 0.69 & 0.78 & 0.31 & 0.39 \\ 
          yolov5m &    20  & 108 & 0.76 & 0.80 & 0.65 & 0.72 & 0.69 & 0.77 & 0.35 & 0.35 \\
          yolov5l &    46  & 108 & 0.80 & 0.81 & 0.68 & 0.71 & 0.76 & 0.77 & 0.40 & 0.38 \\ 
          \hline
          yolov5s10 &  07  & 016 & 0.84 & 0.85 & 0.79 & 0.83 & 0.85 & 0.88 & 0.47 & 0.48 \\ 
          yolov5l6 &   76  & 110 & 0.85 & 0.92 & 0.80 & 0.85 & 0.87 & 0.92 & 0.53 & 0.56 \\
          \hline
          yolov11l10 & 25  & 087 & 0.86 & 0.93 & 0.80 & 0.90 & 0.88 & 0.95 & 0.55 & 0.62 \\
          yolov11l19 & 25  & 087 & 0.87 & 0.94 & 0.81 & 0.94 & 0.89 & 0.97 & 0.56 & 0.70 \\
          \bottomrule
    \end{tabular}
    }
\caption{640 input image size for first four rows, 1024 for 5th row, 1280 for next two rows, 1920 for last. R: trained on a dataset with real faces, F: trained on a dataset with anonymized faces. M: million, G: gigaFLOPS. Comparing performance metrics of various YOLOv5 and YOLOv11 models on face detection tasks. The table presents model size parameters (number of parameters in millions and computational complexity in GigaFLOPS), along with precision, recall, and mean Average Precision (mAP) metrics for both real (R) and anonymized (A) face datasets. Smaller models like YOLOv5n show lower performance but require fewer computational resources, while specialized models like YOLOv11l19 achieve the highest accuracy metrics, particularly on anonymized face detection tasks. Input image resolution increases progressively from 640 pixels for baseline models to 1920 pixels for the highest-performing YOLOv11 variants.}
\label{tab:res}
\end{table*}

\begin{table}[!ht]
    \centering
    \begin{tabular}{lcccc}
    \toprule
     & \multicolumn{2}{c}{\textbf{Box Loss}} & \multicolumn{2}{c}{\textbf{Obj Loss}} \\ 
    \cline{2-5} 
    \textbf{Model Name} & \textbf{R}   & \textbf{A}  & \textbf{R}   & \textbf{A}  \\ \hline
             yolov5n    & 0.038 & 0.024 & 0.011 & 0.009 \\ 
             yolov5s    & 0.030 & 0.024 & 0.008 & 0.008 \\
             yolov5m    & 0.029 & 0.024 & 0.008 & 0.008 \\ 
             yolov5l    & 0.022 & 0.023 & 0.010 & 0.010 \\
             \hline
             yolov5s10  & 0.032 & 0.032 & 0.017 & 0.017 \\ 
             yolov5l6   & 0.013 & 0.012 & 0.010 & 0.009 \\
             \hline
             yolov11l10 & 0.012 & 0.012 & 0.008 & 0.007 \\
             yolov11l19 & 0.012 & 0.011 & 0.007 & 0.007 \\
             DINO (ResNet50) & 0.030 & 0.060 & 0.008 & 0.008 \\
             DINO (SWIN) & 0.030 & 0.020 & 0.006 & 0.005 \\
             \bottomrule
    \end{tabular}
    \caption{Validation loss. DINO: 4 scales with ResNet-50 and Swin backbone. Comparing training loss metrics of YOLOv5, YOLOv11, and DINO models on face detection tasks. The table presents box loss and objectness loss values for models trained on real (R) and anonymized (A) face datasets. The results demonstrate a general trend of decreasing loss values as model complexity increases, with YOLOv11 variants achieving the lowest loss values overall. The DINO architecture with the Swin backbone shows competitive objectness loss performance despite varied box loss metrics. These validation loss measurements provide insight into model convergence and detection accuracy across architecture designs and dataset types.}
    \label{tab:loss}
\end{table}

\subsection{Segmentation}

In Tables \ref{tab:res_seg_bce} and \ref{tab:res_seg_dice}, we present segmentation results using BCE and Dice loss, respectively. Among all evaluated models, \textbf{UNet (ResNet34)} consistently outperforms others across real and anonymized datasets, achieving the highest mIoU and Dice scores with minimal performance drop due to anonymization. For instance, it maintains a Dice score of 0.975 on real and anonymized inputs with BCE loss and achieves a slightly better mIoU when trained with Dice loss. In contrast, DeepLabV3 and Fully Convolutional Networks (FCN) models show more significant performance drops, particularly in mIoU, highlighting the sensitivity of these architectures to facial anonymization. These results suggest that UNet-based architectures are more resilient to anonymized training data in pixel-wise segmentation tasks.

Tables \ref{tab:loss_bce} and \ref{tab:loss_dice} report the corresponding test losses for BCE and Dice loss functions. UNet models, especially the ResNet34 variant, exhibit lower and more stable loss values across both loss types and training settings, further reinforcing their robustness. While DeepLabV3 and FCN models show relatively low BCE loss values, their mIoU and Dice scores do not reflect similar reliability. Interestingly, Dice loss values in Table \ref{tab:loss_dice} show more variance between real and anonymized data, indicating sensitivity to label distribution and model generalization. These findings validate that specific segmentation architectures like UNet can preserve performance under anonymized conditions, making them more suitable for privacy-preserving visual learning tasks.

\begin{table*}[!ht]
\centering
\renewcommand{\arraystretch}{1.5}
    \begin{tabular}{lccccc}
    \toprule
     & \multicolumn{1}{c}{\textbf{Model Size}} & \multicolumn{2}{c}{\textbf{mIoU}} & \multicolumn{2}{c}{\textbf{Dice Coefficient}} \\
    \cline{2-6} 
    \textbf{Model Name} & \#params (M)  & \textbf{R}   & \textbf{A}  & \textbf{R}   & \textbf{A} \\ \hline
          UNet (ResNet34)       & 24 & 0.714 & 0.708 & 0.975 & 0.975  \\
          UNet (ResNet50)       & 32 & 0.709 & 0.708 & 0.975 & 0.974  \\
          DeepLabV3 (ResNet50)  & 39 & 0.600 & 0.588 & 0.962 & 0.961  \\
          DeepLabV3 (ResNet101) & 58 & 0.605 & 0.600 & 0.962 & 0.960  \\
          FCN (ResNet50)        & 32 & 0.630 & 0.619 & 0.964 & 0.963 \\
          FCN (ResNet101)       & 51 & 0.609 & 0.600 & 0.961 & 0.962 \\
          \bottomrule
    \end{tabular}
\caption{R: trained on dataset with real faces, A: trained on dataset with anonymized faces. M: million. Loss: \textbf{BCE}. Comparing segmentation performance of various deep learning architectures for face segmentation tasks. The table presents model size (number of parameters in millions) alongside mean Intersection over Union (mIoU) and Dice Coefficient metrics for models trained on both real (R) and anonymized (A) face datasets. UNet architectures with ResNet backbones demonstrate superior performance across both metrics compared to DeepLabV3 and FCN alternatives, despite having fewer parameters in some cases. Notably, segmentation performance remains relatively consistent between real and anonymized face datasets across all model configurations, with only minimal degradation observed when processing anonymized faces. All models were trained using the Binary Cross-Entropy (BCE) loss function.}
\label{tab:res_seg_bce}
\end{table*}

\begin{table*}[!ht]
\centering
\renewcommand{\arraystretch}{1.5}
    \begin{tabular}{lccccc}
    \toprule
     & \multicolumn{1}{c}{\textbf{Model Size}} & \multicolumn{2}{c}{\textbf{mIoU}} & \multicolumn{2}{c}{\textbf{Dice Coefficient}} \\
    \cline{2-6} 
    \textbf{Model Name} & \#params (M)  & \textbf{R}   & \textbf{A}  & \textbf{R}   & \textbf{A} \\ \hline
          UNet (ResNet34) & 24 & 0.726 & 0.723 & 0.975 & 0.975  \\
          UNet (ResNet50) & 32 & 0.710 & 0.702 & 0.974 & 0.973  \\
          \bottomrule
    \end{tabular}
\caption{R: trained on dataset with real faces, A: trained on dataset with anonymized faces. M: million. Loss: \textbf{Dice}. Comparing the segmentation performance of UNet architectures with different ResNet backbones when trained using the Dice loss function. The table presents model size (number of parameters in millions) alongside mean Intersection over Union (mIoU) and Dice Coefficient metrics for models trained on both real (R) and anonymized (A) face datasets. Results indicate that the lighter UNet (ResNet34) model achieves slightly superior performance compared to the larger UNet (ResNet50) variant across both metrics. Performance remains highly consistent between real and anonymized face datasets, with minimal differences in segmentation quality. Compared to BCE loss results (from the previous table), the Dice loss function appears to yield modest improvements in mIoU scores while maintaining equivalent Dice Coefficient values.}
\label{tab:res_seg_dice}
\end{table*}

\begin{table}[!ht]
    \centering
    \begin{tabular}{lcc}
    \toprule
     & \multicolumn{2}{c}{\textbf{BCE Loss}}  \\ 
    \cline{2-3} 
    \textbf{Model Name} & \textbf{R}  & \textbf{A}  \\ \hline
             UNet (ResNet34) & 0.178 & 0.184 \\ 
             UNet (ResNet50) & 0.204 & 0.204 \\
             DeepLabV3 (ResNet50) & 0.105 & 0.109 \\ 
             DeepLabV3 (ResNet101) & 0.100 & 0.116 \\
             FCN (ResNet50) & 0.099 & 0.102 \\ 
             FCN (ResNet101) & 0.094 & 0.105 \\
             \bottomrule
    \end{tabular}
    \caption{Test loss: Presenting Binary Cross Entropy (BCE) loss values for various segmentation models on test datasets containing both real (R) and anonymized (A) faces. Test BCE loss for segmentation models on real (R) and anonymized (A) face datasets. Despite strong performance metrics elsewhere, UNet shows higher loss, while FCN-ResNet101 performs best. Loss differences between R and A are minimal across models.}
    \label{tab:loss_bce}
\end{table}

\begin{table}[!ht]
    \centering
    \begin{tabular}{lcc}
    \toprule
     & \multicolumn{2}{c}{\textbf{Dice Loss}}  \\ 
    \cline{2-3} 
    \textbf{Model Name} & \textbf{R}  & \textbf{A}  \\ \hline
             UNet (ResNet34) & 0.560 & 0.618 \\ 
             UNet (ResNet50) & 0.659 & 0.544 \\
    \bottomrule
    \end{tabular}
    \caption{Test loss: presenting Dice loss values for UNet segmentation models with different ResNet backbones on test datasets containing both real (R) and anonymized (A) faces. The UNet (ResNet34) model shows lower Dice loss on real faces than the UNet (ResNet50) variant, but this relationship reverses when processing anonymized faces. Notably, the UNet (ResNet50) demonstrates significantly better loss performance on anonymized faces than on real faces, while the opposite pattern appears for the UNet (ResNet34) architecture. These divergent trends in loss values between model variants and dataset types highlight potential differences in how these architectures respond to face anonymization during training, despite the relatively consistent performance metrics (mIoU and Dice Coefficient) reported in previous tables.}
    \label{tab:loss_dice}
\end{table}

\subsection{Limitations} To assess the robustness of the proposed privacy-preserving method, we conducted a validation study on the Euro City Person (ECP) \cite{braun2018eurocity} daytime dataset, which contains approximately 40,000 training and testing images. Among these, 7,712 training images include visible faces, along with 1,533 validation images. Due to the generally low visibility in this dataset, facial features are often not discernible. We employed RetinaFace to detect facial bounding boxes and used Inpaint-Anything to replace the detected regions with synthetic faces. However, the quality of these modified images is noticeably lower compared to those in the Pvt-IDD dataset (see Figure \ref{fig:samples2}, \ref{fig:samples3}, and Table \ref{tab:res_ecp}). 

\begin{figure}[!ht]
    \centering
    \includegraphics[width=\textwidth]{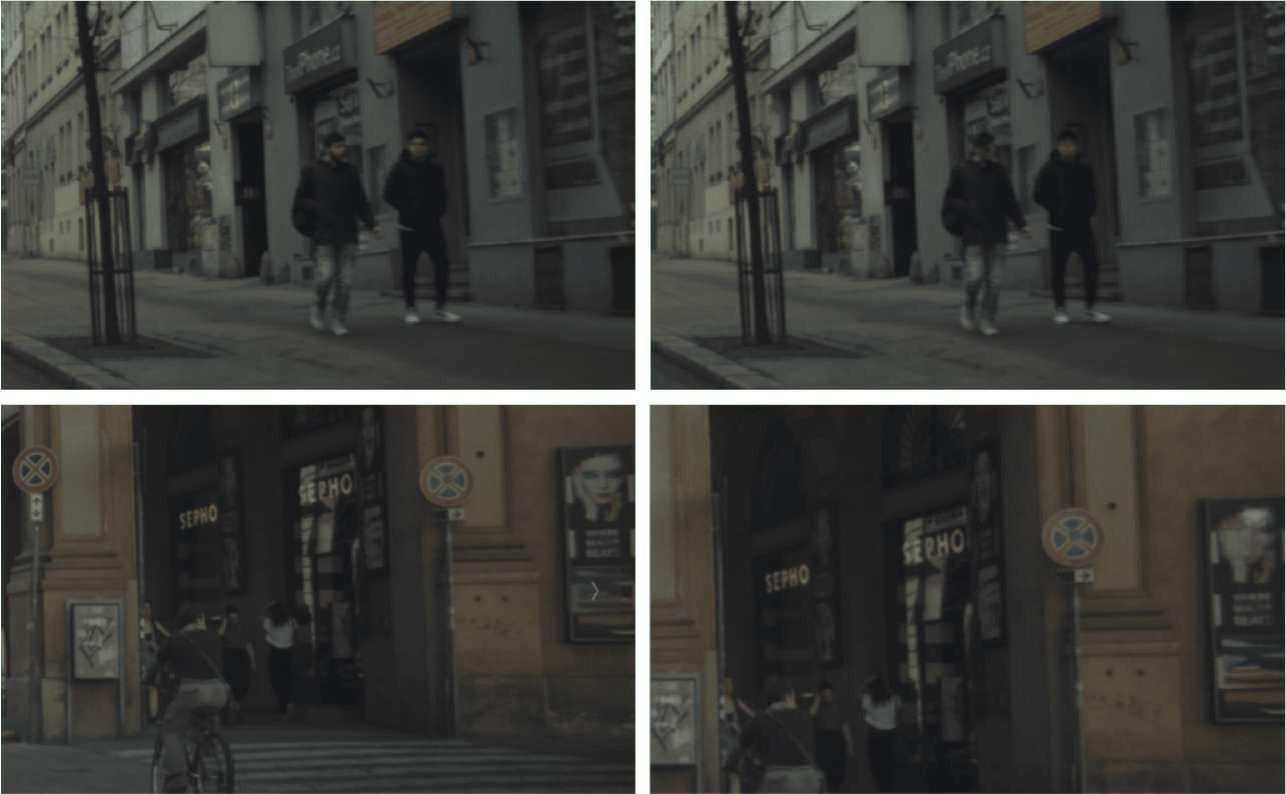}
    \caption{Cropped and zoomed image inpainting samples for the ECP dataset \cite{braun2018eurocity}  images.}
    \label{fig:samples2}
\end{figure}

\begin{figure}[!ht]
    \includegraphics[width=\textwidth]{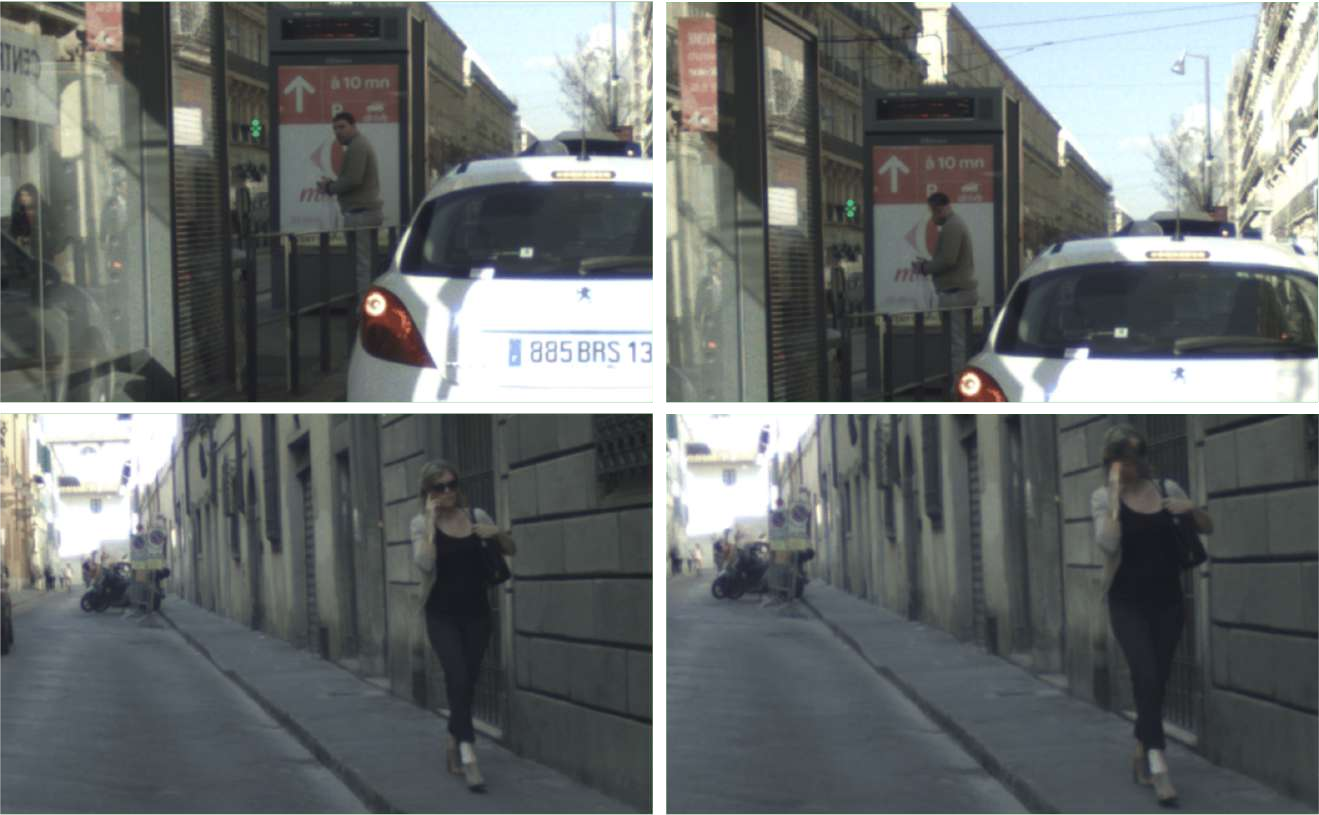}
    \caption{Cropped and zoomed image inpainting samples for the ECP dataset \cite{braun2018eurocity}  images.}
    \label{fig:samples3}
\end{figure}

The results in Table \ref{tab:res_ecp} indicate that face inpainting on the ECP dataset significantly reduces detection performance, with precision and mAP scores notably lower for models trained on anonymized (A) data compared to real (R) data. This drop highlights that the inpainted faces generated using Inpaint-Anything are of poor quality, especially compared to those in the Pvt-IDD dataset, ultimately degrading the effectiveness of downstream face detection models.
\begin{table}[!ht]
\centering
\resizebox{\columnwidth}{!}{
\renewcommand{\arraystretch}{1.5}
    \begin{tabular}{lcccccccccc}
    \toprule
     & \multicolumn{2}{c}{\textbf{Model Size}} & \multicolumn{2}{c}{\textbf{Precision}} & \multicolumn{2}{c}{\textbf{Recall}} & \multicolumn{2}{c}{\textbf{mAP50}}  & \multicolumn{2}{c}{\textbf{mAP50-90}}\\
    \cline{2-11} 
    \textbf{Model Name} & \#params (M)   & Flops (G)  & \textbf{R}   & \textbf{A}  & \textbf{R}   & \textbf{A}  & \textbf{R}   & \textbf{A} & \textbf{R}   & \textbf{A} \\
          \hline
          yolov11l10 &  25 & 87 & 0.50 & 0.40 & 0.73 & 0.55 & 0.51 & 0.37 & 0.49 & 0.39 \\
          yolov11l19 &  25 & 87 & 0.42 & 0.30 & 0.60 & 0.35 & 0.40 & 0.25 & 0.30 & 0.22 \\
          \bottomrule
    \end{tabular}
    }
\caption{Face Detection on ECP val dataset: The image size of 1280 for the first rows, 1920 for the second. R: trained on a dataset with real faces, A: trained on a dataset with anonymized faces. M: million, G: gigaFLOPS.}
\label{tab:res_ecp}

\end{table}

\section{Summary}

This work presents Pvt-IDD, the first dataset annotated with faces and license plates specifically for autonomous driving applications in Indian urban settings. Our experiments will reveal significant challenges Pvt-IDD poses to current state-of-the-art face and license plate detection models, highlighting gaps in privacy-preserving capabilities. We aim for Pvt-IDD to stimulate further research into robust anonymization techniques for autonomous driving.

Building on these findings, we will propose a new baseline face and license plate detection model by refining the YOLOX architecture. This approach, which operates without needing labeled data, surpasses several strong, generic, state-of-the-art models, demonstrating its utility in autonomous driving contexts. To advance the field, we plan to publish an in-depth failure analysis that explores current face and license plate detection model limitations, offering insights to guide future developments.

In conclusion, our experiments demonstrate that adequate face anonymization using the Inpaint-Anything model, applied specifically within bounding boxes detected by RetinaFace, preserves the utility of visual data for downstream tasks such as segmentation and detection. Across segmentation benchmarks, models trained on anonymized datasets exhibited comparable performance to those trained on real-face datasets, with minimal drop in mIoU and Dice scores. Similarly, the face detection models trained on anonymized data maintained substantial precision, recall, and mAP metrics, suggesting that the anonymization process does not significantly degrade detection capability. These results collectively highlight that our selective, bbox-level anonymization approach achieves a practical balance between data privacy and task performance, making it a viable strategy for privacy-preserving computer vision applications.

Future directions include expanding Pvt-IDD with a wider variety of real-world scenarios to support privacy-preserving techniques under diverse conditions. We recognize that our dataset and models could inadvertently propagate biases and stereotypes present in large vision models, posing ethical risks. To address this, we plan to investigate methods, including rule-based filtering and tailored classifiers, for identifying and excluding harmful data, ensuring more responsible application of autonomous driving technologies.

 \chapter{Geological Mapping and Generative Models}\label{chap:geo_vae}

\section{Introduction}

Geological mapping is crucial for various purposes, such as assessing the mineralization potential of a region and creating prospectivity maps \citep{bachri2019machine,wang2021lithological,shirmard2022review}. Remote sensing provides an efficient alternative to traditional fieldwork for geological mapping, which is often costly and sometimes impractical due to harsh topography or political constraints \citep{yu2012towards}. Optical remote sensing images, captured by scanners on platforms such as satellites, cover multiple spectral bands ranging from the visible to infrared regions of the electromagnetic spectrum \citep{clark2003imaging}. The detailed information in multispectral images aids geological analysis by identifying and mapping rocks and minerals based on their specific spectral absorption properties \citep{chen2010integrating,weilin2016application,lu2021lithology}. This spectral data forms the foundation for spectrum-based approaches in classifying and mapping image pixels. The spectral and spatial resolution of remote sensing data, especially from satellites, allows for identifying rock units over extensive areas \citep{pour2018mapping}. However, geological mapping via remote sensing, particularly optical data, faces different challenges due to the influence of land cover, which can obscure rock outcrops, introduce spectral mixing, and vary seasonally, complicating the interpretation of spectral signatures. Geological mapping focuses on identifying and characterizing rock types, geological formations, and structures to understand Earth's history and resources. In contrast, land use/land cover mapping classifies and monitors human activities and natural features on the Earth's surface for urban planning, agriculture, and environmental management \citep{bouslihim2022comparing,tabassum2023exploring,akanbi2024integrating,amare2024impacts,masoudi2024assessment}. Land cover types such as vegetation and urban structures can mask or mimic rock characteristics, and shadows can mislead classifications \citep{hashim2013automatic,bentahar2020fracture,shebl2021reappraisal,shirmard2022review}. However, integrating optical data with other remote sensing data types and employing advanced data analytics tools enhances accuracy, making optical data a valuable tool for geological mapping despite the challenges posed by land cover \citep{galdames2019rock,ran2019rock}. 

Standard geological mapping techniques based on remote sensing data processing involve comparing absorption features to reference spectra or training samples \citep{chen2010integrating}. However, collecting sufficient pixels for reference spectra or training samples is challenging. Geological processes influence the spectral variability in rocks, which depends on their chemical and mineral composition, grain size, texture, and structure \citep{sgavetti2006reflectance}. This spectral variability significantly impacts geological mapping using remote sensing data. Insufficient training samples with highly correlated spectral bands often lead to challenges in discriminating rock units \citep{bruzzone2014review}. Therefore, extracting useful information from images and removing redundant information is essential to improve geological discrimination \citep{sun2019hyperspectral}. Machine learning models, particularly deep learning models, have proven to be powerful tools for extracting valuable information from remote sensing data and mapping geological anomalies \citep{zhao2023recognition,dou2024large}. Such models utilize dimensionality reduction, clustering, and classification approaches to automatically analyze and interpret complex datasets \citep{bedini2009mapping, carneiro2012semiautomated,sahoo2017pattern, awange2020hybrid,dou2024remote}. In remote sensing, machine learning models have found widespread applications \citep{zuo2019deep,dou2024time}, and their adoption in mineral exploration is gaining significant traction \citep{shirmard2022comparative}. These models have been integrated with conventional image processing techniques and geological surveys, significantly enhancing remote sensing for geological mapping and mineral prospectivity mapping \citep{shirmard2022review,guo2023gis,hajaj2024review}.

Unsupervised machine learning techniques, such as dimensionality reduction and transformation methods, have shown remarkable efficiency in distinguishing between geological units, making them invaluable for identification and mapping purposes \citep{behnia2012remote}. Data compression and transformation techniques such as principal component analysis (PCA) \citep{wold1987principal}, independent component analysis (ICA) \citep{comon1994independent,forootan2012independent}, and minimum noise fraction (MNF) \citep{nielsen2010kernel} have the ability to suppress irradiance that dominates the bands of remote sensing data \citep{gao2017optimized}, thereby enhancing the spectral reflectance of geological features \citep{richards1999remote}. Supervised learning models can partially automate the extraction of features specifically related to the labeled data and address some of the challenges of semi-manual analysis. However, supervised learning requires ground truth data and expert input, which are costly and can introduce bias \citep{gewali2018machine}. The availability of fully labeled training data for geological features is limited, compounding the challenges supervised learning faces. Supervised learning models have been applied to multivariate datasets, such as multispectral remote sensing data, to extract specific spectral responses from different rock units \citep{asadzadeh2016review, nalepa2020unsupervised}; however, the large spectral variability and limited availability of training data make classification a challenging task in the geological analysis of remotely sensed data.

Autoencoders are a class of machine learning models primarily used for unsupervised learning tasks such as dimensionality reduction and feature learning \citep{kramer1992autoassociative,kingma2019introduction,li2023comprehensive}. They consist of two main components: an encoder and a decoder. Autoencoders provide a lower-dimensional (reduced) representation of the data \citep{wang2016auto,kingma2019introduction} using a latent vector that enables the data to be represented with fewer features. The latent vector has been used in remote sensing data processing to extract features or structures, such as geological units \citep{protopapadakis2021stacked}, and studies have shown that the features in the latent vector correspond to different types of minerals \citep{calvin2018band,gao2021generalized}. Latent vectors play a crucial role in maintaining feature independence, essential for preventing the mixing of different mineral types \citep{protopapadakis2021stacked}. This separation ensures that each mineral type or rock unit is accurately represented and distinguished, enhancing the clarity and reliability of the analysis.

Stacking is an ensemble learning approach \citep{sagi2018ensemble} that has been shown to enhance the performance of autoencoders. Stacked autoencoders \citep{vincent2010stacked} have proven helpful in various applications, such as feature extraction for multi-class change detection in hyperspectral images \citep{lopez2018stacked} and the classification of multispectral and hyperspectral images \citep{ozdemir2014hyperspectral,lv2017remote}. They provide a powerful framework for learning deep data representations in an unsupervised manner, which is particularly beneficial for machine learning tasks where labeled data is scarce or expensive to obtain \citep{vincent2010stacked}. In a recent study, \cite{protopapadakis2021stacked} addressed noise in input signals due to dimensionality redundancy without losing important features using a stacked autoencoder. Additionally, stacked autoencoders can help in understanding the nonlinearity between spectral bands and distinguishing complex features, such as geological units \citep{dai2023optimization}.

The extracted latent vectors from dimensionality reduction or transformation techniques can be used as inputs for clustering methods, such as $k$-means clustering, to group pixels and map geological units \citep{davies1979cluster,gao2021generalized}. Clustering methods are unsupervised learning techniques \citep{davies1979cluster,omran2007overview} that organize a dataset into groups based on the similarity of samples (instances) using various distance measures \citep{xie2016unsupervised,yadav2019study}. These methods have been widely used in remote sensing applications alongside dimensionality reduction techniques \citep{rodarmel2002principal}. Examples include pixel clustering \citep{bandyopadhyay2007multiobjective}, fuzzy clustering for change detection \citep{ghosh2011fuzzy}, image segmentation \citep{fan2009single}, mean-shift clustering of multispectral imagery \citep{bo2009mean}, and hyperspectral image subspace clustering involving dimension reduction, subspace identification, and clustering \citep{zhang2016spectral}. However, clustering methods face challenges due to the curse of dimensionality, which complicates their use in remote sensing data processing.

In large datasets, such as multispectral remote sensing data covering vast areas, the combined power of autoencoders and clustering can address the limitations and challenges of conventional methods. A major challenge in remote sensing-based geological mapping arises from the intricate and diverse nature of geological features, coupled with the often remote and inaccessible locations of target areas \citep{pal2020optimized,dos2021deep,shirmard2022review}. Although conventional dimensionality reduction methods may struggle with nonlinear data, autoencoders have been effective in modeling nonlinear relationships. Stacked autoencoders, featuring multiple interconnected layers that capture hierarchical data representation, can be useful for remote sensing data processing. Leveraging autoencoders with clustering methods has the potential to provide accurate geological maps.

In this study, we present an unsupervised machine-learning framework that combines stacked autoencoders for dimensionality reduction with $k$-means clustering to map geological units, aiding in the identification of potential mineralized areas. The role of the stacked autoencoders is to compress data with a wide range of spectral and spatial features, enhancing both the accuracy and efficiency of geological mapping. We use $k$-means clustering in our framework to generate clustered maps from the reduced dimensions. We evaluate our framework across multiple multispectral remote sensing datasets (Landsat 8, ASTER, and Sentinel-2) to map geological units in the Mutawintji region of Western New South Wales (NSW), Australia. We also compare our results with PCA and canonical autoencoders and provide open-source code to further extend the study.

\section{Materials and methods}
\paragraph{Geological setting:}

The Mutawintji region is located in the far west of NSW and the semi-arid zone of the state. It is approximately 1,150 kilometers (km) west of Sydney and covers an area of approximately 700 $km^2$ within the Curnamona Province, a geological province that covers a large area of southeastern Australia. As shown in Fig. \ref{fig_1}(a), the study area is situated on the eastern margin of the Curnamona Province, characterized by a thick sequence of sedimentary rocks deposited in a shallow marine environment during the Cambrian and Ordovician periods \citep{hewson2005seamless}. The geological setting of the study area is dominated by sedimentary rocks of the Cambrian and Ordovician periods, which are around 500 to 480 million years old. However, Quaternary residual and colluvial deposits cover a significant part of the sedimentary rocks \citep{young2009ordovician}. The sedimentary rocks comprise various rock types, including sandstone, shale, siltstone, and limestone. These rocks were formed by the accumulation of sediments in an ancient sea that covered much of the region during the Cambrian and Ordovician periods. The sediment was later buried and compressed, eventually forming the sedimentary rocks that can be observed on the surface. In addition to the sedimentary rocks, the geology of the study area also includes a range of other rock types, including volcanic rocks and granites. Major faults in the study area strike North-South or North West--South East, which separates Ordovician quartzite and sandstone units from shale, siltstone, and limestone in the southwest. Fig. \ref{fig_1}(b) presents the geology of the study area, characterized by a complex and diverse range of rock types that reflect the region's long and varied geological history and make it an interesting area for our study.

\begin{figure*}
    \centering
    \includegraphics[width=0.85\textwidth]{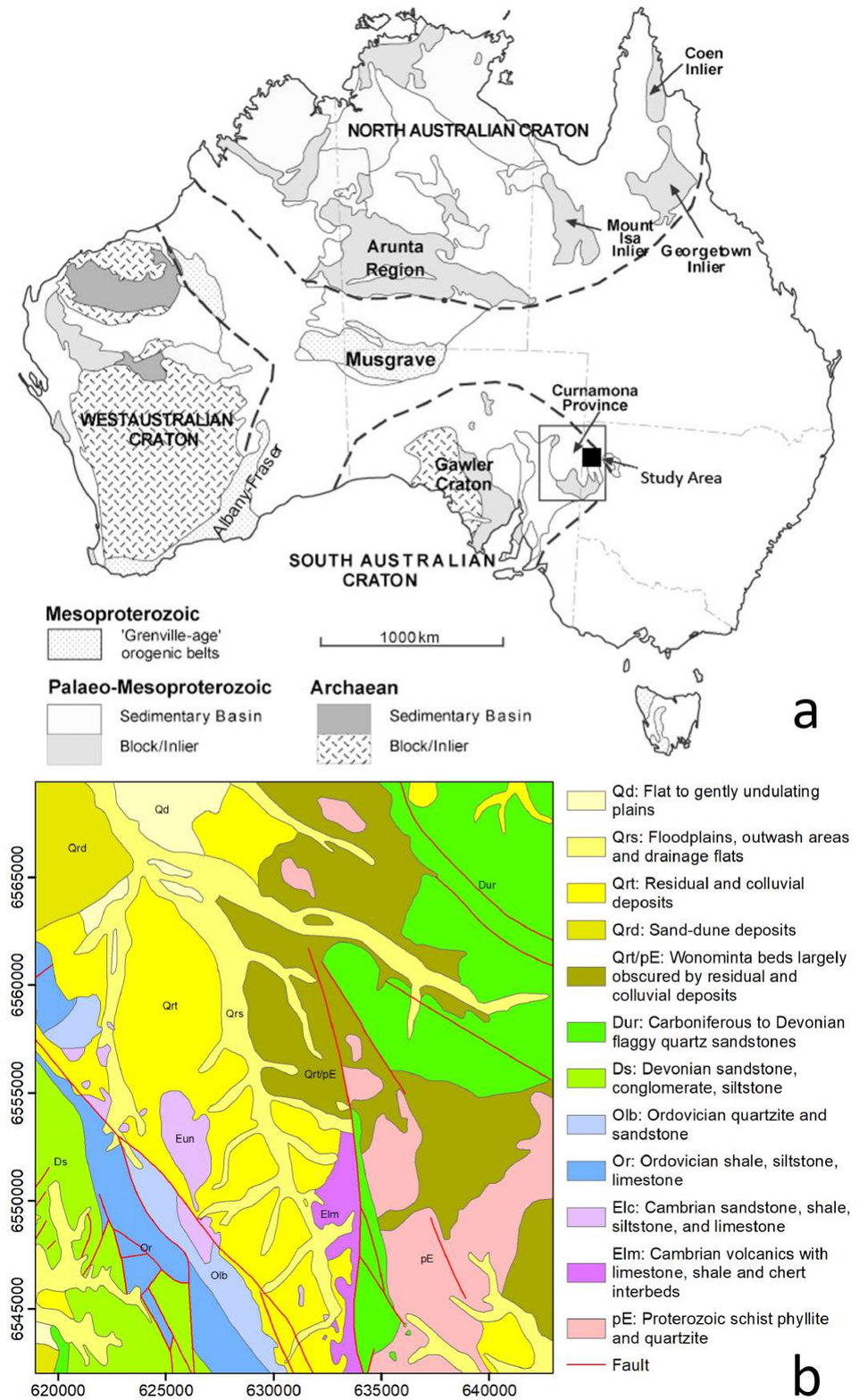}
    \caption{a) The Curnamona Province and other Proterozoic terrains in Australia \citep{barovich2008tectonic}; the study area is shown using a black square. b) Simplified geological map of the study area.}
    \label{fig_1}
\end{figure*}

\paragraph{Remote sensing data and pre-processing:}

Our study utilizes three types of multispectral remote sensing data, each with its unique capabilities and resolutions. Landsat 8 (launched in 2013) is equipped with two sensors—the Operational Land Imager (OLI) and the Thermal Infrared Sensor (TIRS). It provides images in 11 different spectral bands, with resolutions ranging from 15 meters (m) for the panchromatic band to 30 m for the visible and near-infrared (VNIR) and shortwave infrared (SWIR) bands. The thermal bands numbered 10 and 11 have a resolution of 100 m \citep{zhang2016integrating}. Geological remote sensing for mapping has significantly improved with the launch of the ASTER sensor on the Terra platform in 1999. ASTER's VNIR bands have a spatial resolution of 15 m, six SWIR bands with a resolution of 30 m, and five thermal infrared bands with a resolution of 90 m \citep{rowan2003lithologic}. Sentinel-2A and Sentinel-2B are twin satellites in sun-synchronous orbit, phased 180 degrees apart. Their onboard multispectral instrument captures data in 13 spectral bands, ranging from VNIR to SWIR, with spatial resolutions varying from 10 to 60 m \citep{drusch2012sentinel2}. The use of these diverse data types will enable us to comprehensively map the geological units in the study area.

In this study, we focus on spectral bands that are particularly important in geological remote sensing due to their characteristic behaviors, such as high absorption or reflectance in different geological units, which allow for the generation of meaningful geological maps. Accordingly, we select seven bands from OLI (bands 1--7), nine bands from ASTER (bands 1--9), and ten bands from Sentinel-2 (bands 2--8, 8a, 11, and 12). We obtained a cloud-free Landsat 8 scene of the study area from the US Geological Survey Earth Resources Observation and Science (USGS EROS) center\footnote{\url{https://earthexplorer.usgs.gov} (accessed on 31 January 2022)}. The image, captured on 5 October 2021, is a level-1T product that has been terrain-corrected. The acquired ASTER image of the study area was captured on 10 August 2001; this cloud-free level-1-precision terrain-corrected product (ASTER\_L1T) was also obtained from the USGS EROS center. Additionally, we downloaded a cloud-free Sentinel-2A scene of the study area, captured on 19 March 2022, from the European Space Agency via the Copernicus Open Access Hub\footnote{\url{https://scihub.copernicus.eu/} (accessed on 28 March 2022)}. This Sentinel-2 image is a level-1C product that has undergone radiometric and geometric corrections and orthorectification, resulting in top-of-atmosphere reflectance values.

The remote sensing datasets used in this study are pre-georeferenced to the Universal Transverse Mercator (UTM) zone 54 South, eliminating the need for geometric correction. The Landsat 8 OLI and ASTER data have been radiometrically corrected, and the reflectance data are used as inputs to the workflow. To match the spatial resolution of the VNIR bands (15 m), the SWIR bands in the ASTER data are resampled using the nearest neighbor technique \citep{taunk2019brief}. After resampling, a data layer is created by stacking the VNIR and SWIR bands for further processing. The Sentinel-2 image includes atmospheric correction. We stack the Sentinel-2 VNIR and SWIR bands using the nearest neighbor technique to create a 10-band dataset with 10 m spatial resolution. Before processing, all images are resized to fit the target area size. Figure \ref{fig_2} shows false color composite images created by stacking images from the Landsat 8, ASTER, and Sentinel-2 datasets.

\begin{figure*}
    \centering
    \includegraphics[width=\textwidth]{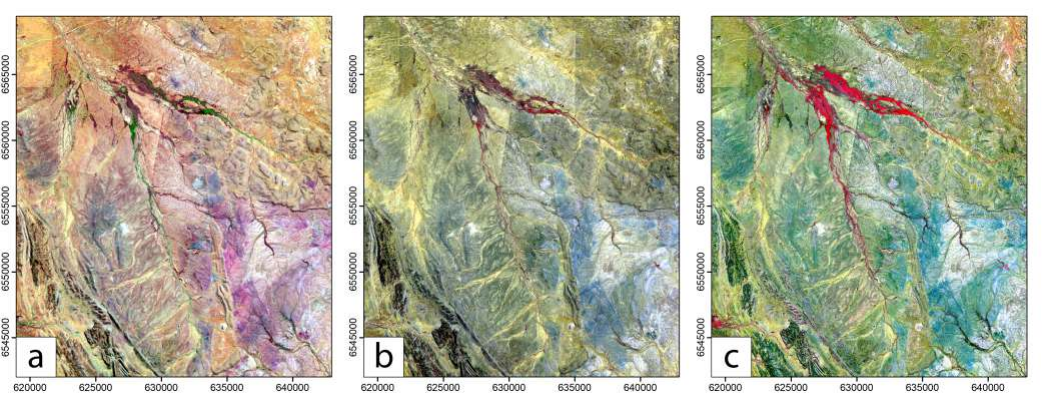}
    \caption{False color composite images generated using a) Landsat 8 (RGB 753), b) ASTER (RGB 321), and c) Sentinel-2 (RGB 843) data.}
    \label{fig_2}
\end{figure*}

\paragraph{Stacked autoencoders:}

The autoencoder is a dimensionality reduction technique used to uncover a lower-dimensional manifold, also known as the latent space (vector) of intrinsic dimensionality, while preserving the essential information present in the original data. Fig. \ref{fig_3}(a) depicts a canonical autoencoder consisting of an encoder, which reduces the dimensionality of the input data, and a decoder, which reconstructs the original data from the encoded representation. Unlike conventional dimensionality reduction methods such as PCA, which find linear combinations of the original features \citep{mackiewicz1993principal,abdi2010principal}, autoencoders can learn more abstract and higher-level features of the data. This capability is particularly useful when the data has complex patterns and structures. Although PCA is relatively simpler to implement and use for dimensionality reduction, training autoencoders can be more complex and require more training time and computational resources \citep{lv2017remote}.

The stacked autoencoder architecture comprises multiple layers of autoencoders, each trained independently as an individual autoencoder \citep{xu2019review}, as shown in Fig. \ref{fig_3}(b). The output from one layer serves as the input for the next, enabling the network to learn hierarchical representations of the data. Fig. \ref{fig_3}(b) presents a layered autoencoder with three encoders and decoders stacked sequentially \citep{dai2023optimization}. A typical stacking approach involves at least two layers: the first layer contains several models (any machine learning model), and the second layer combines the predictions using a simpler model that is also trained. In the context of a multispectral image dataset, the first layer captures meaningful features and patterns, as depicted in Fig. \ref{fig_4}. Additionally, the stacked architecture acts as a regularizer, preventing overfitting by forcing the model to learn more generalized features. Each layer acts as a feature extractor, reducing the data's dimensionality and encouraging the model to learn more generalizable features. Stacked autoencoders can extract features from the input data for use in other machine learning models, thereby improving their performance \citep{zhou2019learning}.

Unlike linear dimensionality reduction methods such as PCA, autoencoders do not aim to preserve specific attributes such as distance or topology. In scenarios where the relationships between input features are deep and nonlinear \citep{zhang2018automated}, traditional dimensionality reduction methods often fail to yield satisfactory results \citep{zhong2021advances}. Recognizing this limitation, the stacked autoencoder was developed, as the canonical autoencoder alone may struggle to address the nonlinearity inherent in many applications \citep{li2021distributed}. The architecture of a stacked autoencoder is entirely learned from the data, ensuring that the model adapts to the data's characteristics without the need for manual selection of nonlinear functions. In many cases, determining whether the relationship between spectral bands in a remote sensing dataset and target characteristics is linear or nonlinear can be challenging. The stacked autoencoder addresses both global and local characteristics within the dataset while reducing its dimensionality \citep{zhang2018local}.

\begin{figure}
    \centering
    \includegraphics[width=0.99\textwidth]{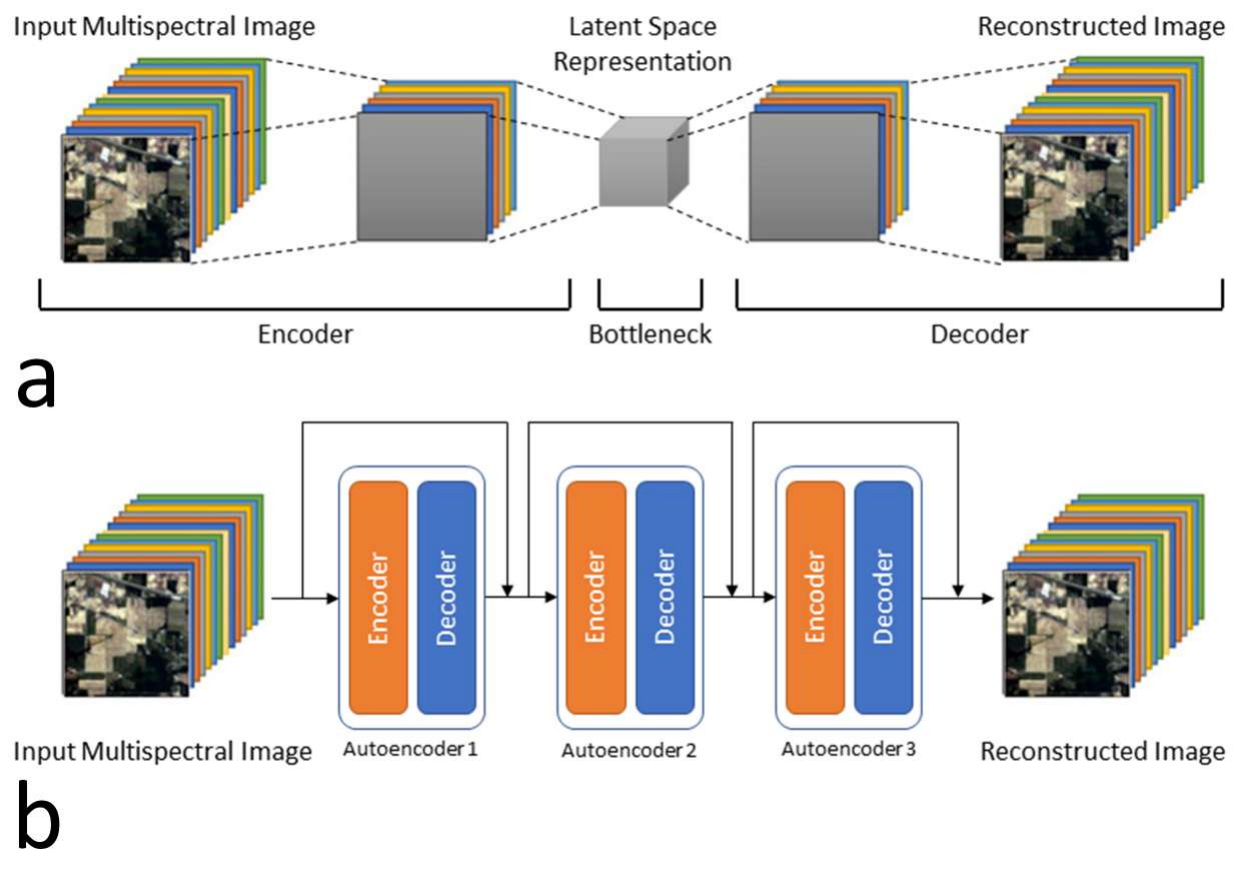}
    \caption{a) The architecture of a canonical autoencoder consists of an encoder and a decoder. The encoder takes the multispectral image as input ($x$) and reduces the dimension to the latent vector ($z$), where dim($x$) $>=$ dim($z$). The decoder reconstructs the image from the latent vector ($Z$). b) A stacked autoencoder with three encoders and decoders for each stacking level. Each stack level's encoder and decoder architecture is the same as the canonical autoencoder, with a number of hidden layers for each.}
    \label{fig_3}
\end{figure}

\begin{figure}
    \centering
    \includegraphics[width=0.79\textwidth]{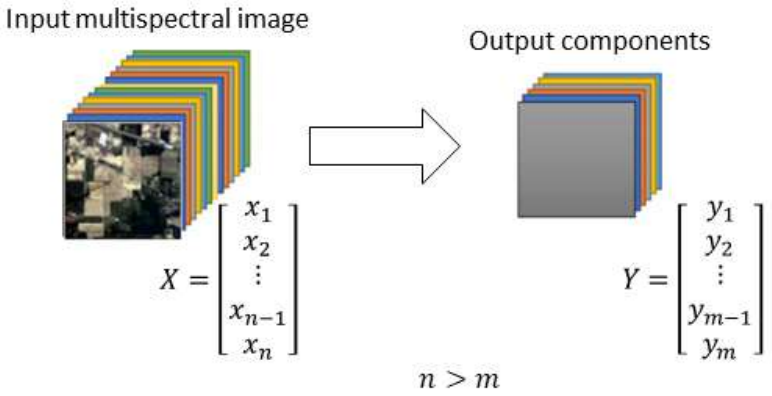}
    \caption{The visualization of the dimensionality reduction for a multispectral dataset. The number of input spectral bands $n$ is reduced to $m$ in the output dataset. Each colored layer in the input image and the output represents a spectral band and a component, respectively.}
    \label{fig_4}
\end{figure}

\paragraph{Clustering:}

Unsupervised machine learning identifies patterns and relationships within the data autonomously without the need for labeled data. This is particularly useful for tasks such as clustering, anomaly detection, and association \citep{yadav2019study,awange2020hybrid}. Clustering, a key technique in unsupervised learning, involves grouping a set of objects in such a way that objects in the same group (or cluster) are more similar to each other than to those in other groups \citep{jain1999data}. Popular clustering methods include hierarchical clustering, DBSCAN (Density-Based Spatial Clustering of Applications with Noise), and $k$-means clustering \citep{jain2010data}. $k$-means clustering is one of the most widely used due to its simplicity and efficiency. It partitions the data into $k$ clusters, where each data point belongs to the cluster with the nearest mean, serving as the cluster's centroid \citep{sakthivel2021conspectus}.

$k$-means clustering has numerous applications, particularly in image processing and geological mapping \citep{shirmard2022review}. $k$-means is often used for image processing (computer vision) tasks such as image segmentation, where the goal is to partition an image into meaningful regions for easier analysis \citep{sakthivel2021conspectus}. In the case of remote sensing data, the algorithm works by treating pixel values as data points and grouping similar pixels into clusters, thereby simplifying the image into distinct segments \citep{selim1984kmeans}. In geological mapping, $k$-means can be instrumental in categorizing different landforms or mineral compositions based on remote sensing data. Geologists can identify and map various geological features, such as rock types, by clustering spectral data from satellite images \citep{shirmard2022review}. This application enhances the accuracy of geological surveys and aids in resource exploration and environmental monitoring. The versatility and effectiveness of $k$-means make it a valuable tool in both image processing and remote sensing-based geological mapping, enabling the extraction of meaningful patterns and insights from complex datasets.

\paragraph{Elbow method:}

The elbow method is a widely used heuristic for determining the optimal number of clusters in $k$-means clustering \citep{celebi2013comparative}. This method involves running $k$-means clustering for a range of $k$ values and plotting the resulting sum of squared errors (SSE), also known as the within-cluster sum of squares (WCSS). The SSE measures the compactness of the clusters, with lower values indicating more tightly packed clusters \citep{patel2022approaches}. As $k$ increases, SSE naturally decreases because clusters have fewer data points, making them more compact. The key idea of the elbow method is to identify the point where SSE improvement dramatically slows down, forming an elbow in the plot \citep{onumanyi2022autoelbow}. This point suggests a suitable trade-off between the number of clusters and the compactness of the clusters, representing the optimal $k$. The elbow method is popular due to its simplicity and intuitive graphical representation, making it accessible even to those without advanced statistical knowledge \citep{onumanyi2022autoelbow}.

Despite its advantages, the elbow method has several limitations. One major drawback is its reliance on visual interpretation, which can be subjective and may lead to different conclusions depending on the interpreter's judgment \citep{shi2021quantitative}. In some cases, the elbow in the SSE plot might not be clearly defined, making it difficult to pinpoint the optimal $k$. Additionally, the elbow method assumes that the best clustering solution is where the rate of SSE improvement slows down, but this might not always align with the actual structure of the data \citep{shi2021quantitative}. In the case of datasets with complex or overlapping cluster structures, the elbow method may not provide a clear or accurate determination of $k$. Moreover, the method does not account for the possibility of multiple valid clustering solutions, each potentially useful for different applications \citep{shi2021quantitative}. These disadvantages suggest that while the elbow method is a valuable tool, it is often best used in conjunction with other techniques to ensure a more robust determination of the optimal number of clusters. However, given the high number of pixels in the satellite images used in this study and the time-consuming process of calculating other statistics like the silhouette score \citep{rousseeuw1987silhouettes}, we rely on the elbow method. We use the \textit{KElbowVisualizer} from the \textit{yellowbrick} Python library to determine the optimal number of clusters. This module employs the Kneedle algorithm \citep{satopaa2011finding}, which detects the elbow point by identifying the maximum curvature in the plot of WCSS. The algorithm normalizes the WCSS values, calculates the difference between these values and a linear approximation, and identifies the point with the maximum difference as the elbow. This approach provides a reliable and automated method for determining the optimal $k$.

\paragraph{Framework:}

We need to adopt a multidisciplinary approach that integrates specialized domain-specific knowledge to implement deep learning models for geological mapping. In this study, we leverage stacked autoencoders and fine-tune their weights by training the model on various datasets to enhance its effectiveness in accurately identifying geological features. We highlighted earlier that autoencoders can be more robust to outliers and noisy data than PCA, as they can learn to ignore or suppress noisy features during training \citep{vincent2008extracting}. Fig. \ref{fig_5} illustrates the first step: acquiring multispectral data from Landsat 8, ASTER, and Sentinel-2, followed by necessary radiometric and geometric corrections and data scaling, which constitute the pre-processing stage. Next, we implement various dimensionality reduction methods, including PCA, canonical autoencoders, and stacked autoencoders, to create a compressed dataset (latent features) with reduced dimensions. Each pixel of this compressed data is considered a collection of non-spatial spectral observations by spectral classifiers. Subsequently, we apply clustering to the compressed data from these dimensionality reduction methods. Using the elbow method to implement $k$-means and generate clustered maps representing geological features, we determine the optimal number of clusters. Finally, we generate the clustered maps and interpret the results from a geological perspective.

After implementing the clustering phase, we compare different dimensionality reduction methods and data types. We utilize metrics that include the Davies-Bouldin index and the variance ratio criterion (Calinski-Harabasz index), which can be calculated without labeled data \citep{patel2022approaches}. The Calinski-Harabasz index measures the ratio of between-cluster variation to within-cluster variance, with higher values indicating better-defined clusters. Therefore, the ideal number of clusters corresponds to maps with a high Calinski-Harabasz index, whereas lower values of the Davies-Bouldin index indicate better model performance.

\cite{mittal2022comprehensive} reported that neighboring pixels in images are likely to belong to the same cluster due to spatial correlation, which aids in the clustering process. Geological units typically have regional distributions in space \citep{suchet2003worldwide}, meaning that adjacent pixels with similar properties are likely to belong to the same geological unit. However, each group of pixels may be close to a neighboring cluster, which can create confusion. Applying a majority filter to the clustered map can improve its accuracy and consistency by removing outliers and inconsistencies, as well as the noise associated with the data. Since autoencoders map the data distribution to a normal distribution for input images without filtering (including noise), mapping without filtering will not be optimal and may result in blurred maps with fewer sharp boundaries. The majority filter can enhance the accuracy and reliability of geological maps, leading to a better understanding of an area's geology. In this study, we apply a majority filter with a kernel size of $7 \times 7$ to the clustered maps, assigning spurious pixels within a large single class to that class. The center pixel in the kernel is replaced with the class value that the majority of the pixels in the kernel possess \citep{kantakumar2015multitemporal}.

If we remove majority filtering from the framework, anomalous pixels may indicate either remote sensing errors or specific non-homogeneities in geological structures. Anomalous pixels resulting from remote sensing errors can be identified through cross-validation with ground truth data and other remote sensing datasets. We can analyze these anomalies without filtering to improve the accuracy and reliability of the remote sensing data. Methods such as statistical outlier detection or comparison with high-resolution imagery can differentiate between true geological anomalies and errors. The anomalous pixels may also indicate genuine non-homogeneities within geological structures, which can provide significant geological insights. In this case, detailed field studies and additional sampling might be necessary to understand the underlying causes of these non-homogeneities.

\begin{figure}
    \centering
    \includegraphics[width=\textwidth]{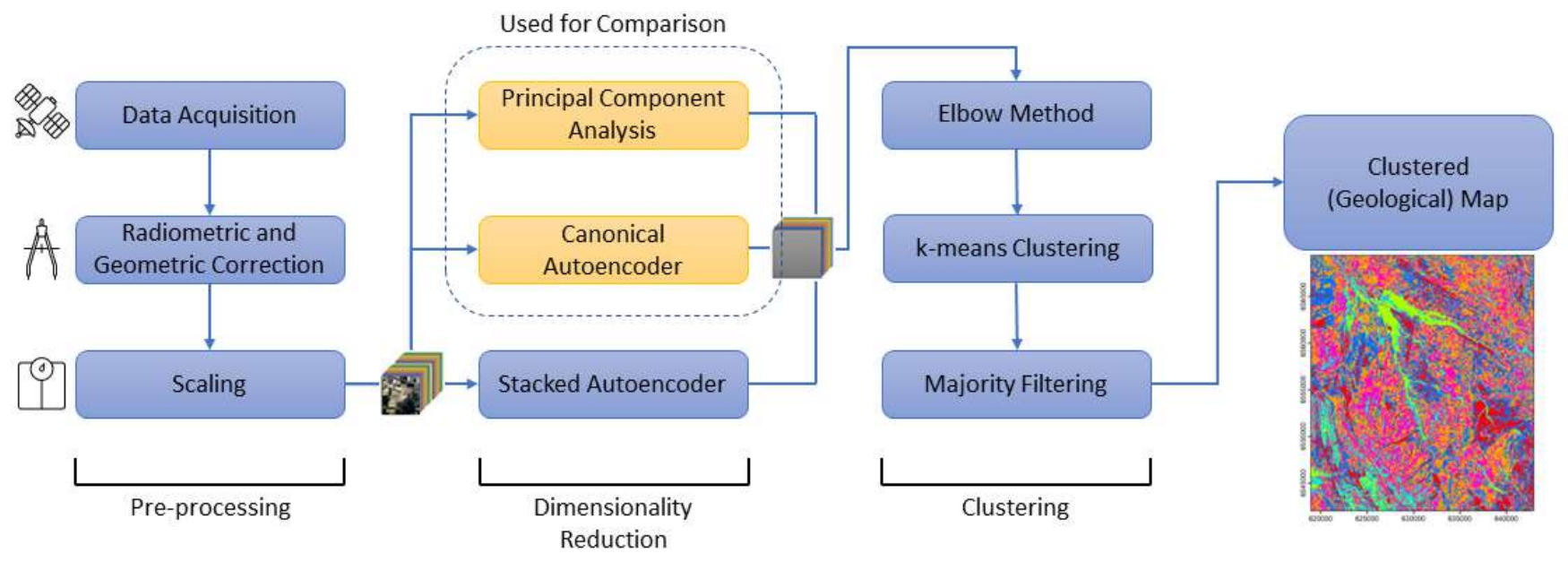}
    \caption{Machine learning framework for creating geological maps using the integration of the dimensionality reduction (PCA, canonical autoencoder, and stacked autoencoder) methods and the $k$-means clustering.}
    \label{fig_5}
\end{figure}

\paragraph{Implementation:}

We implement the framework shown in Fig. \ref{fig_5} using the Python programming language and various machine learning packages, including Keras\footnote{Keras: \url{https://keras.io/api/}}, which facilitate the efficient implementation and execution of deep learning models and streamline the experimentation process. Implementing deep learning models typically requires setting several hyperparameters, and we experiment with different values to achieve the most accurate results.

In the case of PCA, we select the principal components that preserve $90\%$ of the total variance of the input spectral bands, resulting in a different number of components for each data type. The architecture of the canonical autoencoders includes one hidden layer and ten iterations (epochs). We use the rectified linear unit (ReLU) \citep{nair2010rectified} as the activation function for the hidden layers to account for non-linearity and the sigmoid activation function for the output layer. We use the Adam optimizer \citep{kingma2014adam} and mean squared error (MSE) loss function for model training.

The architecture of the stacked autoencoders used in this study comprises two autoencoders with one hidden layer for each. The optimizer, loss, and activation functions for the hidden and output layers are the same as those used for the canonical autoencoders. We provide the code implementation in a GitHub repository\footnote{\url{https://github.com/sydney-machine-learning/autoencoders_remotesensing}} and execute the experiments using an 11th Gen Intel(R) Core(TM) i7-11700 @ 2.50GHz CPU.

\section{Results}

We apply the proposed framework to Landsat 8, ASTER, and Sentinel-2 multispectral data from the Mutawintji region in western New South Wales, Australia. Fig. \ref{fig_6} displays the elbow graphs used to determine the optimal number of clusters for different data types and dimensionality reduction method pairs. In these plots, the optimal $k$ for $k$-means clustering is indicated by a green dashed line, and the black line represents the sum of squared distances to the cluster centers. The point of maximum curvature in the elbow plots marks the optimal $k$ \citep{yuan2019research}. Table \ref{table_1} summarizes the optimal numbers of clusters, suggesting that six or seven major geological units in the study area exhibit specific spectral characteristics depending on the data type and dimensionality reduction method.

In addition to PCA, we train canonical and stacked autoencoders and calculate the reconstruction loss, demonstrating minimal loss of information/features after dimensionality reduction. The features learned from the canonical and stacked autoencoders are then used to cluster the remote sensing data. We observe that the reconstruction loss stabilizes after a few epochs. Several metrics are available to evaluate the efficiency of clustering without labeled data, including the Silhouette coefficient, Calinski-Harabasz index, and Davies-Bouldin score \citep{patel2022approaches}. Due to the large number of pixels, calculating the Silhouette coefficient is time-consuming on a standard computer and is impractical. Therefore, we evaluate model performance using the Calinski-Harabasz and Davies-Bouldin scores \citep{renjith2020pragmatic,gao2021generalized}, as shown in Table \ref{table_2}. A higher Calinski-Harabasz score (first column for each method in Table \ref{table_2}) indicates better-defined clusters, while a lower Davies-Bouldin score suggests more efficient clustering results. It is noteworthy that we conducted additional analyses to quantify the impact of varying numbers of clusters ($k$) on our results. Specifically, we calculated the Calinski-Harabasz and Davies-Bouldin scores for $k$ values of 5, 6, 7, and 8 across all pairs of remote sensing data and dimensionality reduction methods. We found that the $k$ values determined by the Kneedle algorithm yielded the best results according to both the Calinski-Harabasz and Davies-Bouldin scores.

Fig. \ref{fig_7} presents the clustered maps obtained by applying the framework to various pairs of remote sensing data and dimensionality reduction methods. We observe that geological maps generated using PCA followed by $k$-means clustering differ from those produced by canonical and stacked autoencoders. The maps generated using stacked autoencoders on Landsat 8 and ASTER, as well as canonical autoencoders on Sentinel-2, consist of seven clusters, while the others have six clusters. In addition to the previously mentioned metrics, we use 30 rock samples and associated information provided by the Geological Survey of NSW\footnote{\url{https://minview.geoscience.nsw.gov.au}} as ground truth data. We calculate the overall accuracy for each pair of dimensionality reduction methods and remote sensing data by dividing the correctly clustered samples by the total number of samples (Table \ref{table_3}). Given that the optimal number of clusters varies among data types and dimensionality reduction methods, we simplify the rock types and categorize them into six (Fig. \ref{fig_8}(a)) or seven classes according to Table \ref{table_1}.

\begin{figure}
    \centering
    \includegraphics[width=\textwidth]{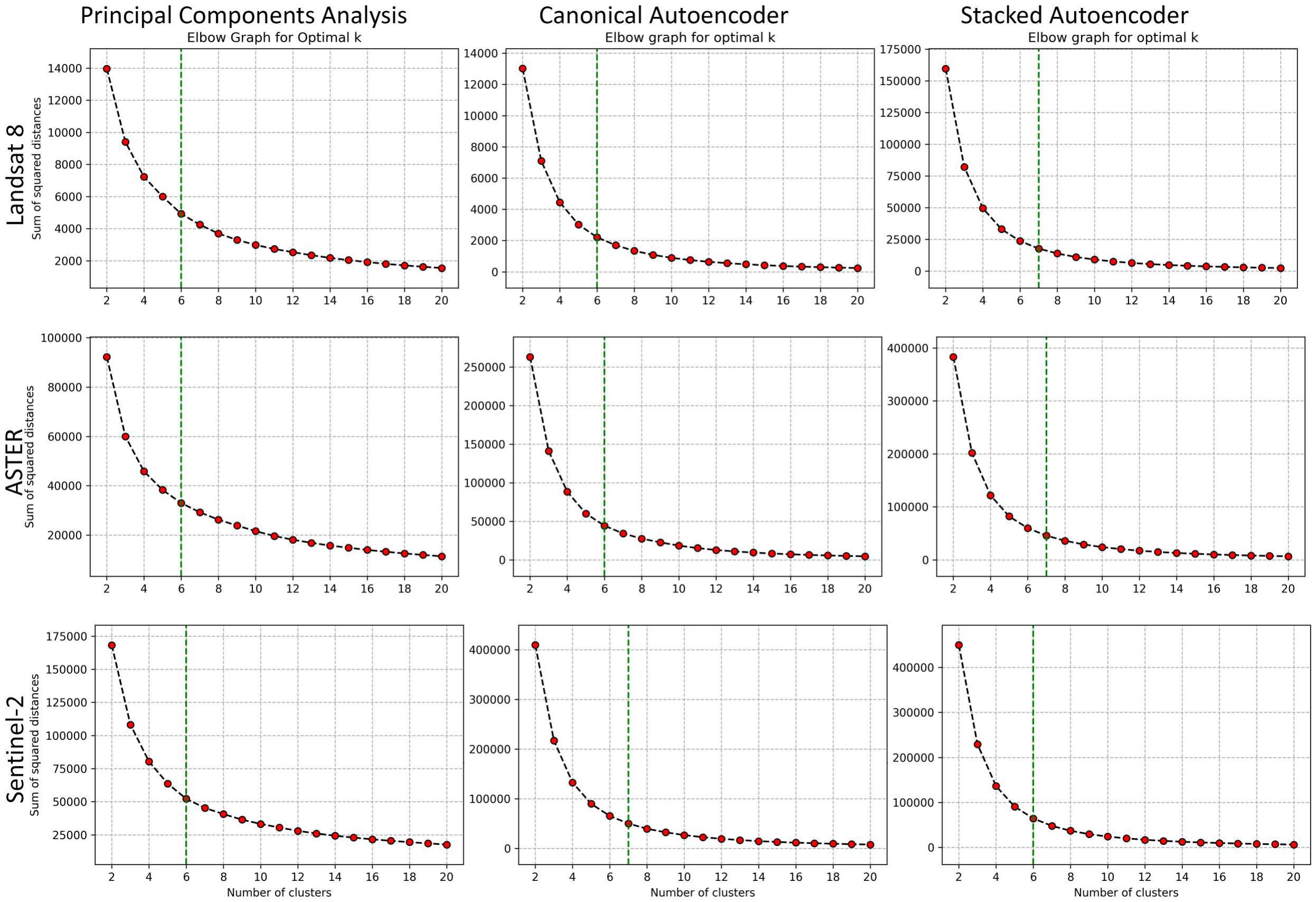}
    \caption{Elbow plots used to determine the optimal number of clusters ($k$) for each combination of remote sensing data (Landsat 8, ASTER, Sentinel-2) and dimensionality reduction method (PCA, canonical autoencoder, stacked autoencoder). The x-axis shows the number of clusters, and the y-axis represents the sum of squared distances to the cluster centers. The point of maximum curvature in each plot, identified by the green dashed line, indicates the optimal value of $k$.}
    \label{fig_6}
\end{figure}

\begin{table*}[htbp]
    \centering
    \begin{tabular}{lccc}
    \toprule
    \multicolumn{1}{c}{Data Type/Method} & PCA & Canonical Autoencoder & Stacked Autoencoder \\ \hline
    Landsat 8                              & 6   & 6                     & 7    \\ 
    ASTER                                  & 6   & 6                     & 7    \\ 
    Sentinel 2                             & 6   & 7                     & 6    \\ 
    \bottomrule
    \end{tabular}
    \caption{Optimal number of clusters ($k$) identified using the elbow method for different combinations of remote sensing data (Landsat 8, ASTER, and Sentinel-2) and dimensionality reduction techniques (PCA, canonical autoencoder, and stacked autoencoder). The selected $k$ values reflect the point of maximum curvature in the elbow plots, indicating distinct spectral patterns corresponding to six or seven major geological units.}
    \label{table_1}
    
\end{table*}

\begin{table*}[htbp]
    \centering
    \resizebox{\textwidth}{!}{%
    \begin{tabular}{lcccccc}
    \toprule
    \multicolumn{1}{c}{\multirow{2}{*}{Data Type/Method}} & \multicolumn{2}{c}{PCA}                                & \multicolumn{2}{c}{Canonical Autoencoder}              & \multicolumn{2}{c}{Stacked Autoencoder}                \\ \cline{2-7} 
    \multicolumn{1}{c}{}                                  & \multicolumn{1}{c}{Calinski-Harabasz} & Davies-Bouldin & \multicolumn{1}{c}{Calinski-Harabasz} & Davies-Bouldin & \multicolumn{1}{c}{Calinski-Harabasz} & Davies-Bouldin \\ \hline
    Landsat 8                                               & \multicolumn{1}{c}{725156}            & 0.848          & \multicolumn{1}{c}{2049661}           & 0.538          & \multicolumn{1}{c}{2928417}           & 0.520          \\ 
    ASTER                                                   & \multicolumn{1}{c}{3203478}           & 0.882          & \multicolumn{1}{c}{8940866}           & 0.534          & \multicolumn{1}{c}{10204767}          & 0.533          \\ 
    Sentinel 2                                              & \multicolumn{1}{c}{7975444}           & 0.799          & \multicolumn{1}{c}{20868268}          & 0.530          & \multicolumn{1}{c}{22433325}          & 0.525          \\ 
    \bottomrule
    \end{tabular}%
    }
    \caption{Calinski-Harabasz and Davies-Bouldin scores calculated for different pairs of data types and dimensionality reduction methods. Evaluation of clustering performance for different combinations of remote sensing data (Landsat 8, ASTER, and Sentinel-2) and dimensionality reduction techniques (PCA, canonical autoencoder, and stacked autoencoder) using Calinski-Harabasz and Davies-Bouldin scores. Higher Calinski-Harabasz values indicate better cluster separation, while lower Davies-Bouldin scores suggest more compact and well-defined clusters. These metrics confirm the effectiveness of autoencoder-based methods, particularly stacked autoencoders, in enhancing clustering quality.}
    \label{table_2}
\end{table*}

\begin{figure}
    \centering
    \includegraphics[width=0.97\linewidth]{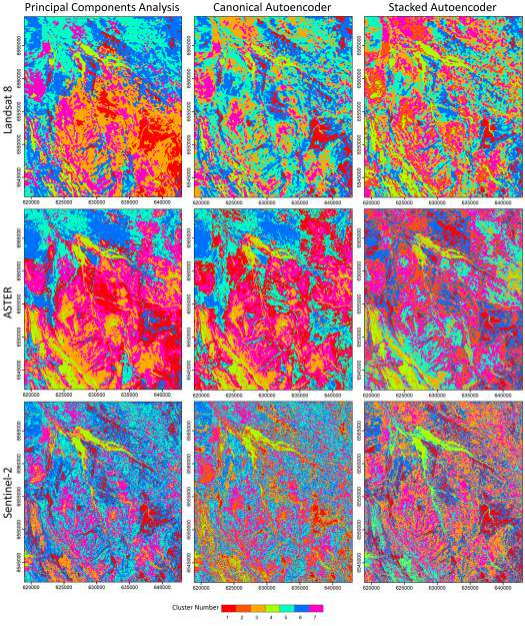}
    \caption{Clustered maps of the study area obtained by different pairs of data types and dimensionality reduction methods. Each distinct colored region in the maps represents a unique geological unit on the ground. We find that PCA (first column) generally results in less detailed geological unit differentiation than autoencoder-based methods, reflecting its limitations in capturing complex nonlinear relationships in the data. Sentinel-2 data (last row) shows the best cluster compactness and separation performance, indicating a more distinct geological unit classification.}
    \label{fig_7}
\end{figure}

\begin{table}[htbp]
    \centering
    \begin{tabular}{lccc}
    \hline
    \multicolumn{1}{c}{Data Type/Method} & PCA   & Canonical Autoencoder & Stacked Autoencoder \\
    \toprule
    Landsat 8               & 0.767 & 0.800         & \textcolor{blue}{0.866}  \\ 
    ASTER                   & 0.833 & 0.833         & \textcolor{blue}{0.900}  \\
    Sentinel 2              & 0.800 & 0.833         & \textcolor{blue}{0.900}  \\ 
    \bottomrule
    \end{tabular}
    \caption{Overall accuracy of different pairs of data types and dimensionality reduction methods based on ground truth data (rock samples). Overall clustering accuracy for different combinations of remote sensing data (Landsat 8, ASTER, and Sentinel-2) and dimensionality reduction methods (PCA, canonical autoencoder, and stacked autoencoder), evaluated using 30 rock samples as ground truth. Accuracy is computed as the proportion of correctly clustered samples. The highest accuracy for each data type is highlighted in blue, with stacked autoencoders achieving the best performance across all datasets.}
    \label{table_3}
\end{table}

\begin{figure}
    \centering
    \includegraphics[width=\textwidth]{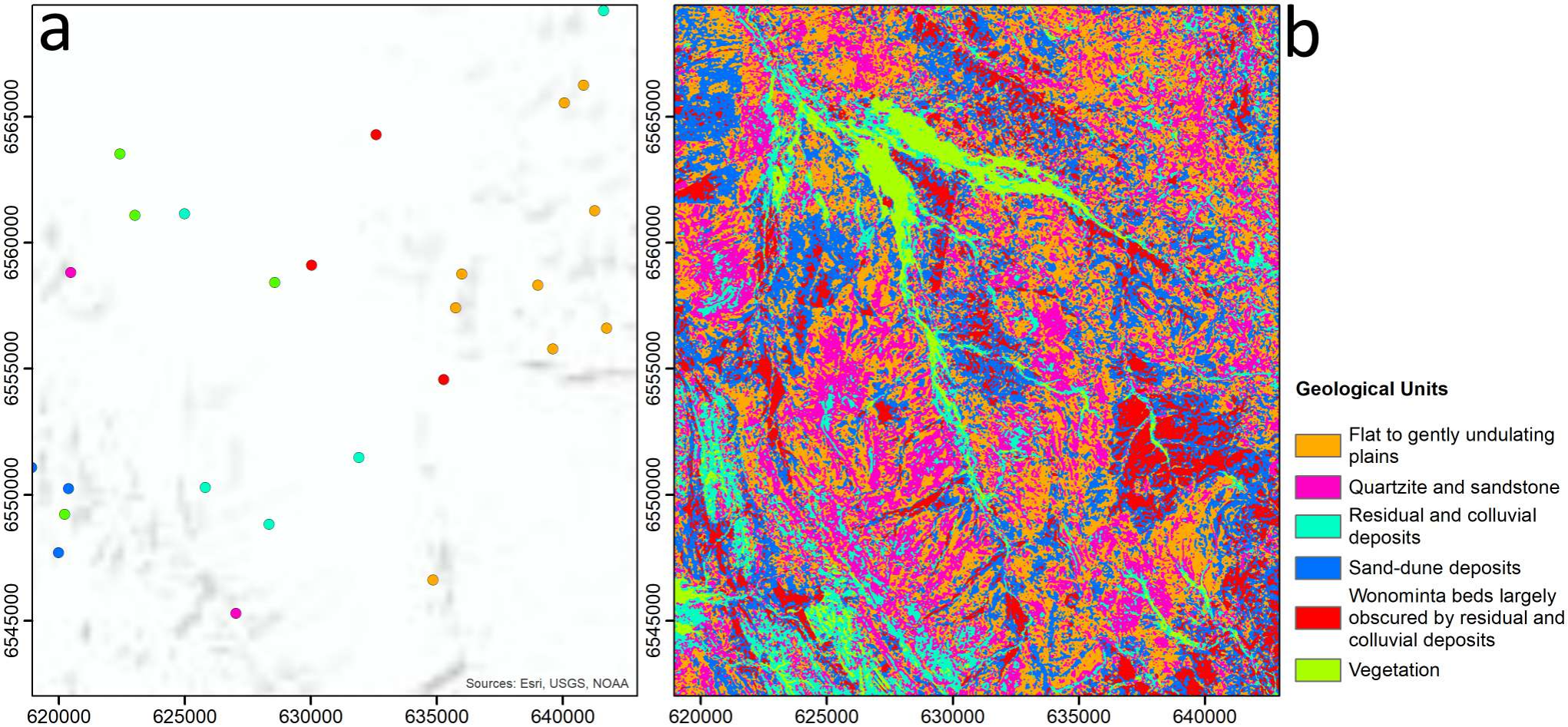}
    \caption{a) Rock samples are available from across the study area through MinView; b) The clustered map is interpreted by implementing the proposed framework and applying stacked autoencoders to the Sentinel-2 dataset.}
    \label{fig_8}
\end{figure}

We compare the clustered maps obtained from our framework (Fig. \ref{fig_7}) to evaluate the efficiency of different dimensionality reduction methods in terms of noise removal, data compression, and geological unit discrimination. We use three metrics to compare these methods. According to Table \ref{table_2}, stacked autoencoders achieve larger Calinski-Harabasz scores and lower Davies-Bouldin scores for each data type, indicating that the clusters generated using the latent vectors of stacked autoencoders are more separable compared to principal components and the latent vectors obtained by canonical autoencoders. We also observe that as spatial resolution increases, both scores improve, with Sentinel-2 yielding the highest scores due to its superior spatial resolution. Stacked autoencoders applied to Sentinel-2 data provide the best results among all data types and dimensionality reduction methods.

Each uniquely colored region in the maps shown in Fig. \ref{fig_7} represents a distinct geological unit clustered using a pair of dimensionality reduction methods and data types. By comparing these clustered maps with the geological map shown in Fig. \ref{fig_1}(b), it is evident that Sentinel-2 data produce a more detailed clustered map than other data types and conventional approaches for mapping geological units. Sentinel-2 data uniquely identifies vegetation and assigns a separate cluster to relevant pixels. Additionally, the comparison of dimensionality reduction methods reveals that only stacked autoencoders accurately differentiate between various sedimentary units in the south and southeast of the study area. Utilizing high-resolution remote sensing data such as Sentinel-2, combined with non-linear dimensionality reduction methods like stacked autoencoders, enhances the signal-to-noise ratio of input features to clustering algorithms, resulting in more detailed and accurate geological maps than conventional approaches like field surveys.

The metrics reported in Table \ref{table_2} do not consider ground truth or labeled data, relying solely on input features and assigned labels for each observation or pixel. The overall accuracy presented in Table \ref{table_3} incorporates the adaptation of collected rock samples from the study area shown in Fig. \ref{fig_8}(a) with the clustered maps, providing a more reliable metric for determining the best approach for generating a clustered or geological map. According to Table \ref{table_3}, employing stacked autoencoders on Sentinel-2 data yields the highest accuracy, indicating that the majority of rock samples have been assigned to the correct cluster. Consequently, the map generated using Sentinel-2 and stacked autoencoders is interpreted as the geological map shown in Fig. \ref{fig_7}(b). This map identifies five different geological units plus vegetation, effectively distinguishing between different sedimentary units in the study area, such as sandstones, residual and colluvial deposits, and undulating plains.

\section{Discussion}

The results presented in this study provide insights into the efficacy of different dimensionality reduction methods combined with various remote sensing data types for geological mapping. The evaluation metrics offer quantitative validation of the proposed machine learning framework for geological mapping via dimensionality reduction and clustering. According to Table \ref{table_2}, the Calinski-Harabasz score indicates the separability of clusters generated by PCA and autoencoders, with stacked autoencoders consistently yielding higher scores across all data types compared to PCA. This suggests that the non-linear nature of autoencoders facilitates better discrimination between geological units, particularly with Sentinel-2 data, which benefits from higher spatial resolution. The Davies-Bouldin score, which assesses cluster compactness and separation, reveals similar findings, with the stacked autoencoder on Sentinel-2 data showing the best performance. This highlights the importance of considering spectral characteristics and data suitability for specific models when evaluating clustering performance in remote sensing applications.

The visual examination of clustered maps (Fig. \ref{fig_7}) further demonstrates the advantages of using Sentinel-2 data with stacked autoencoders. The resultant map provides detailed and accurate delineation for geological mapping. Moreover, stacked autoencoders exhibit superior discrimination capabilities, especially for different sedimentary units in specific regions of the study area. Including ground truth data enhances the assessment, with overall accuracy serving as a robust metric. The results underscore the effectiveness of stacked autoencoders on Sentinel-2 data, yielding the highest accuracy and ensuring the proper assignment of rock samples to clusters (Table \ref{table_3}). Consequently, leveraging Sentinel-2 data and stacked autoencoders produces a detailed geological map of the study area, successfully discriminating between various geological units.

Although autoencoders are more computationally intensive to train and may require extensive hyperparameter tuning, the findings demonstrate that the proposed framework yields compelling results for geological mapping applications using multispectral data without labeled data. However, the effectiveness of the proposed approach should be validated across diverse geological terrains to assess its broader applicability. Additionally, preprocessing steps applied to remote sensing data, such as atmospheric correction, radiometric calibration, or geometric registration, can significantly impact the quality and accuracy of the derived geological maps and should be carefully considered.

The choice of the number of clusters ($k$) introduces uncertainties in the creation of clustered maps, which are interpreted as geological maps. An incorrect $k$ value can lead to either over-segmentation or under-segmentation of geological units, affecting the accuracy and reliability of the map. Over-segmentation might result in an excessive number of clusters, misrepresenting homogenous geological units as multiple distinct entities. Conversely, under-segmentation may group distinct geological units together, masking important geological variations.

Additionally, the low number of rock samples can further introduce uncertainties. Limited samples may not adequately capture the variability of geological units, leading to less accurate cluster definitions. This insufficiency can result in clustered maps that do not accurately reflect the true geological diversity of the area. These uncertainties have significant implications for geological map interpretations. Variations in $k$ values and sample sizes can lead to different geological unit definitions, affecting exploration and decision-making processes. For example, misidentified units might lead to incorrect assumptions about mineral resources, structural stability, or environmental conditions. Therefore, careful consideration of $k$ values and efforts to obtain a sufficient number of representative rock samples are crucial. Employing robust validation techniques and cross-referencing with existing geological data can help mitigate these uncertainties, leading to more accurate and reliable geological maps.

In our study, the stationary assumptions overlook the dynamic nature of geological features (e.g., due to landslides, earthquakes, seasonal erosion, and human activities), necessitating methods to incorporate temporal dynamics for more accurate mapping. Reliance on subjective ground truth data introduces uncertainty, which can be mitigated by integrating multiple sources and using consensus-based approaches. Although advanced techniques improve discrimination, efforts are needed to enhance the interpretability of the resulting maps. Other dimensionality reduction methods, such as manifold learning \citep{izenman2012introduction,pless2009survey}, can be used to extend our framework. This study illustrated the potential of unsupervised feature learning methods in feature extraction, motivating large-scale applications using hyperspectral datasets. Moreover, incorporating ancillary data sources, such as geological maps, digital elevation models, and hydrological data, could enrich the analysis and improve the discrimination of geological units. The framework can be extended with novel clustering methods, such as spectral and hierarchical clustering \citep{saxena2017review} and Gaussian mixture models, which have shown promising results in remote sensing and image-based data processing \citep{deo2024reefcoreseg,barve2023reef}.

\section{Summary}
We presented a framework that combined stacked autoencoders with $k$-means clustering to generate geological maps. By applying various pairs of remote sensing data types and dimensionality reduction methods, we created input features for the $k$-means algorithm, resulting in automated geological maps for the Mutawintji region in NSW, Australia. Our investigation revealed that combining stacked autoencoders with Sentinel-2 data yields the highest spatial resolution and accuracy compared to other combinations. The stacked autoencoders were demonstrated to be highly effective for dimensionality reduction and feature learning, enabling better extraction of complex and hierarchical representations of the input data when compared to canonical autoencoders and PCA. Our findings reveal that the integration of stacked autoencoders with Sentinel-2 data yields the highest spatial resolution and accuracy. The stacked autoencoders proved to be highly effective for dimensionality reduction and feature learning, enabling the extraction of more complex and hierarchical representations of the input data compared to canonical autoencoders and PCA. The flexibility of our framework allows for further enhancements with the incorporation of novel dimensionality reduction and clustering methods. This adaptability ensures that our approach can evolve alongside advancements in remote sensing and machine learning techniques, making it a robust tool for geological mapping in various regions and contexts.

We proposed a flexible framework that integrates stacked autoencoders with \$k\$-means clustering to generate automated geological maps. By combining different remote sensing data types, Landsat 8, ASTER, and Sentinel-2, with dimensionality reduction techniques, we extracted effective feature representations for clustering. Among the combinations tested, stacked autoencoders applied to Sentinel-2 data produced the most accurate and spatially detailed geological maps for the Mutawintji region in New South Wales, Australia. The results demonstrate that stacked autoencoders outperform canonical autoencoders and PCA in capturing complex and hierarchical data structures, leading to improved clustering performance. Furthermore, the modularity of our framework allows for the integration of emerging dimensionality reduction and clustering techniques, ensuring its continued relevance and applicability across diverse geological mapping scenarios.

 \chapter{Image Art Restoration and Generative Models}\label{chap:art_restore}
\section{Introduction}
Preserving cultural heritage is of paramount importance. While history has preserved countless masterpieces, the ravages of time have left many artworks faded, damaged, or on the brink of disappearance. Traditional restoration methods are often time-consuming and require extensive expertise. By resurrecting damaged or obscured artworks, we breathe life back into these forgotten stories, reviving the narratives that have shaped our collective consciousness. The domain of image art restoration (IR) holds significant importance within the low-level vision discipline, aiming to enhance the perceptual quality of images that have suffered a wide array of degradation. This intricate task operates as a versatile and interpretable solution to a range of inverse problems, utilizing readily available denoising techniques as implicit image priors \cite{zhu2023denoising}. Within low-level vision research, IR has persistently remained a focal point, contributing substantially to enhancing image aesthetics \cite{li2023diffusion}. 

In the context of deep learning advancements, a plethora of IR methodologies have harnessed the power of datasets tailored for diverse IR challenges, such as super-resolution (DIV2K, Set5, Set14), rain removal (Rain800, Rain200, Raindrop, DID-MDN), and motion deblurring (REDS: REalistic and Dynamic Scenes datase, Gopro) \cite{li2023diffusion}. Notably, the emergence of diffusion models (DM) has ushered in a new paradigm within generative models, catalyzing remarkable breakthroughs across various visual generation tasks. The diffusion model excels through a sequential application of denoising networks to replicate the image synthesis process \cite{xia2023diffir}. Capitalizing on the exceptional generative prowess of diffusion models, we employ them as a benchmark for image restoration.

Traditional supervised learning approaches hinge upon extensive collections of distorted/clean image pairs, while zero-shot methods predominantly rely on known degradation modes. However, these methodologies encounter limitations in real-world scenarios characterized by diverse and unknown distortions. To address this concern, recent research work has extended diffusion models to accommodate blind/real-world image restoration scenarios by integrating real-world distortion simulations and kernel-based techniques. This expansion seeks to bridge the gap between diffusion models and the complexity of real-world image restoration challenges, offering a potential avenue for more effective applications in practical settings.

\paragraph{About challenge:}
This work is part of the \emph{Competitions @ ICETCI 2023} \href{https://ietcint.com/user/competitions}{link}.

Motivation: By resurrecting damaged or obscured artworks, we breathe life back into these forgotten stories, reviving the narratives that have shaped our collective consciousness. The participants must develop an innovative model that can automate the restoration process and ensure the longevity of art pieces for future generations. 

Objective: The challenge aims to design and implement an advanced computer vision model for restoring deteriorated art images. Participants are encouraged to explore various techniques, architectures, and training methodologies to
develop a robust and efficient solution.

\section{Related Work}

\paragraph{Traditional image restoration methods:}  Diffusion-based image restoration techniques rely on partial differential and variational methods grounded in geometric image models. These methods employ edge information in the damaged area to guide diffusion direction, propagating known information to the target region. While effective for minor image damage, they may yield fuzzy results when handling extensive damage or complex textures \cite{qiang2019survey}.

\paragraph{Deep learning based methods:}
Convolutional neural networks (CNNs) possess remarkable capabilities for learning and representing image features, enabling effective prediction of missing image content. The image restoration process primarily relies on supervised learning methods \cite{zhang2017learning}. 

In contrast to CNNs, which face challenges in supervised image restoration learning, autoencoders (AE) are artificial neural networks proficient at unsupervised learning, effectively learning and expressing input data \cite{mao2016image}. AEs try to regenerate the images from the latent vector and fail to remove the unordered noise/distortion. 

The GAN-based image restoration method differs from the convolution autoencoder-based method \cite{goodfellow2014generative}. The GAN-based image restoration method generates the image to be repaired directly through the generator. The input can be a random noise vector, and the former is applied to the damaged image to generate the repair area. While GANs excel in generating high-quality images, training them can be challenging, primarily due to the complexity of the loss function employed in the training process.

Another branch of deep learning and generative models, normalizing flows (NF) \cite{rezende2015variational}, is also used for image restoration. NFs are based on the invertible CNN layers \cite{nagar2021cinc, visapp23}, but NFs are slow and cost more computation for high-quality input images than CNNs, GANs, and VAEs. NFs work better for the deburring due to their inevitability and tractable nature \cite{wei2022deep}.

Image restoration (IR) represents an essential and demanding endeavor within the realm of low-level vision. Its objective is to enhance the perceptual quality of images afflicted by diverse degradation types. Notably, the diffusion model has made remarkable strides in the visual generation of Artificial Intelligence Generated Content (AIGC), prompting a natural inquiry: "Can the diffusion model enhance image restoration?" \cite{li2023diffusion}. Motivated by this question, we used the Diffusion Model (DM) based super-resolution to solve the problem of art restoration in the images.

\section{Method Description}
In artistic representation, the integrity of artworks can be compromised by many factors, such as motion-induced disruptions, various forms of noise, filters, and even water intrusion. This degradation or distortion also extends to encapsulating art within images, further entailing inherent discrepancies within the representation. Consequently, restoring genuine artistic essence and preserving authentic artifacts presents a formidable challenge due to the reliance on these images as the sole source of information regarding the artwork.

\begin{figure*}[!ht]
    \centering
    \includegraphics[width=0.99\linewidth]{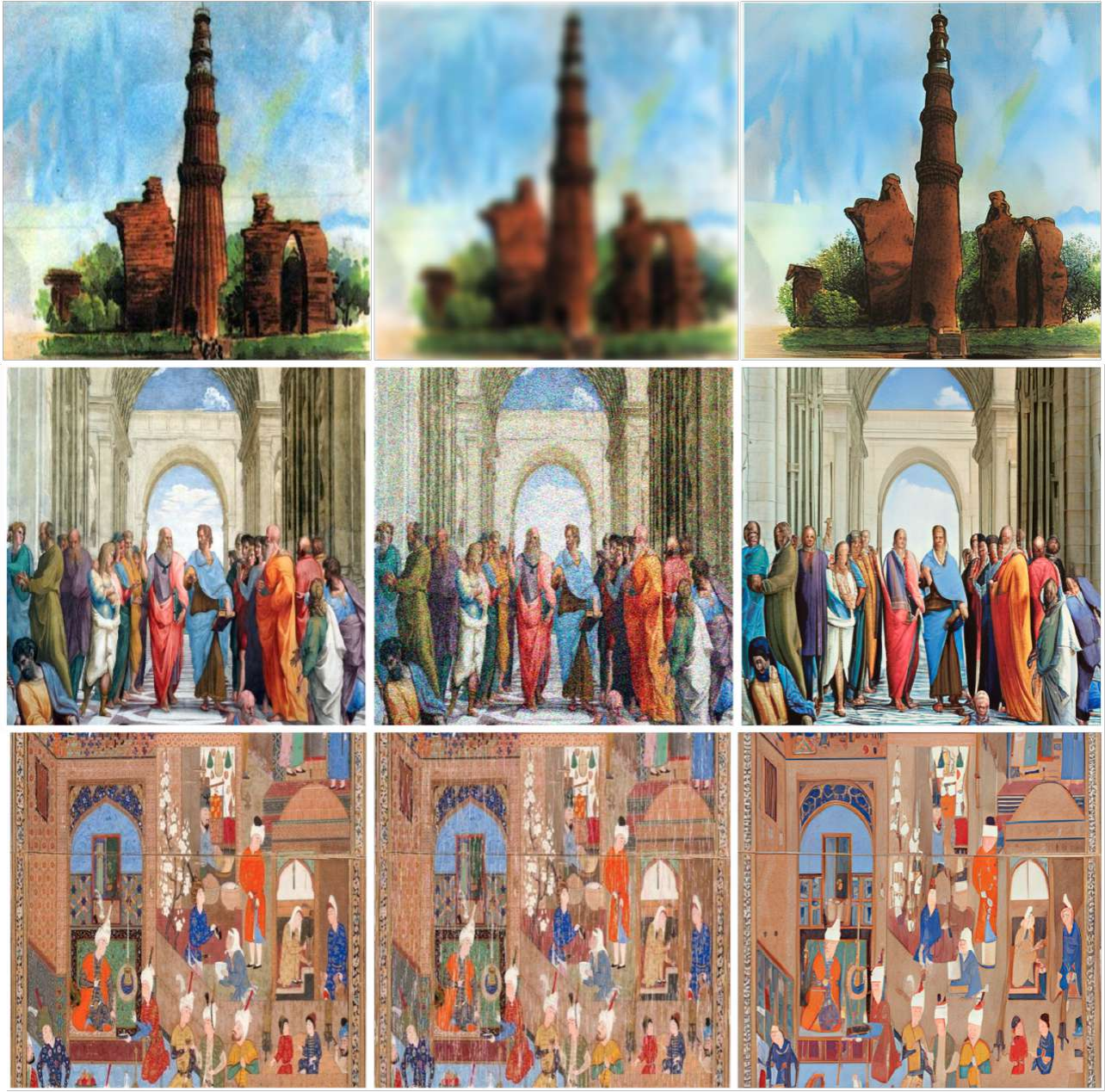}
    \caption{Method: StableSR. Left: GT, Center: distorted by \textbf{noise/ damaged}, Right: Restored. Image restoration results using the StableSR framework under various distortion conditions, including noise and structural damage. Each triplet of images shows (Left) the Ground Truth, (Center) the degraded input, and (Right) the output restored by StableSR. The results demonstrate the framework's effectiveness in recovering fine textures, edges, and semantic content across diverse image types and degradation patterns.}
    \label{fig:sample_fig_label}
\end{figure*}

\begin{figure*}[!ht]
    \centering
    \includegraphics[width=0.99\linewidth]{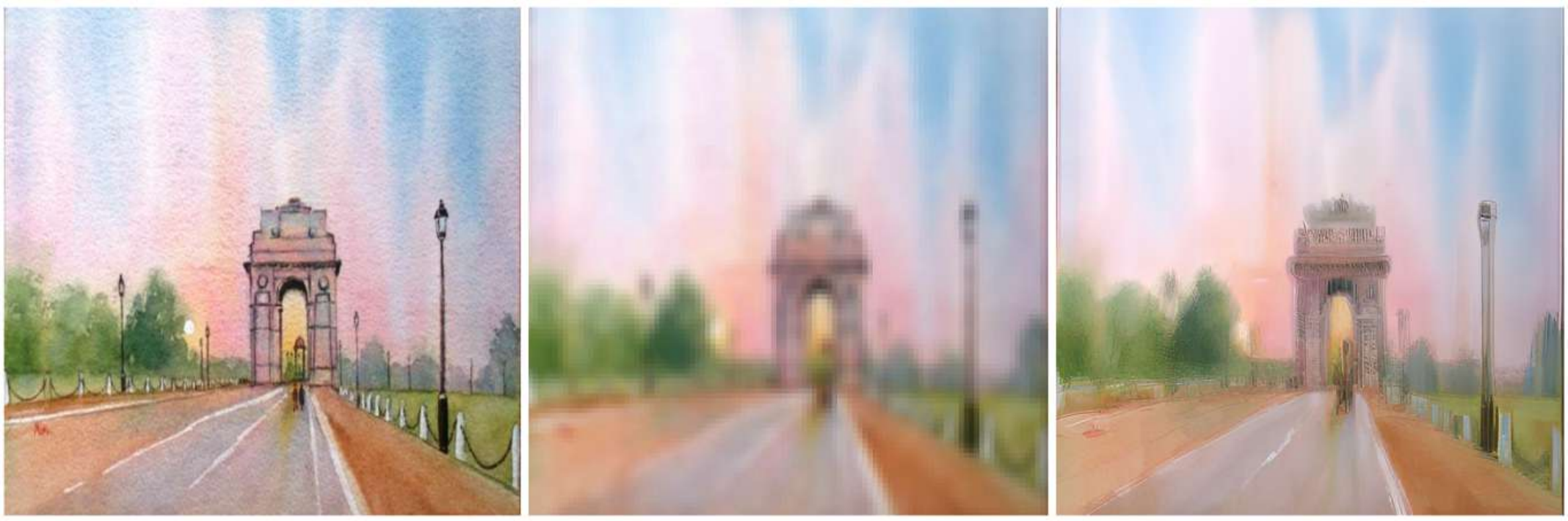}
    \caption{Method: StableSR. Left: GT, Center: distorted by \textbf{pixelating}, Right: Restored. Restoration results using the StableSR framework on images degraded by pixelation. The figure shows (Left) the Ground Truth image, (Center) the pixelated input, and (Right) the image restored by StableSR. The method effectively reconstructs structural elements and smooth textures, mitigating pixelation effects.}
    \label{fig:sample_fig_label4}
\end{figure*}


The endeavor of rejuvenating impaired artifacts is intricate, marked by its irrevocable nature. This compels the exploration of computer vision models as a resource, harnessing their capacity to leverage embedded attributes within the images. Among the array of techniques, diffusion models emerge as the paramount state-of-the-art (SOTA) method for image generation and restoration. In particular, the application of image super-resolution models proves to be salient in the restoration process. In pursuit of enhanced resolution, these models inherently address a broad spectrum of prevailing distortions and degradation in the captured images.

It is imperative to acknowledge that the acquisition of images itself is predisposed to quality deterioration, often stemming from the intricacies of the capturing apparatus. This encompasses introducing supplementary noise and filters, exacerbating the challenges intrinsic to preserving the fidelity of art images. 

So we propose to use the super-resolution SOTA (StableSR) \cite{wang2023exploiting} to restore the art. Further, we fine-tune the StableSR model for art restoration and reconstruction. Further, to verify and compare the StableSR model to other existing super-resolution SOTA, we also test the sample images using the ResShift \cite{yue2023resshift} super-resolution model (see Figure \ref{fig:ResShift1}). 

Image restoration aims to recover high-quality images from degraded versions, where everyday tasks include super-resolution (SR), denoising, and deblurring. The StableSR framework integrates the power of pre-trained weights from the Stable Diffusion model \cite{li2023diffusion} with fine-tuning strategies tailored for specific degradation types. This method improves image resolution and ensures effective restoration in various domains, such as removing noise, correcting blur, and enhancing image details.

\section{StableSR Framework}

The StableSR framework consists of an Encoder-Decoder Architecture. This component utilizes a pre-trained Stable Diffusion model as the backbone. The encoder extracts features from degraded images, and the decoder generates high-resolution outputs based on these features.

The approach begins by leveraging pre-trained weights from the Stable Diffusion model as initialization for the restoration network. This provides a strong feature representation, significantly improving the model's ability to handle various degradation types. The fine-tuning process adjusts the model to perform specific image restoration tasks, such as super-resolution, deblurring, and noise reduction, by employing a set of multi-task learning objectives. 

\paragraph{Method Overview}

As illustrated in Figure \ref{fig:stablesr_architecture}, the method involves a series of steps to restore images through the StableSR framework. These steps include:

\begin{enumerate}
    \item Degraded Image Input: The image undergoes degradation, including noise, blur, or low resolution.
    \item Feature Extraction: The encoder, based on the pre-trained Stable Diffusion model, extracts feature maps from the degraded image.
    \item Restoration: The model progressively restores the high-resolution image using a combination of encoder and decoder fine-tuned decoder layers. 
    \item Output: The restored image is obtained as the final output, which is visually enhanced compared to the input degraded image.
\end{enumerate}

\paragraph{Loss Function}

The loss function used to train the StableSR model includes several components:
\begin{itemize}
    \item Pixel-wise loss: Measures the difference between the restored and ground truth image at the pixel level, typically using Mean Squared Error (MSE).
    \item Perceptual loss: Extracted from the pre-trained VGG network, this loss ensures that the restored image preserves high-level perceptual features.
    \item Adversarial loss: A discriminator is introduced during training to encourage the model to generate more realistic images, helping to avoid artifacts in the restoration process.
\end{itemize}

\paragraph{StableSR Architecture}
The architecture of the StableSR model is based on the encoder-decoder structure, enabling efficient and smooth image restoration while maintaining computational efficiency. Figure \ref{fig:stablesr_architecture} visually represents the architecture. The encoder part is responsible for extracting essential image features from the degraded input, while the decoder generates the restored image.

\begin{figure}[h!]
    \centering
    \includegraphics[width=0.99\textwidth]{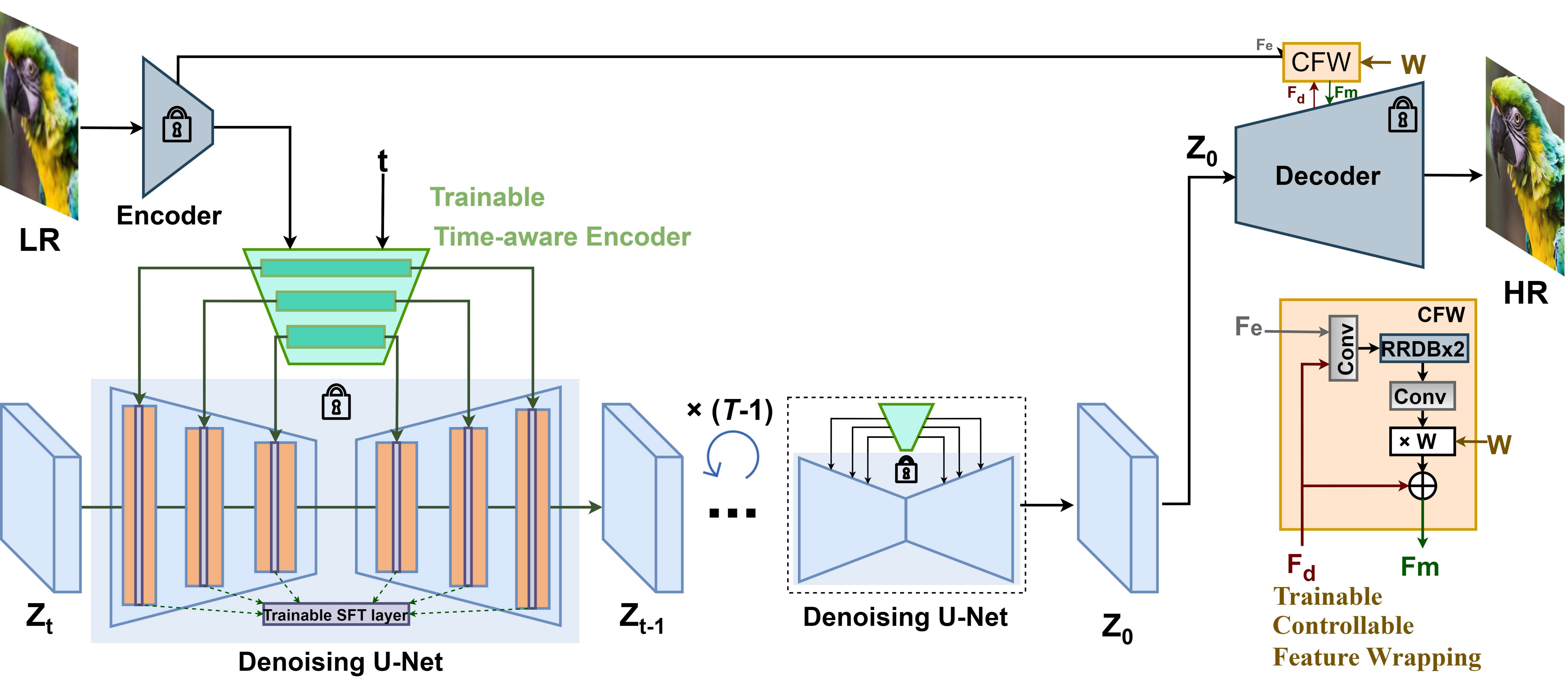}
    \caption{Architecture of the StableSR model for image restoration. The model leverages pre-trained weights from the Stable Diffusion model and incorporates an Encoder-Decoder for efficient image restoration.}
    \label{fig:stablesr_architecture}
\end{figure}

\paragraph{Results and Evaluation:} To evaluate the performance of the StableSR model, we conduct experiments across multiple image restoration tasks, including super-resolution, denoising, and deblurring. The model is compared against existing SOTA methods on benchmark datasets, demonstrating superior performance in quantitative metrics (e.g., PSNR, structural similarity index measure) and qualitative image quality.

\paragraph{Super-Resolution Results:} For super-resolution tasks, the StableSR model achieves a notable increase in the peak signal-to-noise ratio (PSNR) compared to traditional methods. This is due to the ability of the encoder-decoder structure to preserve fine-grained details while enlarging the image.

\paragraph{Denoising and Deblurring Results:} The model also performs well in denoising and deblurring tasks, efficiently removing noise and sharpening blurred images without introducing artifacts. The perceptual loss and adversarial loss components help ensure the restoration results look natural and visually accurate.
The StableSR model introduces a practical framework for image restoration that integrates the power of pre-trained models with tailored fine-tuning strategies. By leveraging the Encoder-Decoder and the capabilities of the Stable Diffusion model, StableSR achieves high-quality restoration results across multiple tasks while maintaining computational efficiency. The model demonstrates its potential for real-world applications in image enhancement, restoration, and generation functions, offering a versatile tool for computer vision problems.

\section{Experimental Results}

Within this section, we present the outcomes of our experimentation and conduct a comparative analysis between the ground truth images and their restored counterparts: StableSR (see Figure \ref{fig:sample_fig_label}, \ref{fig:sample_fig_label4}). The ensuing paragraphs elaborate on the results obtained through this dual-model approach, shedding light on the performance of StableSR in terms of image restoration. 

\section{Summary and Future Work}

In conclusion, this endeavor has spotlighted the efficacy of contemporary diffusion models in image restoration (IR), harnessing their robust generative potential to amplify structural and textural revitalization. The initial phase of this work entailed leveraging pre-trained weights to establish a foundational baseline, followed by the progressive evolution of the diffusion model for IR applications, with a specific focus on the adaptation of StableSR through a systematic fine-tuning process. This research has further delved into the comprehensive categorization of ten distinct distortions, shedding light on their nuances through the lens of training strategies and degradation scenarios.

Through meticulous analysis, we undertook a comparative assessment of existing works, encompassing both super-resolution and IR domains. Each approach was dissected precisely, affording an intricate taxonomy delineating its strengths and weaknesses. The evaluation process involved an overview of prevalent datasets and evaluation metrics within the diffusion model-based IR landscape. This culminated in a comprehensive comparison of two cutting-edge open-source SOTA methodologies, evaluated through a fusion of distortion and perceptual metrics across three quintessential tasks: image super-resolution, deblurring, and inpainting.

Remarkably, our observations highlighted the effectiveness of training diffusion models on specialized datasets tailored to distinct degradation types. This strategy yielded commendable outcomes, particularly in scenarios mirroring the noise or degradation patterns akin to the training data. Addressing the challenges inherent in diffusion model-based IR entails exploring diverse baseline datasets and refining training strategies as we steer toward prospects. By doing so, the realm of diffusion models can be further optimized for achieving superior outcomes, marking a promising direction for future exploration. 


\section*{Appendix}

Below, we present additional results where noise becomes integrated into the images, making it difficult to restore the original content (see fig-\ref{fig:sample_fig_label2})

\begin{figure*}[!ht]
    \centering
    \includegraphics[width=0.99\linewidth]{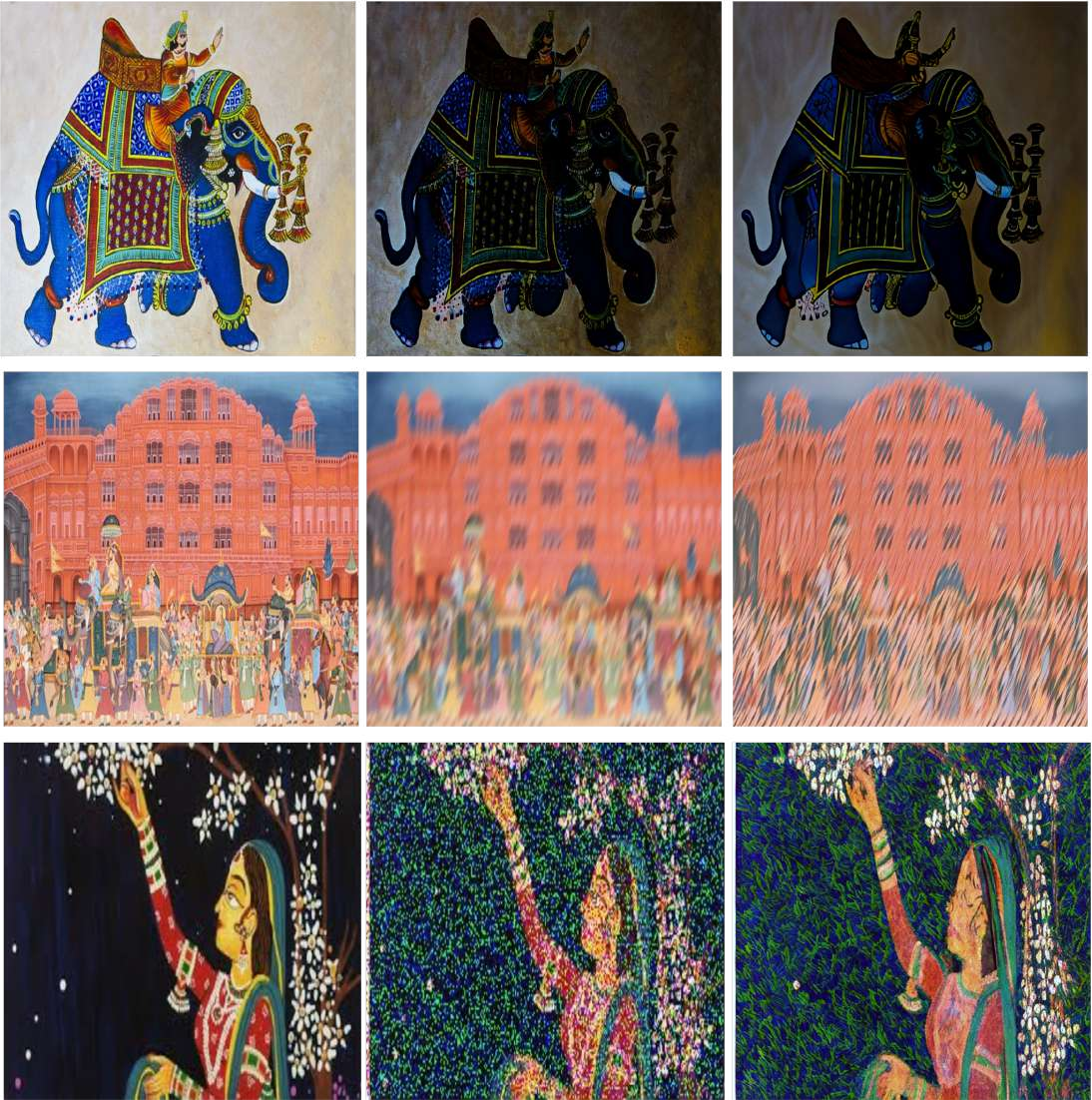}
    \caption{Method: StableSR. Left: GT, Center: distorted by \textbf{damage}, Right: Restored. Visual comparisons using the StableSR framework for image restoration. Each triplet illustrates (Left) the Ground Truth image, (Center) the image degraded by severe distortions or damage, and (Right) the image restored using StableSR. The results highlight StableSR’s ability to reconstruct fine details and recover structural integrity in heavily damaged images.}
    \label{fig:sample_fig_label2}
\end{figure*}


\begin{figure*}[!ht]
    \centering
    \begin{tabular}{ccc}
        \includegraphics[width=.33\linewidth,valign=m]{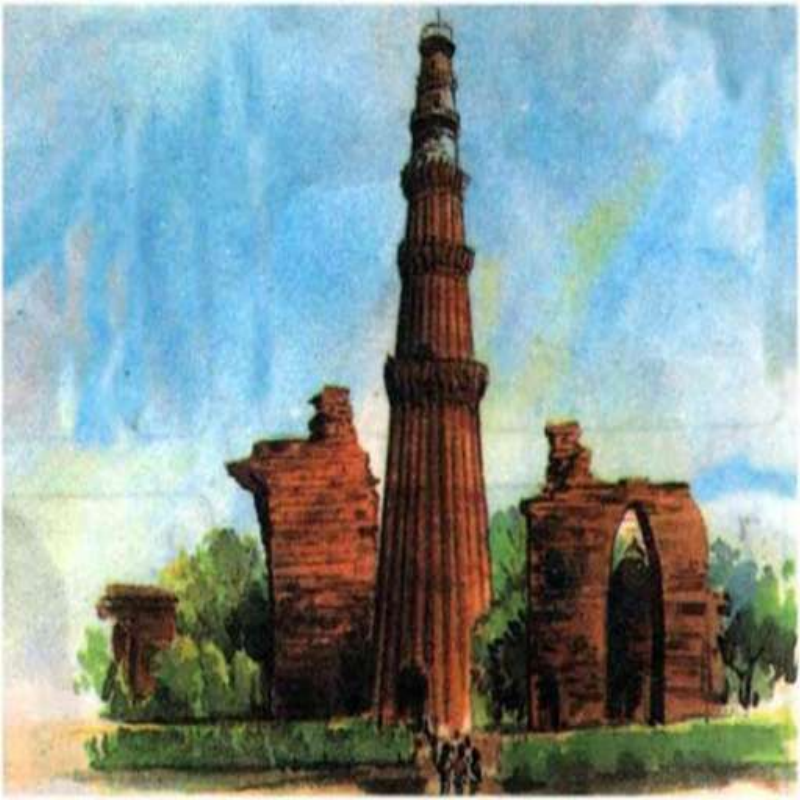} &
        \includegraphics[width=.33\linewidth,valign=m]{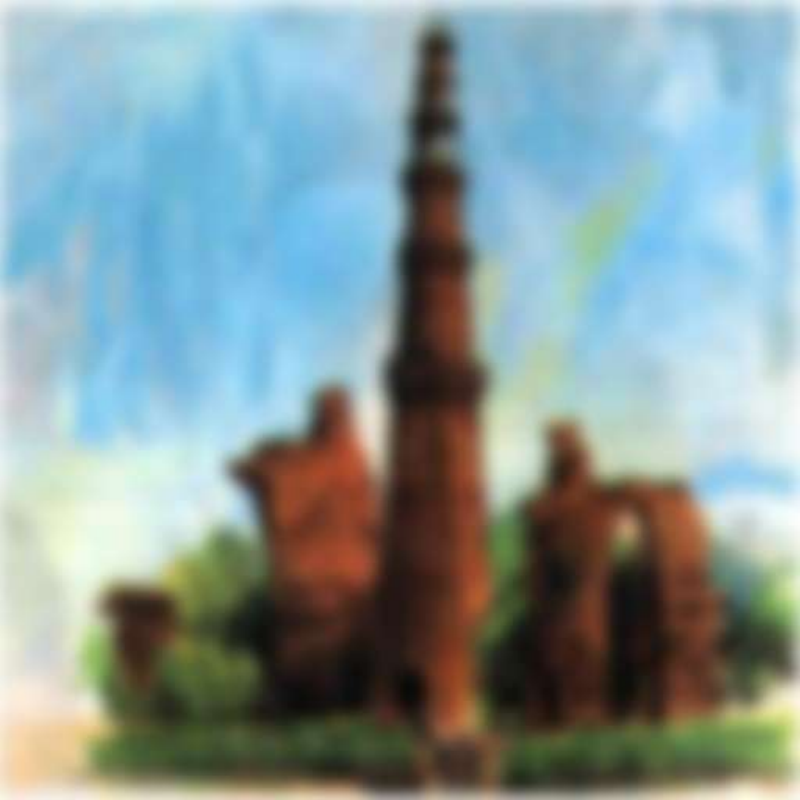} &
        \includegraphics[width=.33\linewidth,valign=m]{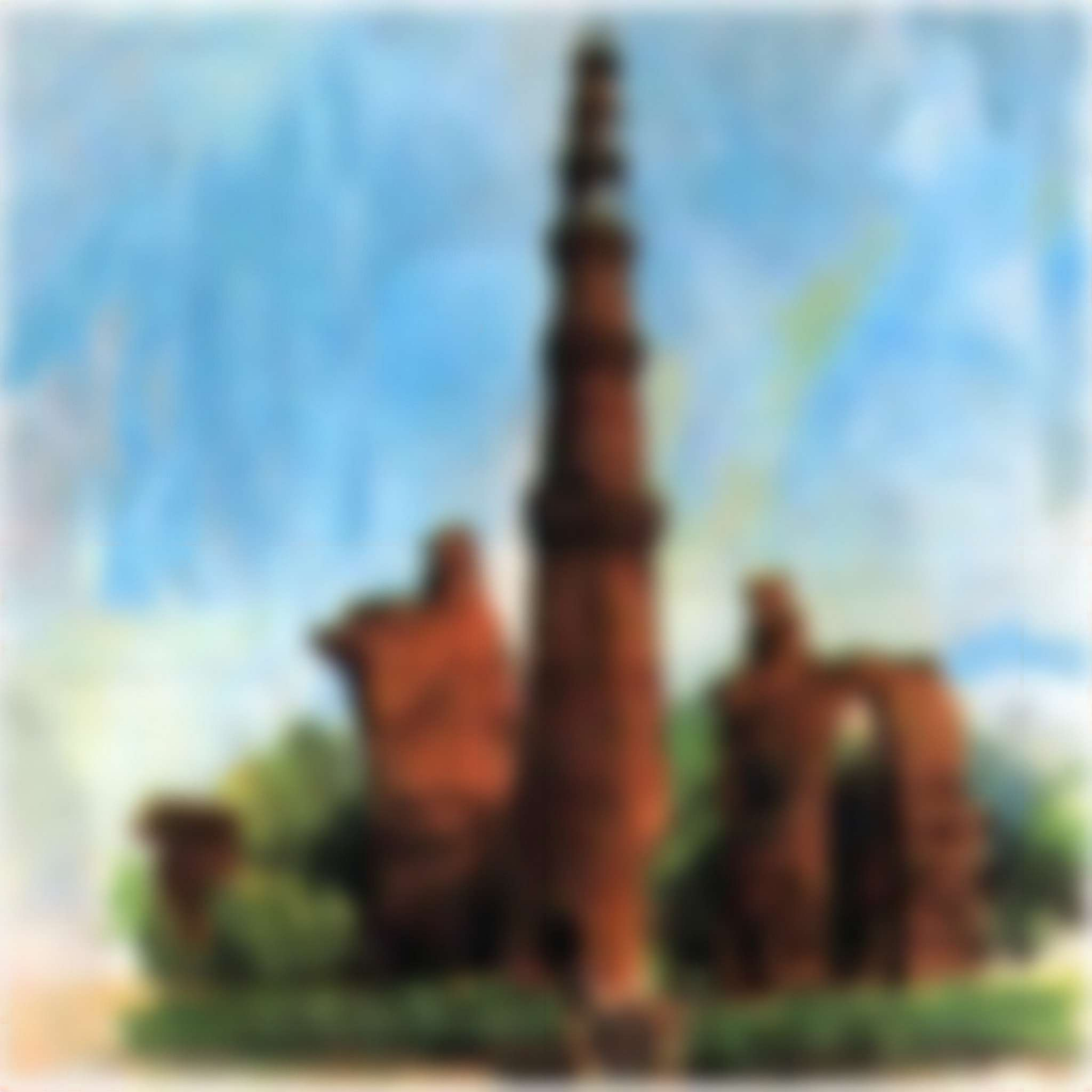} \\
        
        \includegraphics[width=.33\linewidth,valign=m]{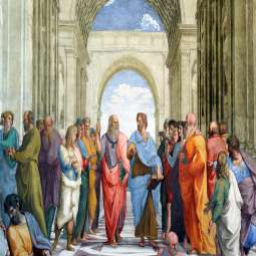} &
        \includegraphics[width=.33\linewidth,valign=m]{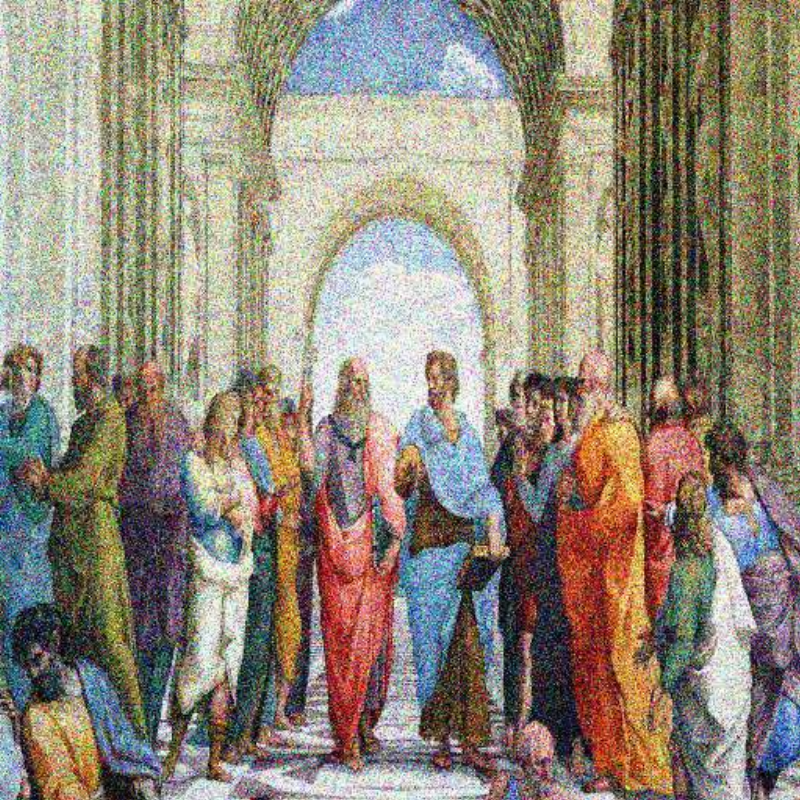} &
        \includegraphics[width=.33\linewidth,valign=m]{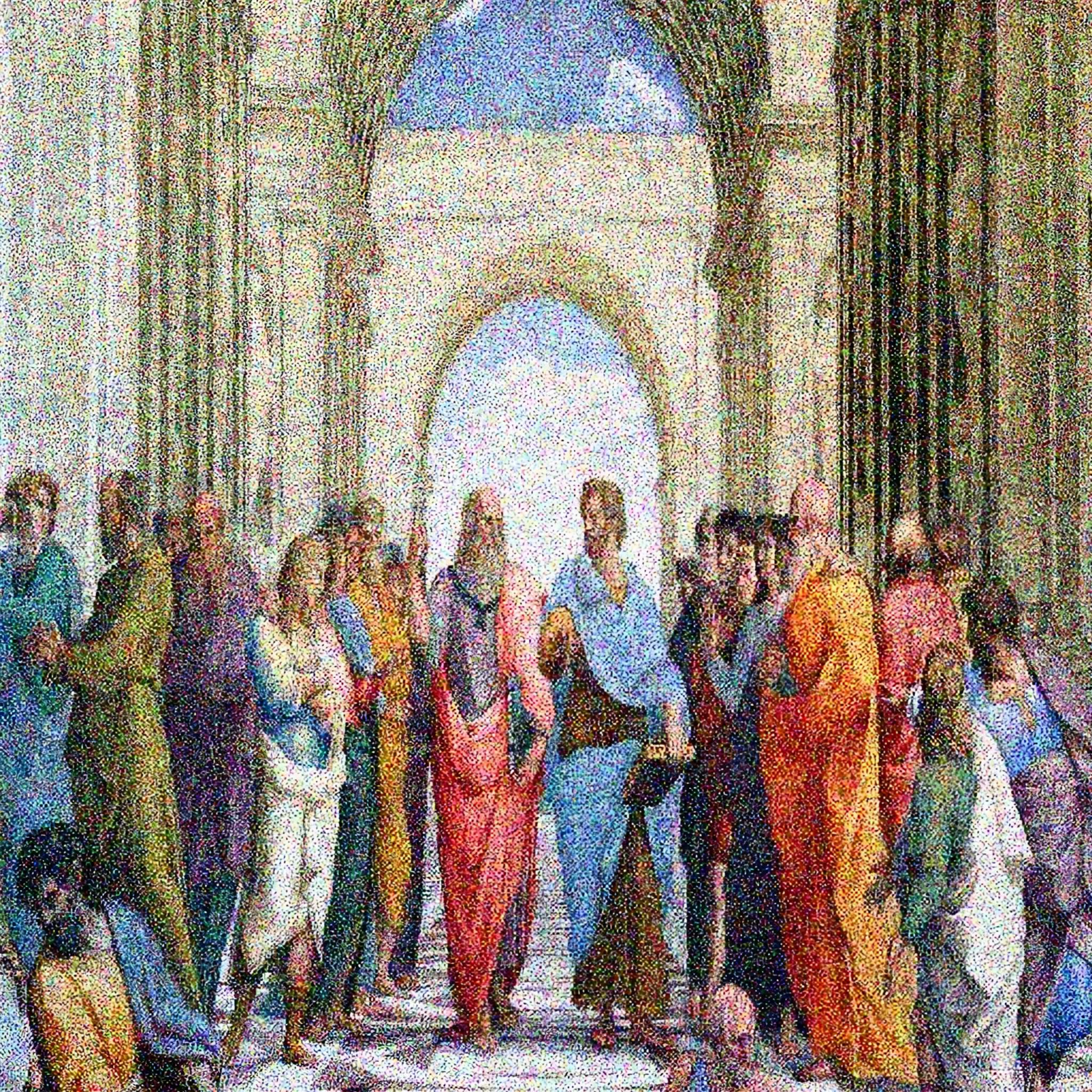} \\
        
        \includegraphics[width=.33\linewidth,valign=m]{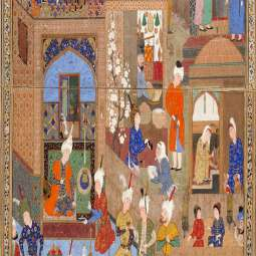} &
        \includegraphics[width=.33\linewidth,valign=m]{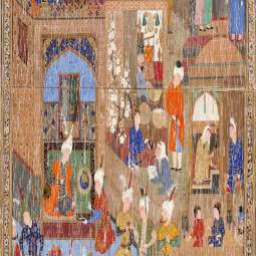} &
        \includegraphics[width=.33\linewidth,valign=m]{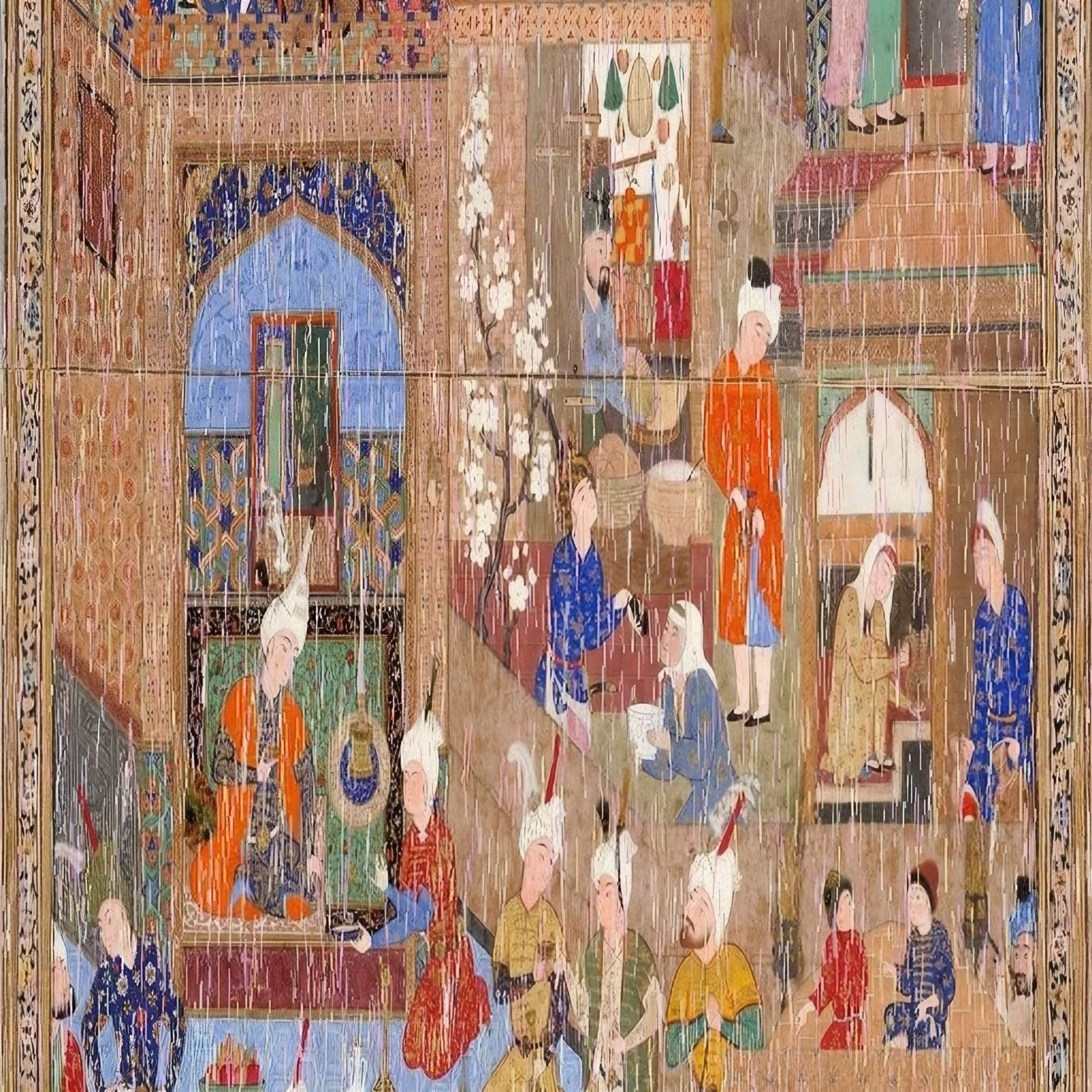} \\
    \end{tabular}
    \caption{Method: ResShift. Left: GT, Center: distorted by artifact, Right: Restored. Qualitative results of the ResShift method integrated with the StableSR framework for image restoration. Each triplet shows (Left) the Ground Truth image, (Center) the degraded image affected by artifacts such as blur or noise, and (Right) the image restored using ResShift. The restored outputs demonstrate substantial visual improvement and detail recovery compared to the degraded inputs, indicating the effectiveness of the proposed approach.}
    \label{fig:ResShift1}
\end{figure*} 
 \chapter{Robust Missing Traffic Sign Detection with Data Augmentation and Model Fine-Tuning}\label{chap:missing_sign}
\section{Introduction}

Global Autonomous Vehicle (AV) market is projected to reach $2162$ billion by $2030$. Around $ 1.3 $M lives are lost yearly in road accidents, with traffic sign violations contributing significantly to this number. The traffic signs are generally installed at the side of the road to control traffic flow or convey information about the road environment to Vulnerable Road Users (VRUs). Often, the data is also available in the form of cues present in the context around the traffic signs (\textit{obstacle-delineator}) or in the cues away from it (\textit{pedestrian-crossing}), which we refer to as \textit{contextual cues}.

The Missing Traffic Signs Video Dataset (MTSVD) \citep{gupta2023cuecan} is the first publicly accessible dataset for missing traffic signs.

Traffic signs, despite being crucial for road safety, frequently remain absent. This challenge provides $200$ scenes from a recent Missing Traffic Signs Video Dataset (MTSVD), distributed over four types of missing traffic signs: \textit{left-hand-curve}, \textit{right-hand-curve}, \textit{gap-in-median}, and \textit{side-road-left}, individually observed with their respective contextual cues. $2000$ training images, each containing one of the four traffic signs with corresponding bounding boxes, are provided. Two tasks are proposed for the challenge: i) \textit{Object Detection}, wherein the model is trained using bounding box annotations, and ii) \textit{Missing Traffic Sign Scene Categorization} \ref{fig:context_variety}, wherein the model is trained using road scene images with in-painted traffic signs, provided with the challenge dataset. Baselines were provided to the participants for both tasks. 54 teams registered for the challenge. Overall, the participants could improve the top-1 accuracy significantly by a margin of $31.5\%$ over the baseline. This work presents the MTSVD, challenges, baselines, and the methodology.

\begin{figure}
    \centering
    \includegraphics[width=0.99\linewidth]{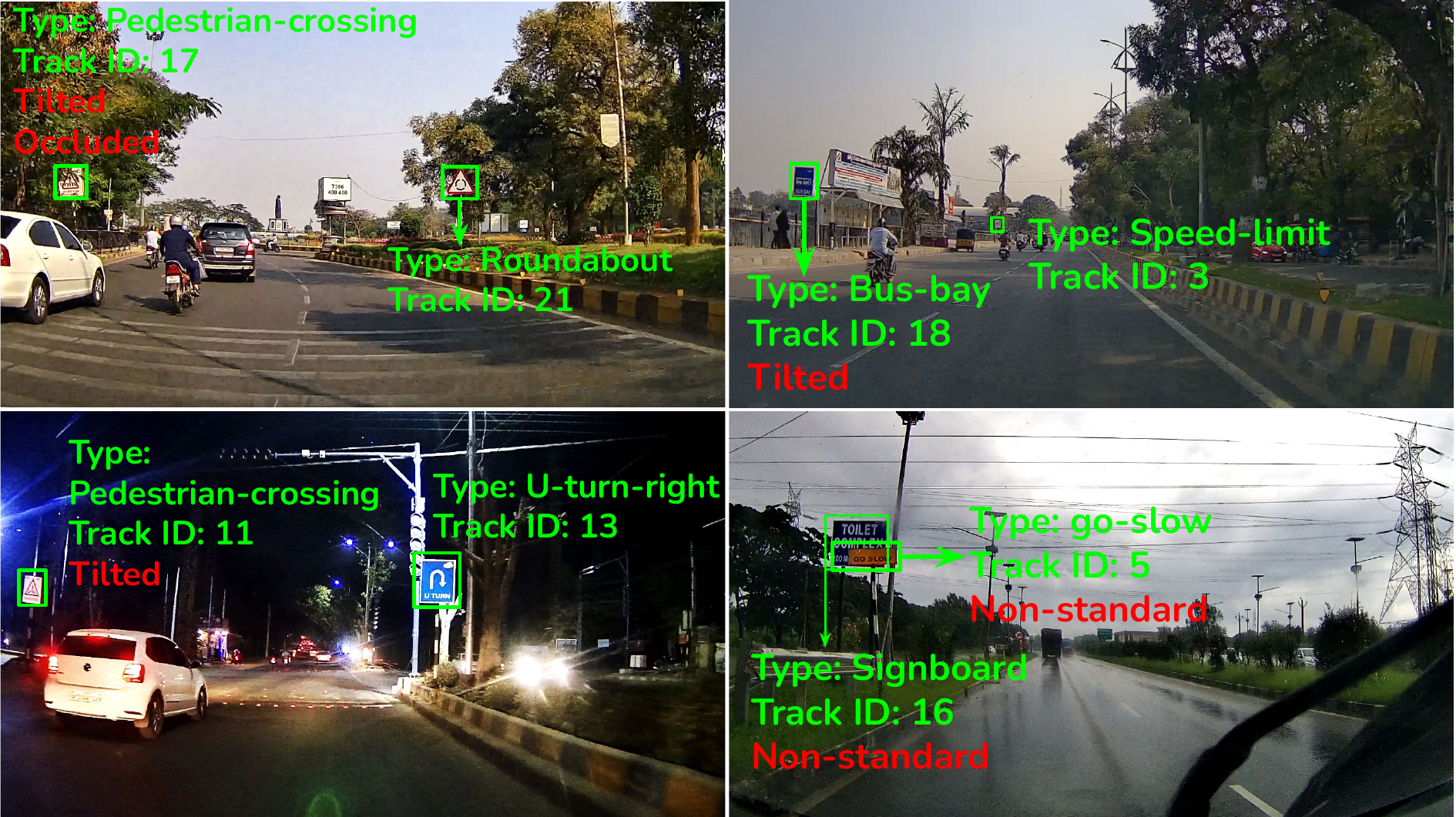}
    \caption{Examples of frame-level traffic sign annotations in the MTSVD \cite{gupta2023cuecan}. Each frame includes bounding-box tracks for traffic signs, annotated with type, track ID, and visual attributes such as tilted, occluded, and non-standard. The dataset captures a variety of sign types across diverse lighting and weather conditions, contributing to its robustness for real-world scene understanding and detection tasks.}
    \label{fig:enter-label}
\end{figure}

\begin{figure}
    \centering
    \includegraphics[width=0.99\linewidth]{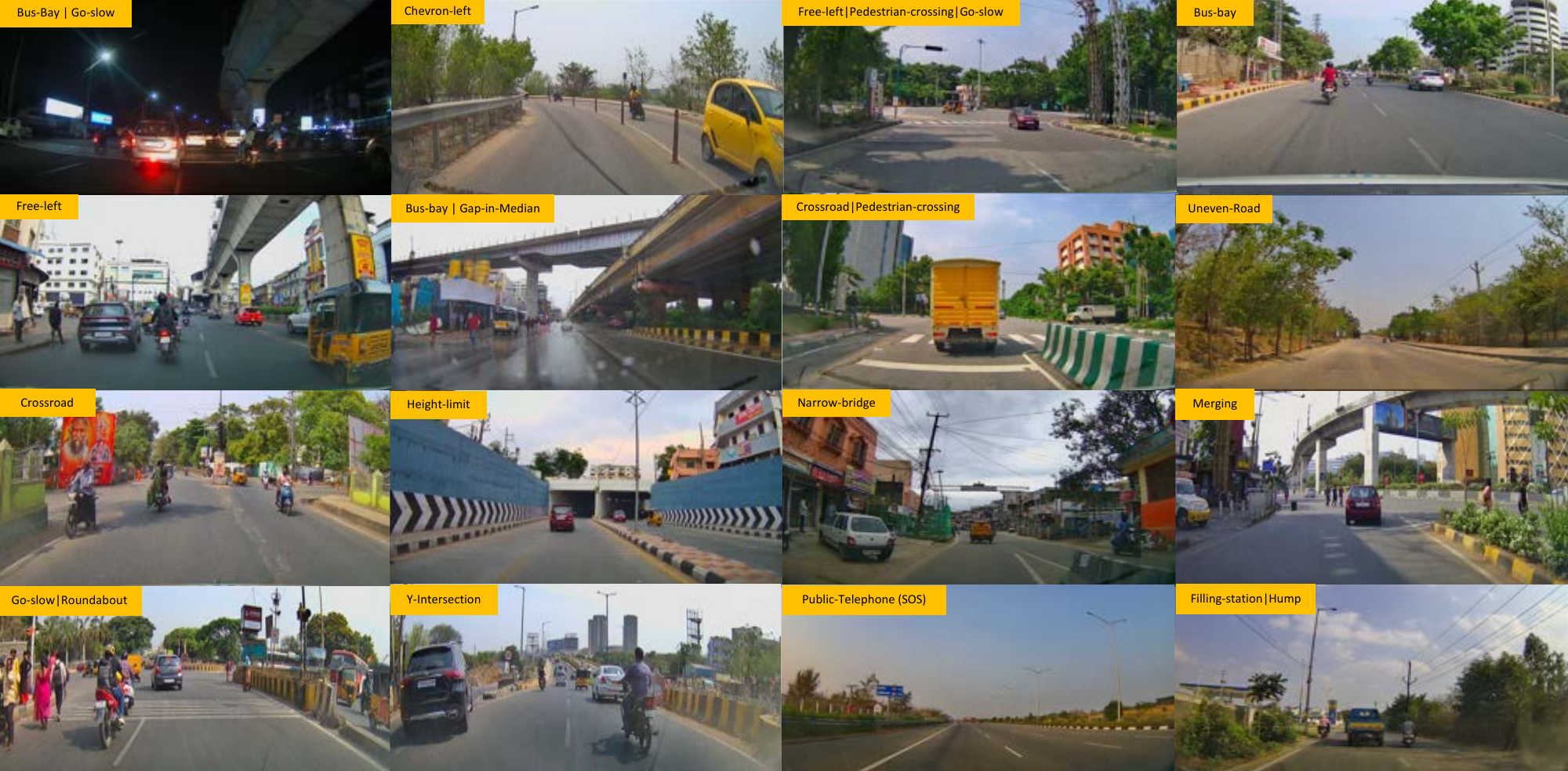}
    \caption{Illustrative examples of contextual scene categories from the MTSVD, where the associated traffic sign is missing. These scenes demonstrate a variety of real-world road environments such as pedestrian crossings, bus bays, sharp curves, and merging roads, each containing consistent contextual cues despite the absence of the expected traffic sign. Such context-driven patterns are critical for learning-based models to infer missing traffic signage and enable effective scene categorization \cite{gupta2023cuecan}.}
    \label{fig:context_variety}
\end{figure}

\section{Challenge: Missing Traffic Signs}

Traffic signs are fundamental components of road infrastructure, critical in ensuring road safety. However, traffic signs are often missing or obscured in real-world driving conditions due to environmental factors, vandalism, or poor maintenance. To address this issue, the Missing Traffic Signs Video Dataset (MTSVD) was introduced, comprising $200$ curated scenes that depict the absence of specific traffic signs. These scenes are categorized into four classes based on the missing traffic sign: left-hand curve, right-hand curve, gap-in-median, and side-road-left. Each class is defined by its distinctive contextual cues observed in the scene.

In addition to the $200$ testing scenes, the dataset includes $2000$ training images, each containing one of the four traffic signs annotated with bounding boxes. The challenge is structured around two primary tasks:
(i) Object Detection, where models are trained using bounding-box annotations to detect the presence and location of traffic signs; and
(ii) Missing Traffic Sign Scene Categorization, where models are trained on road scene images with traffic signs to categorize scenes based on which sign is missing.

This chapter presents a detailed overview of the MTSVD dataset, the baseline approaches, and the methods adopted by the top two performing teams, offering valuable insights into current methodologies for addressing the problem of missing traffic signs using computer vision.


\begin{figure}[!ht]
   \centering
   \includegraphics[width=\linewidth]{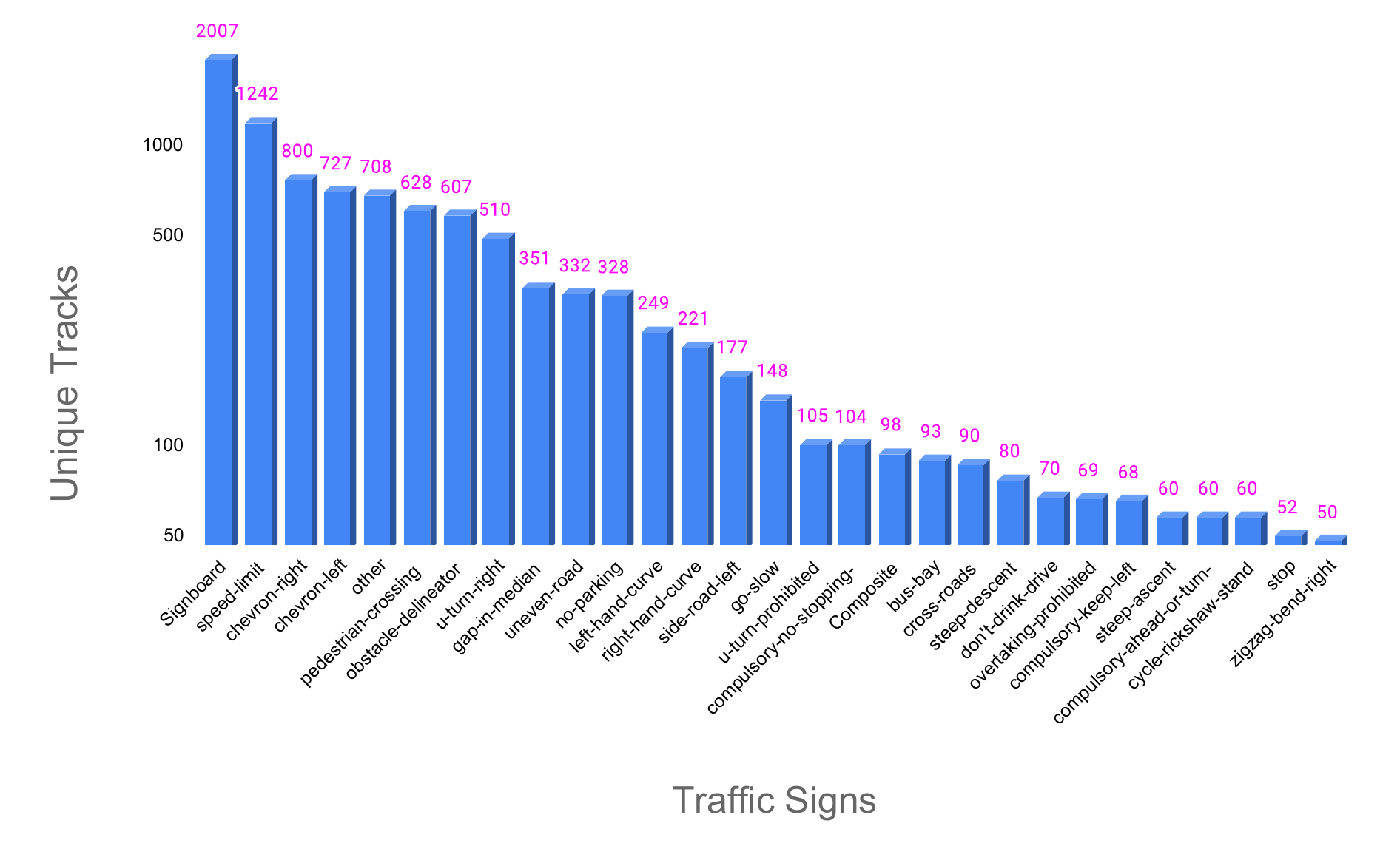}
   \caption{Distribution of the top 30 most frequent traffic signs in the MTSVD dataset, represented by the number of unique annotated tracks. Statistics of traffic sign annotations in the Missing Traffic Signs Video Dataset (MTSVD). These statistics highlight the imbalance in traffic sign presence and absence within the dataset, reflecting real-world scenarios where certain signs are more commonly encountered or missed. Annotated and Missing traffic signs in the MTSVD \cite{gupta2023cuecan}.}
   \label{fig:30sub_a} 
\end{figure}
    
\begin{figure}[!ht]
   \centering
   \includegraphics[width=\linewidth]{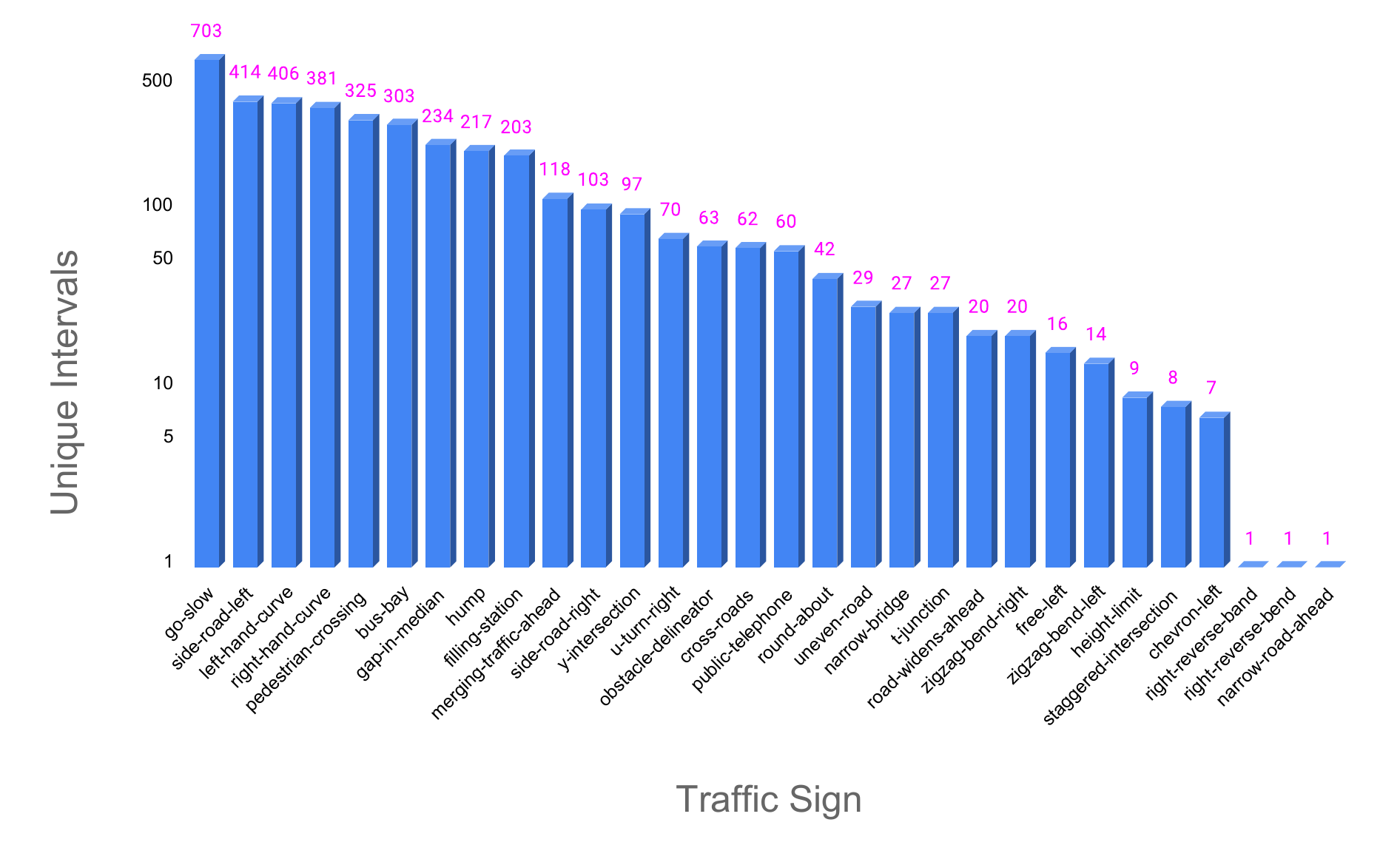}
   \caption{Distribution of missing traffic sign categories in MTSVD, annotated as unique frame intervals where the corresponding sign is contextually absent. Statistics of traffic sign annotations in the Missing Traffic Signs Video Dataset (MTSVD). These statistics highlight the imbalance in traffic sign presence and absence within the dataset, reflecting real-world scenarios where certain signs are more commonly encountered or missed. Annotated and Missing traffic signs in the MTSVD \cite{gupta2023cuecan}. }
   \label{fig:30_sub2}
\end{figure}
    

\section{Method}
To fine-tune the classification and detection model, we applied the following data augmentation techniques from \href{https://docs.roboflow.com/datasets/image-augmentation}{Roboflow}: random rotation, mosaic data augmentation, horizontal flip, vertical flip, and generated $8000$ extra samples from the task-1 dataset, resulting in a total of $10000$ instances. It is observed that adding Dropout also improved the results. The small footprint of the traffic sign in the images suggests that this is a fine-grained detection task. Given the relatively limited samples in the dataset, small-to-medium models are considered to avoid possible overfitting.

\paragraph{Task-1: Object Detection} The task-1 dataset is split into train and validation, with an image size of $720\times720$. Existing SoTA models: YOLOv8n, YOLOv8s, YOLOv8m, and YOLO8l are fine-tuned using pre-trained weights, with the YOLOv8s model being optimal. The models are trained on the NVIDIA A100 for $80$ epochs with the following hyperparameters: $1e-4$ learning rate, $2e-4$ weight decay, $0.982$ momentum, $0.25$ Dropout. All these hyperparameters were determined empirically through hyperparameter tuning on a validation set to achieve the best balance between training performance and generalization, effectively mitigating overfitting.

\paragraph{Task-2: Traffic Sign Scene Categorization}  Train a model wherein model training happens using road scene images with in-painted traffic signs spread over four traffic sign categories: \textit{left-hand-curve}, \textit{right-hand-curve}, \textit{gap-in-median}, and \textit{side-road-left} for scene categorization. The evaluation for task 2 is done by comparing the predicted labels submitted by participants to the test labels of the $200$ test images, with top-1 accuracy as the chosen metric to evaluate the models. The test set contains a proportionate combination of real and in-painted images to ensure that the model doesn't favor the in-painting artifacts and performs poorly on real scenes. \citep{gupta2023c4mts}

\section{Results}

The quantitative results for both tasks in the challenge are summarized in Table \ref{tab:results_c4mts}. In Task 1: Object Detection, the method developed by IAMGROOT (our) achieved the highest mean Average Precision (mAP) of $0.90$, outperforming the baseline ($ mAP = 0.88$) by a margin of $2.27\%$. In contrast, Sahajeevi’s method recorded an mAP of $0.84$, falling short of the baseline by $4.76\%$. This underperformance can be partly attributed to differences in training durations; as noted in Section 5 of Gupta et al. \cite{gupta2023c4mts}, the baseline model was trained for approximately 40 more epochs than Sahajeevi’s submission.

In Task 2: Missing Traffic Sign Scene Categorization, both participating teams demonstrated substantial improvements over the baseline. Table \ref{tab:results_c4mts}, while the baseline achieved a Top-1 Accuracy of only $0.290$, Sahajeevi’s method reached $0.514$, and IAMGROOT surpassed all with an accuracy of $0.605$, indicating a $108\%$ relative improvement over the baseline. These results highlight the effectiveness of tailored approaches for contextual scene understanding in the missing sign detection task, especially under weak supervision. IAMGROOT’s consistent top performance across both tasks underscores the robustness and generalization of their method in both detection and scene categorization settings.

\begin{table}[ht]
    \centering
    \caption{Task 1 and Task 2 results: The results demonstrate the effectiveness of our method (IAMGROOT) in both tasks of the missing traffic sign challenge. For Task 1 (Object Detection), our approach achieved the highest mean Average Precision (mAP) of $0.90$, marginally surpassing the baseline and significantly outperforming the other participants. For Task 2 (Scene Categorization), our method yielded a Top-1 Accuracy of $0.605$, more than doubling the performance of the baseline ($0.290$), indicating strong contextual understanding in the presence of weak supervision. These results underscore the robustness and generalizability of our proposed pipeline across both detection and scene-level reasoning tasks, setting a new benchmark for missing traffic sign detection using contextual cues.}
    \label{tab:results_c4mts}
    \begin{tabular}{lcc}
        \toprule
         \textbf{Method} & \textbf{mAP} (Task-1) & \textbf{Top-1 Accuracy} (Task-2)\\
         \midrule
            Baseline & $0.88$ & $0.290$\\
            Sahajeevi & $0.84$ & $0.514$ \\
            IAMGROOT (our) & $\textcolor{blue}{0.90}$ & $\textcolor{blue}{0.605}$\\
        \bottomrule
    \end{tabular}
\end{table}

\section{Discussion} 

In this chapter, we presented a comprehensive study on detecting and categorizing missing traffic signs using the Missing Traffic Signs Video Dataset (MTSVD). This is a critical yet underexplored problem with direct implications for autonomous driving systems and road safety. The chapter introduced a unique challenge that encapsulates two key tasks: (i) object detection using annotated traffic sign data, and (ii) scene categorization for identifying which traffic sign is contextually missing.

Our proposed method, IAMGROOT, outperformed the provided baselines and other participating teams across both tasks. For Task 1, IAMGROOT achieved a state-of-the-art mean Average Precision (mAP) of $0.90$, reflecting its superior ability to localize traffic signs under limited supervision and imbalanced conditions. For Task 2, which demanded contextual reasoning in the absence of explicit annotations, IAMGROOT again led with a Top-1 Accuracy of $0.605$, outperforming the baseline by over $108\%$, highlighting the effectiveness of our fine-tuned models and augmentation strategies.

The results suggest that (a) appropriate data augmentation and model regularization techniques (e.g., Dropout) are critical for small object detection, and (b) leveraging in-painted data for contextual learning can significantly enhance scene-level understanding. These findings underscore the importance of designing robust models that generalize well in weakly supervised and imbalanced data settings—realistic conditions for many real-world V systems.

This work validates the MTSVD dataset's value and sets a strong benchmark for future research in missing sign detection and contextual scene analysis. In the future, incorporating multi-modal inputs (e.g., LiDAR, temporal cues), improving generalization to unseen sign categories, and exploring self-supervised learning approaches offer promising directions for enhancing autonomous scene understanding.

 \part{Conclusion \& Summry}
 \chapter{Summary and Future Directions}\label{chap:summary}

\section{Summary}

This thesis proposal has significantly contributed to both the theoretical foundations of generative models and their practical applications in computer vision. Our work spans two primary directions: advancing the efficiency of normalizing flow models and applying generative modeling techniques to solve real-world problems. In the domain of flow-based models, we introduced CInC Flow and Inverse-Flow, novel architectures that significantly advance the state-of-the-art through several key innovations. These include the development of an efficient backpropagation algorithm for inverse convolution optimized for parallel operations, implementing a multi-scale architecture that substantially accelerates sampling in normalizing flow models, and creating a CUDA-based GPU implementation that enhances computational performance. Comprehensive experimental validation has demonstrated improved training capability and sampling speed compared to existing approaches.

Our applications of generative models to computer vision problems have yielded significant results across multiple domains. In geological mapping, we demonstrate the defectiveness of stacked autoencoders combined with k-means clustering, achieving superior feature extraction compared to traditional methods. Framework's integration with Sentinel-2 data yielded the highest spatial resolution and accuracy, particularly in the Mutawintji region of NSW, Australia.

For image super-resolution, we developed Affine-StableSR, which successfully addresses computational and model size challenges through the integration of affine-coupling layers in encoder-decoder architectures, efficient utilization of pre-trained weights from stable diffusion models, and introduction of a diverse four-in-one dataset for robust training.

We advanced the field in art restoration by adapting diffusion models for comprehensive image restoration. This involved developing systematic fine-tuning processes, creating a taxonomy of ten distinct distortion types, and conducting comparative assessments of existing methodologies across multiple restoration tasks. Our approach demonstrated superior performance in handling various degradation through unified fine-tuning, making it particularly valuable for cultural heritage preservation.

Our work, the Pvt\_IDD dataset, establishes a practical framework for privacy-preserving visual learning, enabling effective model training while protecting sensitive identity information.

\section{Future Work}
Looking ahead, several promising research directions emerge from this work. We plan to explore enhanced convolutional inversion models, particularly for blind deconvolution, and expand privacy-preserving capabilities for autonomous driving applications. Developing more robust face and license plate detection systems and extending our methods to broader stable diffusion applications remain the same. In the domain of flow-based models, we introduced CInC Flow and Inverse-Flow fusion model-based image restoration, focusing on specialized datasets for distinct degradation types.

This work establishes a strong foundation for future advancements in efficient generative modeling and its applications, contributing significantly to computer vision, geological mapping, image enhancement, and art restoration. Frameworks and methodologies developed here offer promising pathways for continued innovation in these domains, particularly in addressing real-world challenges where computational efficiency and accuracy are paramount.


 \addcontentsline{toc}{chapter}{Publications}
 \thispagestyle{plain}
\begin{center}
    \Large \textbf{\uppercase{Publications}}
\end{center}

\vspace{2\baselineskip}

\justifying

\section*{Publications Towards Thesis}
\subsection*{Journals:}

\begin{enumerate}
    \item \textbf{Nagar, S.}, Farahbakhsh, E., Awange, J., Chandra, R.: Remote sensing framework for geological mapping via stacked autoencoders and clustering. Journal of Advances in Space Research - Volume 74: pages 4502-4516, 2024. ISSN 0273-1177. doi: 10.1016/j.asr.2024.09.013. \textcolor{blue}{COSPAR Outstanding Paper Award for Young Scientists} \href{https://github.com/Naagar/stackedAE4Geo}{code}
\end{enumerate}

\subsection*{Conferences:}

\begin{enumerate}

    \item \textbf{Nagar, S.}, and Varma, G: "Parallel Backpropagation for Inverse of a Convolution with Application to Normalizing Flows." Proceedings of The 28th International Conference on Artificial Intelligence and Statistics (AISTATS), vol. 239, Proceedings of Machine Learning Research, PMLR, 2025, pp. 03-05 May. \href{https://naagar.github.io/InverseFlow/}{, code}

    \item \textbf{Nagar, S.}, Alam, M., Beiser, D. G., Hasegawa-Johnson, M., Ahuja, N.: R2I-rPPG: A Robust Region of Interest Selection for Remote Photoplethysmography to Extract Heart Rate. \textbf{Under review}, arxiv:2410.15851 \href{https://github.com/Naagar/R2I-rPPG}{, code}

    \item \textbf{Nagar, S.}, Dufraisse, M. and Varma, G.: CInc flow: Characterizable invertible 3 × 3 convolution. In The 4th Workshop on Tractable Probabilistic Modeling, Uncertainty in Artificial Intelligence (UAI 2021). \emph{Oral} \href{https://naagar.github.io/CInCFlow/}{, code}

    \item \textbf{ Nagar, S.}, Bala, A., Patnaik, S. A.: Adaptation of the super-resolution SOTA for art restoration in-camera captures images. In International Conference on Emerging Techniques in Computational Intelligence (ICETCI 2023) (pp. 158-163). IEEE. \emph{Oral} \href{https://github.com/Naagar/art_restoration_DM}{, code}

    \item  \textbf{Nagar, S.}, Pani, P., Nair, R., Varma, G.: Automated Seed Quality Testing System Using GAN and Active Learning. In International Conference on Pattern Recognition and Machine Intelligence, 2023 (pp. 509-519), Springer International Publishing. \emph{Oral} \href{https://naagar.github.io/cornseedsdataset/}{, code}

    \item Gupta, V., \textbf{Nagar, S.}, Choudhury, S.P., Singh, R., Jamwal, A., Gupta, V., Subramanian, A., Jawahar, C.V. and Saluja, R.: C4MTS: Challenge on Categorizing Missing Traffic Signs from Contextual Cues. In National Conference on Computer Vision, Pattern Recognition, Image Processing, and Graphics (pp. 141-154). Singapore: Springer Nature Singapore, 2023. \href{https://github.com/Naagar/Missing_Traffic_Sign}{, code}

    \item \textbf{Nagar, S.}, Puttagunta, S., Varma, G.: Privacy in Indian Driving Dataset using Inpainting. \textbf{In preparation} \href{https://github.com/Naagar}{, code}

    \item \textbf{Nagar, S.}, Dendi, S. V. R., Nair, P., Gadde, R. N.: Affine-StableSR: Coupling Layers based Lightweight Autoencoder and Stable Diffusion for Super-Resolution. \textbf{Preprint} \href{https://github.com/Naagar}{, code}
    
\end{enumerate}
\section*{Additional Publications}
\begin{enumerate}

    \item Ray, A., \textbf{Nagar, S.}, Varma, G., Paul, A.: Unlocking simple features to predict DFT spin state gaps of 3D metal complexes using Machine Learning. Accepted: Journal of Physical Chemistry Chemical Physics (PCCP) 2025. \href{https://github.com/Naagar}{, code.}
    
    \item Kallappa. A., \textbf{Nagar. S.}, and Varma. G.: Finc flow: Fast and invertible k × k convolutions for normalizing flows. Proceedings of the 18th International Joint Conference on Computer Vision, Imaging and Computer Graphics Theory and Applications (VISIGRAPP 2023) - Volume 5: VISAPP, pages 338–348, 2023. ISSN 2184-4321. doi: 10.5220/0011876600003417. \emph{Oral} \href{https://naagar.github.io/FInC-Flow/}{, code.}
    
    \item Shaik, F.A., \textbf{Nagar. S.}, Maturi, A., Sankhla, H.K., Ghosh, D., Majumdar, A., Vidapanakal, S., Chaudhary, K., Manchanda, S. and Varma, G., 2024, November. ICPR 2024 Competition on Safe Segmentation of Drive Scenes in Unstructured Traffic and Adverse Weather Conditions. In International Conference on Pattern Recognition (pp. 145-160). Cham: Springer Nature Switzerland \emph{, Oral}  \href{https://github.com/Furqan7007/IDDAW_kit}{, code.}
\end{enumerate}

    

 \addcontentsline{toc}{chapter}{Bibliography}
 \bibliographystyle{unsrt}
 \bibliography{references}
 
\end{document}